\documentclass[10pt]{article}
\usepackage{enumerate}
\usepackage[OT1]{fontenc}
\usepackage{amsmath, amssymb}
\usepackage{natbib}
\usepackage[usenames, dvipsnames]{xcolor}
\usepackage{bm}
\usepackage{geometry}
\usepackage{graphicx}
\usepackage{acronym}
\usepackage{hhline}
\usepackage{dsfont}
\usepackage{pgfplots}
\usepackage{smile}
\usepackage{multirow}
\usepackage{rotating}
\usepackage{enumerate}
\usepackage{esvect}
\usepackage{threeparttable}
\usepackage{tikz}
\usetikzlibrary{patterns}
\usetikzlibrary{arrows}
\usepackage{longtable}         
\usepackage{booktabs}           
\usepackage{array}          
\usepackage{multirow}  
\usepackage{subcaption}
\usepackage{framed}
\colorlet{shadecolor}{orange!15}
\usepackage{algorithm}
\usepackage{algorithmic}
\usepackage{siunitx}

\definecolor{mydarkblue}{rgb}{0,0.08,0.45}
\hypersetup{
	colorlinks=true,
	urlcolor=magenta,
	citecolor=mydarkblue,
}

\usepackage[most,skins]{tcolorbox}
\definecolor{rliableolive}{HTML}{BBCC33}
\definecolor{rliableorange}{HTML}{9A5F51}
\definecolor{rliableblue}{HTML}{77AADD}
\definecolor{rliablered}{HTML}{EE8866}
\definecolor{LightCyan}{rgb}{0.88,1,1}
\definecolor{darkblue}{HTML}{2878D9}
\definecolor{navyblue}{HTML}{0000FF}

\newtcolorbox{AIbox}[2][]{aibox,title=#2,colback=rliableblue!10!white,#1}

\newenvironment{olivebox}{%
    \begin{tcolorbox}[colback=rliableolive!10!white,colframe=black,boxsep=3pt,top=4pt,bottom=4pt,left=3pt,right=3pt]
}{%
    \end{tcolorbox}
}

\def\##1\#{\begin{align}#1\end{align}}
\def\$#1\${\begin{align*}#1\end{align*}}

\usepackage{arydshln}
\usepackage{diagbox}

\newtheorem{hypothesis}{Hypothesis}

\acrodef{llm}[LLM]{Large Language Models}
\acrodef{ffn}[FFN]{Feed-Forward Networks}
\acrodef{mha}[MHA]{Multi-Head Attention}
\acrodef{svd}[SVD]{Singular Value Decomposition}
\acrodef{gd}[GD]{Gradient Descent}
\acrodef{oco}[OCO]{Online Convex Optimization}

\newcommand{\ff}{\texttt{ff}}
\newcommand{\sm}{\texttt{sm}}
\newcommand{\gate}{\text{gate}}
\newcommand{\head}{\text{head}}
\newcommand{\inn}{\text{in}}
\newcommand{\out}{\text{out}}
\newcommand{\gd}{\text{GD}}
\newcommand{\adam}{\text{SignGD}}
\newcommand{\muon}{\text{Muon}}
\newcommand{\nor}{\text{norm}}
\newcommand{\opt}{\text{opt}}
\newcommand{\maxd}{\varrho}

\usepackage{tcolorbox}
\usepackage{xcolor}
\definecolor{lightroyalblue}{HTML}{F6F8FD} % more blue: E5EAFB
\definecolor{lightorange}{RGB}{252,236,219}
\definecolor{royalblue}{HTML}{4169E1}
\definecolor{lighterblue}{HTML}{f2fafd}  % more blue: e4f4fa
\newtcolorbox{abox}{colback=lightorange,frame empty}
\definecolor{LightCyan}{rgb}{.9, .95, 1.}
\definecolor{rowblue}{RGB}{232,241,255}

\title{\textbf{Muon Outperforms Adam in Tail-End\\ Associative Memory Learning}}

\author{Shuche Wang\textsuperscript{1,$\ast$}\quad Fengzhuo Zhang\textsuperscript{1,$\ast$,$\dagger$}\quad Jiaxiang Li\textsuperscript{2,$\ast$}\quad Cunxiao Du\textsuperscript{3} \\\quad Chao Du\textsuperscript{3}\quad Tianyu Pang\textsuperscript{3}\quad Zhuoran Yang\textsuperscript{4} \quad Mingyi Hong\textsuperscript{2} \quad Vincent Y.~F.~Tan\textsuperscript{1}\\
\textsuperscript{1}National University of Singapore\quad \textsuperscript{2}University of Minnesota \quad \\\textsuperscript{3}Sea AI Lab  \quad \textsuperscript{4}Yale University\\
\texttt{shuche.wang@u.nus.edu}\quad \texttt{fzzhang@u.nus.edu} \quad \texttt{li003755@umn.edu}\quad\\ 
\texttt{cnsdunm@gmail.com}\quad \texttt{ \{duchao, tianyupang\}@sea.com}\\ \texttt{zhuoran.yang@yale.edu}\quad\texttt{mhong@umn.edu}\quad \texttt{vtan@nus.edu.sg}
}
\date{}

\begin{document}
    \maketitle
    \begingroup
    \renewcommand\thefootnote{}\footnotetext{$\ast$ Equal contribution.}\footnotetext{$\dagger$ Project Lead.}\footnotetext{Fengzhuo conducted this work during his internship at Sea AI Lab.}
    \endgroup

    \begin{abstract}
        The Muon optimizer is consistently faster than Adam in training Large Language Models (LLMs), yet the mechanism underlying its success remains unclear. This paper demystifies this mechanism through the lens of associative memory. By ablating the transformer components optimized by Muon, we reveal that the associative memory parameters of LLMs, namely the Value and Output (VO) attention weights and Feed-Forward Networks (FFNs), are the primary contributors to Muon’s superiority. Motivated by this associative memory view, we then explain Muon’s superiority on real-world corpora, which are intrinsically heavy-tailed: a few classes (tail classes) appear far less frequently than others. The superiority is explained through two key properties: (i) its update rule consistently yields a more isotropic singular spectrum than Adam; and as a result, (ii) on heavy-tailed data, it optimizes tail classes more effectively than Adam. Beyond empirical evidence, we theoretically confirm these findings by analyzing a one-layer associative memory model under class-imbalanced data. We prove that Muon consistently achieves balanced learning across classes regardless of feature embeddings, whereas Adam can induce large disparities in learning errors depending on embedding properties. In summary, our empirical observations and theoretical analyses reveal Muon’s core advantage: its update rule aligns with the outer-product structure of linear associative memories, enabling more balanced and effective learning of tail classes in heavy-tailed distributions than Adam. \looseness=-1 %\mh{[generally looks good, can be slightly more specific about 'heavy tail' in the abstract.]}
        
        % strength lies in its superior ability to train the transformer's associative
        % memory components—specifically, the value-output (VO) attention matrices
        % and feed-forward networks (FFNs). We demonstrate empirically that this advantage
        % translates to dramatically improved learning of rare, tail-end facts
        % from heavy-tailed distributions, a bottleneck in knowledge acquisition
        % for large language models (LLMs). Our theoretical analysis of a one-layer
        % memory model confirms this, showing that Muon's spectral update rule ensures
        % balanced learning across facts of varying frequencies, whereas Adam's
        % performance is unstable and sensitive to data structure. We conclude that
        % Muon's fundamental alignment with the outer-product structure of associative
        % memories makes it exceptionally well-suited for the challenges of tail-end
        % learning in modern LLMs.
    \end{abstract}

    \section{Introduction}
The effectiveness of Adam~\citep{kingma2014adam} across diverse training scenarios
has made it one of the most widely used optimizers for neural networks, serving
as a cornerstone of the tremendous successes of Large Language Models (LLMs). Building
on this foundation, Muon~\citep{jordan6muon} has emerged as a matrix-parameter optimizer
designed to surpass Adam. Empirical studies show that Muon is nearly $2$ times
faster than Adam across a wide range of model sizes and architectures~\citep{liu2025muon,jordan6muon}.
Its key innovation is to replace the raw gradient with the sum of its normalized
orthogonal factors, which can be interpreted as performing steepest descent with
respect to the spectral norm~\citep{bernstein2024old}. \looseness=-1%\textcolor{red}{replace operator norm with ``the maximum singular value''.}

However, despite its empirical success, a rigorous understanding of why and how Muon
outperforms Adam in transformers remains incomplete. In particular, the steepest
gradient descent interpretation %\mh{this steepest gradient descent interpretatin is a bit vague, not sure if should be mentioned here.} \shuche{We add a detailed introduction to this steepest descent view in Appendix B. But we are not sure whether we can put an Appendix link at the beginning of the Introduction.}
does not clarify why optimization with respect to the spectral norm, as in Muon,
should outperform optimization with respect to the infinity norm (for vectors), as in Adam. %\textcolor{red}{note that there are many operator norms and many vector norms. the use of the word ``the'' here suggests that there's only one}. 
Consequently, convergence analyses of Muon derived from this interpretation fail to account for its observed superiority
over Adam~\citep{li2025note,shen2025convergence}.

% While Adam~\citep{kingma2014adam}
% has long been the effective optimizer, recent specialized optimizers for matrix
% parameters, especially Muon~\citep{jordan6muon}, have demonstrated superior performance,
% consistently outperforming Adam across a range of model sizes and architectures~\citep{liu2025muon,jordan6muon}.

This paper takes the first step toward understanding the mechanisms underlying
Muon’s superiority over Adam in training LLMs. Specifically, we first ask the question:
\begin{quote}
    \textit{Which transformer components benefit most from Muon’s matrix-norm–based
    optimization compared to Adam?} \vspace{-.5em}
    %\mh{how about this:

    %     (1) Which transformer components benefit most from Muon’s matrix-norm–based optimization compared to Adam?

    % (2) What structural features of the transformer allow Muon to optimize these components more effectively? }%\mh{It is not clear what question has been asked here, also not clear why directly asking the specific question about components of the transformer architecture. I would ask slightly higher-level questions, such as, 'Is Muon’s performance gain over AdamW universal, or is it amplified by the specific design of the Transformer architecture?'}\fengzhuo{previous experiments https://github.com/KellerJordan/cifar10-airbench/tree/master?tab=readme-ov-file, show that Muon can also accelerate CNN. Maybe it is better to constrain ourselves to LLM, since the mechanism for CNN may be different. Concretely, no existing papers report there is inherent associative memory in CNN. Is it better to make this question more concrete with this constraint?}
\end{quote}
To address this question, we apply Muon to different transformer components. Our experiments consistently show
that the more rapid convergence of the validation loss using the Muon optimizer compared
to Adam is primarily due to the former's focus on the value-output (VO) matrices
of the attention mechanism and the \ac{ffn} blocks. % \textcolor{red}{VT modified: Our findings consistently show that the more rapid convergence of the validation loss using the Muon optimizer compared to  Adam is predominantly due to the former's focus on  the value-output (VO) matrices of the attention mechanism and the \ac{ffn} blocks} \mh{this sentence is a bit long and confusing. 'loss convergence'... concentrated on .. matrices?}.
This leads to our first key insight: VO and \ac{ffn} blocks, which serve as the primary
associative memory stores in the model \citep{geva2020transformer,bietti2023birth},
are the main beneficiaries of Muon’s optimization strategy. This naturally raises the following question:
\begin{quote}
    \textit{What structural features of the transformer allow Muon to
    optimize these components more effectively?}
\end{quote}

Building on the previous finding, we address this question by linking Muon’s update mechanism
to the learning dynamics of associative memory. Prior work suggests that the
behavior of these memory components can be modeled as a sum of outer products
representing stored facts~\citep{meng2022locating}. Since Muon's update
assigns equal update magnitudes to each outer product  of the gradient corresponding to orthogonal singular directions, %Since Muon's update
%attributes the same magnitudes \textcolor{red}{same magnitudes of what??} to the
%outer products with orthogonal vectors, 
we hypothesize that it optimizes associative memories more effectively than Adam because: (i) Muon’s spectral normalization procedure balances the rates of learning
of these outer products. (ii) Thus, when training on heavy-tailed data
(i.e., where a few facts appear much more frequently than the rest),  %\mh{i.e.,....},
Muon reduces %\textcolor{red}{``minimizes'' or ``reduces'' is better. you ``mitigate'' a problem. } 
the dominance of frequent (head) facts and enables more effective learning
from infrequent (tail) facts compared to Adam.

% Our second finding is that \textbf{Muon achieves a more stable and balanced spectrum of parameters than Adam.}

% We validate these hypotheses through a combination of empirical analysis and
% theoretical modeling. First, we conduct a spectral analysis of the weight matrices and
% show that Muon consistently yields more isotropic representations than Adam, indicating
% that its normalization prevents spectral energy from concentrating in
% dominant components. Second, we conduct experiments on a knowledge-intensive, heavy-tailed task highlight the practical benefit: while both optimizers perform well on the head
% classes (the classes with large probabilities), Muon outperforms Adam on tail classes (the classes with small probabilities), achieving robust and uniform
% convergence. %\mh{not clear what head classes and tail classes are referred to here.}

We validate these hypotheses through a combination of empirical analysis and theoretical
modeling.
%\textcolor{red}{somehow the structure of the introduction is strange. you talk about
%experiments in the two paras above. then you talk about experiments below again?
%can these discussions be somehow merged more seamlessly?} 
Empirically, we conduct two experiments. First,  we   measure the singular value spectra of the weight matrices and show that
Muon consistently yields more isotropic representations than Adam, indicating that
its normalization prevents spectral energy from concentrating in dominant components.
Second, we evaluate the performance of both optimizers %\textcolor{red}{you evaluate the
%{\bf performance} of both optimizers!} 
on a knowledge-intensive, heavy-tailed task to demonstrate the practical benefit of Muon’s more balanced updates: while both optimizers perform well on head classes (frequent in
training data), Muon outperforms Adam on tail classes (rare in training data), leading
to more stable and uniform convergence. %\mh{not clear what head classes and tail classes are referred to here.}

Theoretically, we focus on a one-layer linear associative memory model
to rigorously explain these empirical findings. Under class imbalance in the
training data, mimicking a heavy-tailed distribution, we show that Muon
maintains balanced learning across classes, regardless of the feature embeddings.
In contrast, we prove that Adam’s performance is unstable and strongly dependent
on the embedding structure, which can lead to large disparities in learning
error across classes. By closely examining the parameter updates, we find that
the singular spectrum of weight matrices trained by Muon is nearly isotropic, whereas Adam’s is uneven. \looseness=-1

Summarizing  the empirical and theoretical findings, we identify a clear mechanism underlying Muon’s superiority: 
\begin{olivebox}
    The Muon update rule is aligned with the outer-product structure of linear associative memories, enabling more balanced and effective learning of tail classes in heavy-tailed distributions as compared with Adam.
\end{olivebox}

 %\mh{as compared with adam? or GD?}

%\mh{I think the intro is still a bit confusing, as it is not clear the tone is more comparative, showing muon is supervior than adam, or it is more trying to understand why muon }

% \textbf{Notations:} Let $[N]$ for the set $\{1,\dots,N\}$. For a matrix $X \in \mathbb{R}^{d \times N}$, $X_i$ is its $i$-th column and $X_{:,-1}$ is its last column. $I_{K,K}$ is the $K \times K$ identity matrix, and $J_{K,K}$ is the all-ones matrix. $\odot$ denotes the element-wise product. \fengzhuo{which is necessary for the main paper}
    \section{Related Works}\label{sec:related_works}
Adam, proposed by \cite{kingma2014adam}, was designed to make \ac{gd} adaptive to the complex optimization landscape of neural networks. Existing works analyze Adam from two primary perspectives: online optimization and feature learning. The online convex optimization view focuses on Adam's properties when optimizing convex or non-convex loss functions. From this perspective, \cite{chen2018convergence} and \cite{zhou2018convergence} derive non-convex convergence results for Adam, and a series of subsequent works continuously relaxed the required assumptions for Adam's convergence while tightening its convergence rate. For instance, %\cite{reddi2019convergence} first identifies Adam's divergence behavior, even in online convex optimization problems with quadratic loss functions, and subsequently proposed the improved version, AMSGrad. 
\cite{zou2019sufficient} proposes a set of easy-to-verify sufficient conditions for Adam's update rules to guarantee convergence. \cite{defossez2020simple} derives the tightest dependency on the heavy ball momentum parameters. More recently, \cite{zhang2022adam} demonstrates that Adam can converge without modification of its procedures, and \cite{li2023convergence} relaxes the smoothness assumption by employing an adaptive Lipschitz constant for gradients. The feature learning view, on the other hand, highlights the relationship between deep learning characteristics and Adam, focusing more on how Adam's mechanisms influence the properties of learned features within deep networks. For example, \cite{pan2023toward} examines the sharpness of \ac{gd} and Adam and relates Adam's superiority to its low sharpness. \cite{kunstner2024heavy} finds that Adam is better at learning heavy-tailed distributions than \ac{gd}. Furthermore, \cite{zhang2024transformers} shows that Adam is adaptive to heterogeneous Hessian structures, thus optimizing faster than \ac{gd}. More literature on Adam is included in the survey by \cite{abdulkadirov2023survey}.

Muon, proposed by \cite{jordan6muon}, applies spectral normalization of the gradient to update parameters. At a high level, Muon can be understood as steepest descent with respect to the matrix operator norm~\citep{bernstein2024old}. Alternatively, it can be viewed as maximizing the feature update subject to a parameter update constraint~\citep{yang2023spectral}. Experiments show that Muon consistently outperforms Adam across diverse model sizes and architectures, including dense transformers and Mixture-of-Experts~\citep{liu2025muon,jordan6muon}. Building on this, \cite{si2025adamuon} introduces an adaptive variant of Muon. To explain its advantages, \cite{lau2025polargrad} introduces a unifying preconditioning framework, distinguishing optimizers that address curvature anisotropy (like Adam) from those that address gradient anisotropy (like Muon), and proposes a generalized optimizer class named PolarGrad. \cite{sato2025analysis} and \cite{shah2025practical} examine the critical batch size of Muon, while other works analyze its convergence in convex and non-convex settings~\citep{li2025note,an2025asgo,kovalev2025understanding,pethick2025training,shen2025convergence}. Concurrently, \cite{vasudeva2025generalization} study Muon on shallow ViTs for computer vision, grounding their results for gradient descent and Muon in linear regression. In contrast, we investigate Muon in the context of LLMs, focusing on its effects on associative memory in next-token prediction.

Associative Memories have a long history in neural network design and knowledge storage~\citep{hopfield1982neural,kohonen2009correlation,willshaw1969non}. They have inspired architectures capable of retaining long histories, including RNNs~\citep{orvieto2023resurrecting} and Mamba~\citep{zhang2024motion}. With the success of transformers, recent work has examined them through the lens of associative memories. \cite{geva2020transformer} and \cite{dai2021knowledge} show that feed-forward modules store knowledge in $W_\text{out}$, while \cite{bietti2023birth} demonstrates that the attention output matrix $W_O$ also encodes associations of knowledge. Building on these findings, a series of works edit knowledge directly by modifying these weights~\citep{meng2022mass,fang2024alphaedit}. Beyond empirical results, theoretical analyses have further clarified how transformers leverage associative memories: \cite{bietti2023birth} conducts a dynamic analysis of memory formation, while \cite{nichani2024understanding} constructs explicit associative memory mechanisms in both attention and feed-forward modules.

    %\vspace{-.1in}
\section{Preliminaries}
\label{sec:prelim}%\vspace{-1em}
% Adam \citep{kingma2014adam} (and AdamW \citep{loshchilov2017decoupled}) is the default optimizer before the emergence of more recent optimizers such as Shampoo and Muon. It is shown that Adam can converge with a proper set of momentum hyperparameters $(\beta_1,\beta_2)$ \citep{chen2018convergence}, however, for the default hyperparameter setting $(\beta_1,\beta_2)=(0.9,0.999)$, there exists a counterexample showing the divergence of Adam \citep{reddi2019convergence}. It is later shown that Adam could converge to a stationary point if $\beta_2$ is large enough (problem-dependent) \citep{zhang2022adam}.

In this section, we first introduce the notations and then present the Muon optimizer, the transformer architecture, and their associative memory components.

\textbf{Notations.} Let $[N]$ for the set $\{1,\dots,N\}$. For a matrix $X \in \mathbb{R}^{d \times N}$, $X_i$ is its $i$-th column and $X_{:,-1}$ is its last column. $I_{K,K}$ is the $K \times K$ identity matrix, $\bbI_{K}$ is all-ones vector and $J_{K,K}$ is the all-ones matrix. $\odot$ denotes the element-wise product. 

\textbf{Muon} is an optimizer tailored for matrix parameters
that replaces the raw (or momentum) gradient with the sum of its \emph{normalized
orthogonal factors}, producing a scale-invariant, norm-controlled update direction~\citep{jordan6muon}.
For a weight matrix $W\in\mathbb{R}^{m\times n}$ at step $t$, let $G_{t}=\nabla_{W}
\mathcal{L}(W_{t})$ denote its gradient. Muon maintains a momentum accumulator
of gradients as
$B_{t} = \mu B_{t-1}+ G_{t}\text{ with }B_{0} = 0, \text{ and }\mu\in[0,1).$ At
each step, Muon computes the \ac{svd} of $B_{t}$ as
$B_{t} = U_{t} S_{t} V_{t}^{\top} \text{ with }U_{t}\in\mathbb{R}^{m\times r_t},\;
V_{t}\in\mathbb{R}^{n\times r_t},$
where $r_{t} = \operatorname{rank}(B_{t})$, and form the nearest (semi)–orthogonal
matrix $O_{t} = U_{t} V_{t}^{\top}$. Then Muon updates the parameter as $W_{t+1}=
W_{t} - \eta_{t} O_{t}.$ In practice, one can approximate $O_{t}$ using a fixed
number (e.g., $5$) of Newton--Schulz iterations applied to $B_{t}(B_{t}^{\top} B_{t}
)^{-1/2}$, which avoids the full SVD while preserving the scale normalization
effect. \cite{bernstein2024old} interprets Muon as steepest gradient descent with respect to the matrix operator norm. Concretely, the Muon update $O_t$ can be characterized (up to a scalar factor) as the solution to
\begin{equation*}
    \argmin_{W} \left[ \langle B_t, W\rangle_F + \frac{\lambda}{2} \|W\|_{\ell_2 \to \ell_2}^2 \right],
\end{equation*}
where $\|\cdot\|_{\ell_2 \to \ell_2} $ denotes the matrix operator norm, i.e., the largest singular value, and $\lambda \in \mathbb{R}$ determines the step size. By contrast, as explained in Appendix~\ref{app:dgd_understand}, Adam can be viewed as steepest gradient descent with respect to the vector norm. However, this perspective alone does not explain why using the matrix operator norm rather than the vector norm leads to better performance. %A more detailed introduction to Muon is provided in the related works section (Appendix~\ref{sec:related_works}).

\textbf{Transformers} serve as the backbone of LLMs. It predicts the probability of
the next token given a sequence of $N$ tokens~\citep{radford2019language}. %Each token is embedded into a
%vector in $\mathbb{R}^{d}$, and collecting all $N$ token embeddings yields the matrix
%X^{(0)}\in \mathbb{R}^{d \times N}$. 
A sequence of $N$ tokens is embedded into a matrix $X^{(0)}\in\mathbb{R}^{d\times N}$. The first layer takes $X^{(0)}$ as the
input, and each subsequent layer takes the previous layer's output as its input.
Every layer $\ell\in[L]$ processes its input through two sequential components: an
attention module and a \ac{ffn} module. The attention module computes
\begin{align}
	H^{(\ell)} %&= X^{(\ell-1)}+\sum_{h=1}^{H}\att \big(X^{(\ell-1)},W_{Q,h}^{(\ell)},W_{K,h}^{(\ell)},W_{V,h}^{(\ell)},W_{O,h}^{(\ell)}\big)\nonumber\\
	          & =X^{(\ell-1)}+\sum_{h=1}^{H}W_{O,h}^{(\ell)}W_{V,h}^{(\ell)}X^{(\ell-1)}\sm\big(X^{(\ell-1),\top}W_{K,h}^{(\ell),\top}W_{Q,h}^{(\ell)}X^{(\ell-1)}\big),\label{eq:attn}
\end{align}
where $\sm(\cdot)$ is the column-wise softmax operator, $H$ is the number of
attention heads, $W_{Q,h}^{(\ell)},W_{K,h}^{(\ell)}\in\mathbb{R}^{d_k\times d}$ capture
token relationships, and
$W_{V,h}^{(\ell)}\in\mathbb{R}^{d_v\times d},W_{O,h}^{(\ell)}\in\mathbb{R}^{d\times
d_v}$
apply linear transformations. The feed-forward module then updates the representation
as %\emph{the notation $H^{(\ell)}$ here overlaps with H?}
\begin{align}
	X^{(\ell)}= H^{(\ell)}+ \ff(H^{(\ell)}, W_{\inn}^{(\ell)},W_{\out}^{(\ell)})=H^{(\ell)}+ W_{\out}^{(\ell)}\sigma(W_{\inn}^{(\ell)}H^{(\ell)}),\label{eq:ff}
\end{align}
where $\sigma(\cdot)$ is the element-wise activation function, and $W_{\inn}^{(\ell)}
\in\mathbb{R}^{d_f\times d}, W_{\out}^{(\ell)}\in\mathbb{R}^{d\times d_f}$ are learnable
parameters. In addition to the \ac{ffn} in Eqn.~\eqref{eq:ff}, a gated variant is widely used in \ac{llm}s~\citep{touvron2023llama,hui2024qwen2}, which replaces the standard form with
\begin{align*}
	\ff_{\gate}(H^{(\ell)}, W_{\inn}^{(\ell)},W_{\out}^{(\ell)}, W_{\gate}^{(\ell)})= W_{\out}^{(\ell)}\big(\sigma(W_{\inn}^{(\ell)}H^{(\ell)})\odot (W_{\gate}^{(\ell)}H^{(\ell)}) \big),
\end{align*}
where $\odot$ is the Hadamard product, and $W_{\gate}^{(\ell)}\in\mathbb{R}^{d_f\times
d}$ is an additional mapping. After $L$ layers, the final hidden state of the last
token, $X_{-1}^{(L)}$, is projected by the language model head $E_{\head}\in\mathbb{R}
^{K\times d}$ to produce logits $E_{\head}X_{-1}^{(L)}$, which has a vocabulary
of size $K$.

\textbf{Associative memory} refers to architectures that store and retrieve patterns
based on learned associations between inputs and outputs. Recent research has
examined \emph{linear} associative memory in LLMs. Specifically, consider a triplet
$(s,r,o)$, where $s$ is the subject, $r$ the relation, and $o$ the object (e.g.,
$s=$``The United Nations headquarters'', $r=$``is located in'', $o=$``New York
City''). A linear associative memory $W$ maps a key vector $e_{s}$ encoding
$(s,r)$ to a value vector $e_{o}$ encoding $o$, such that $e_{o} = W e_{s}$ holds
for all possible $(s,r,o)$~\citep{nichani2024understanding}. Under the orthogonality of embeddings $e_{s}$ and
$e_{o}$, $W$ can be expressed as $W=\sum_{i=1}^{K}e_{o_i}e_{s_i}^{\top}$, where the
summation is taken over the index $i$ of $K$ facts. Prior work has investigated associative
memory in both attention and \ac{ffn} modules. For the attention module, \citet{bietti2023birth}
showed that the parameter $W_{O}$ can serve as a linear associative memory when $W
_{V}$ is fixed. Since $W_{O}$ and $W_{V}$ play symmetric roles, we also treat $W_{V}$
as part of the associative memory parameters. For \ac{ffn}, prior work on knowledge editing~\citep{geva2020transformer,dai2021knowledge,meng2022locating,meng2022mass} has shown that this module functions as an associative memory and can be well approximated by linear associative memory models. Thus, throughout this
paper, we refer to $W_{O}$, $W_{V}$, and \ac{ffn} in LLMs as the
{\em associative memory parameters}.

    \section{Main Results}
\label{sec:empirical}

%\emph{it is strange to start the paper with experiment results. }

\subsection{Associative Memories Are Main Beneficiaries of Muon}
\label{sec:am}

In the Muon implementation~\citep{jordan6muon}, the token embedding and language model head parameters are optimized with Adam rather than Muon. This observation motivates a closer examination of the efficacy of Muon across different components of the transformer architecture. In this section, we identify the transformer components that benefit most from Muon by measuring validation loss on the FineWeb dataset~\citep{penedo2024fineweb} using a 160M NanoGPT
model. We adopt a two–stage protocol. First, in the ``Independent Blocks'' setting,
we apply Muon to a single block at a time while keeping all other blocks on Adam,
covering the attention projections $W_{Q}, W_{K}, W_{V}, W_{O}$ and the feed-forward
matrices $W_{\inn}$, $W_{\out}$. Second, in the ``Combined Configurations'' setting,
we apply Muon to the most impactful subsets identified in the first stage to
examine whether a partial application can recover the performance gains of full
Muon. As introduced in Section~\ref{sec:prelim}, we evaluate both gated and non-gated
\ac{ffn} variants of NanoGPT. The experimental details are in Appendix~\ref{app:exp_details}.

\begin{figure}[t]
    \centering
    \subfigure[Independent Blocks with Non-gated \ac{ffn}]{ \includegraphics[width=0.48\textwidth]{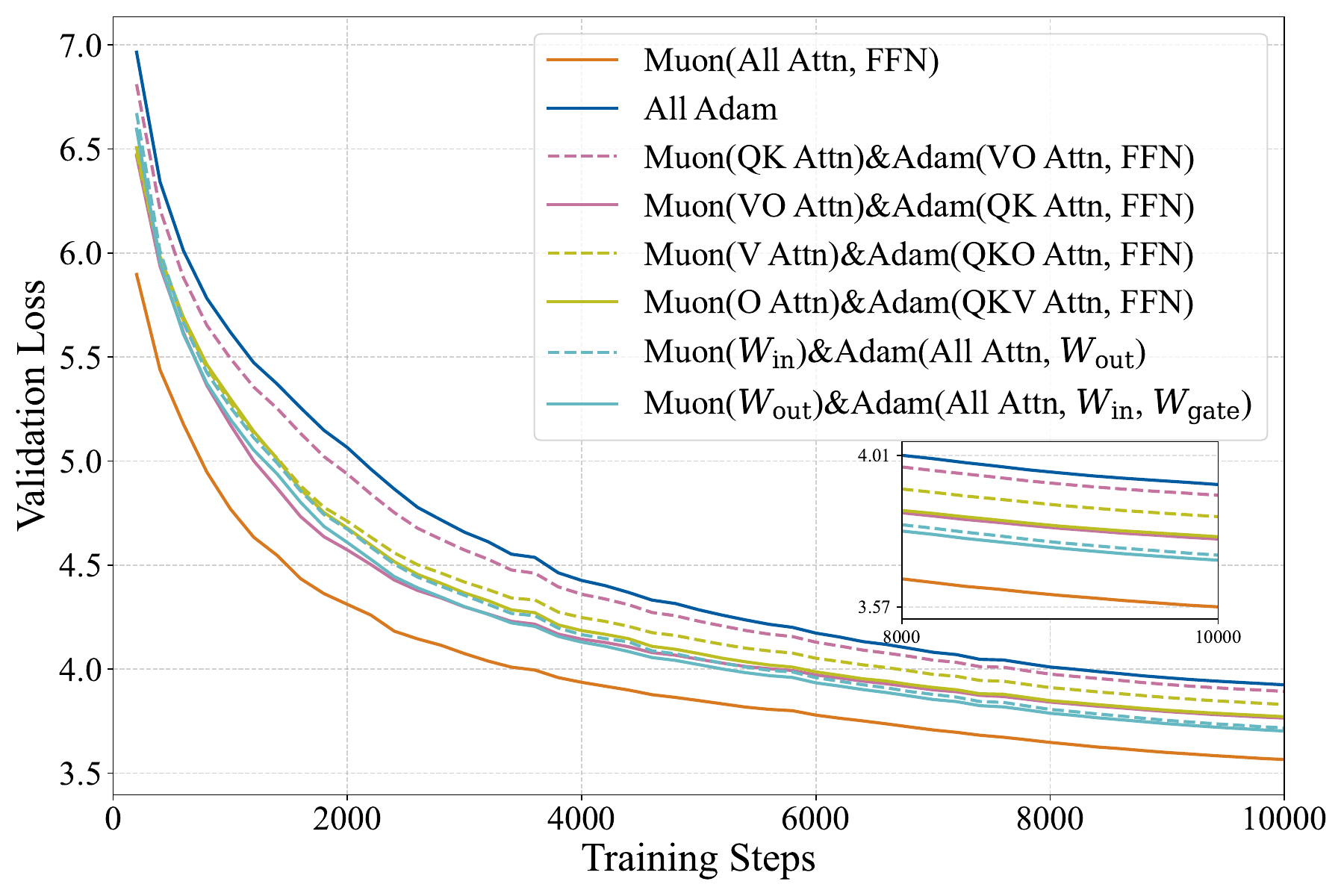}\label{fig:non_indep}}
    \subfigure[Independent Blocks with Gated \ac{ffn}]{ \includegraphics[width=0.48\textwidth]{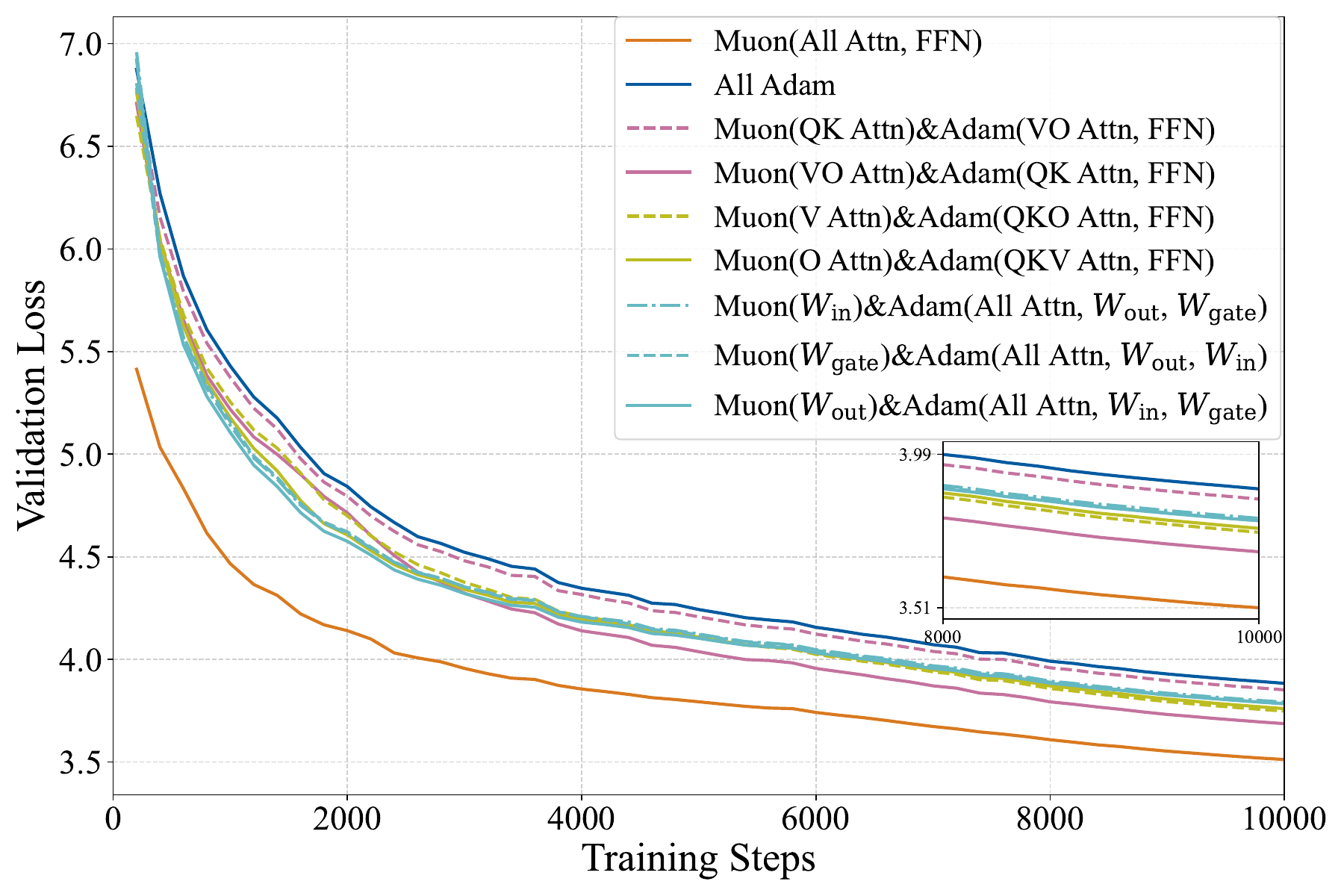}\label{fig:gate_indep}}%\\[-0.3cm]

    % \caption{Validation loss comparison on the 160M NanoGPT model. (a) Independent
    % blocks, where different components are optimized separately. (b) Combined configuration,
    % where multiple components are optimized jointly.}
    % \label{fig:fineweb_160m_nogated}

    \centering
    \subfigure[Combined Configuration with Non-gated \ac{ffn}]{ \includegraphics[width=0.48\textwidth]{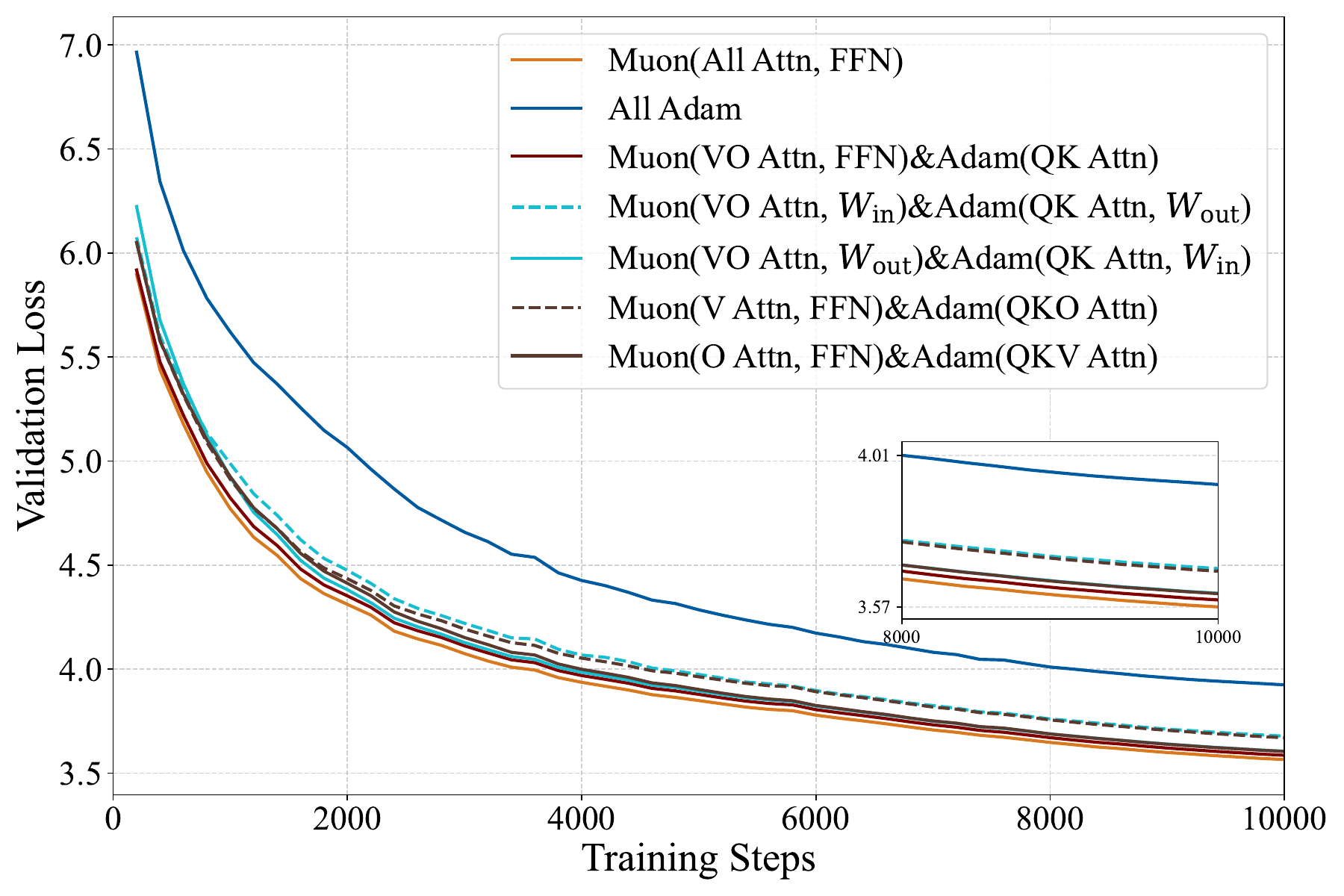}\label{fig:non_comb}}
    \subfigure[Combined Configuration with Gated \ac{ffn}]{ \includegraphics[width=0.48\textwidth]{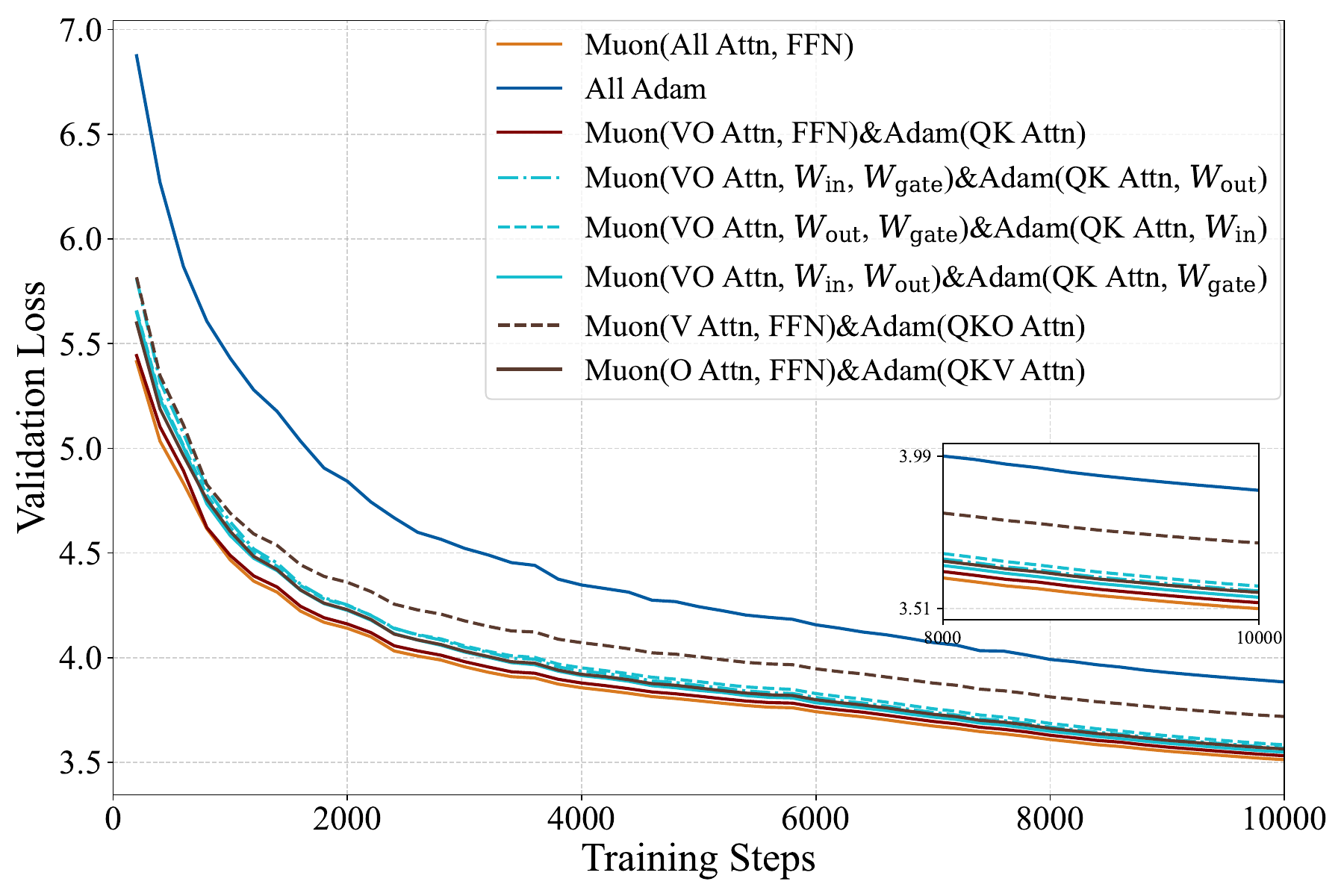}\label{fig:gate_comb}}
    \caption{Validation loss comparison on the 160M NanoGPT model with ungated
    and gated \ac{ffn}. Panels (a) and (b) show the ``Independent Blocks''
    results, where individual components are optimized separately, for models
    with ungated and gated \ac{ffn}, respectively. Panels (c) and (d) show the ``Combined
    Configurations'' results, where multiple components are optimized jointly, again
    for ungated and gated \ac{ffn} models.}
    % \label{fig:fineweb_160m_gated}
    \label{fig:fineweb_160m}
\end{figure}

\begin{table}[ht!]
\centering
\caption{Validation loss at 10000 training steps. The order of methods follows the legends in the original figure. The best result within each configuration block is highlighted in \textbf{bold}. A dash (—) indicates the configuration was not compatible for that FFN type.}
\label{tab:valloss_summary}
\sisetup{detect-weight=true, detect-family=true} 
\begin{tabular}{@{} l S[table-format=1.4] S[table-format=1.4] @{}}
\toprule
\textbf{Method Description} & {\textbf{Non-gated FFN}} & {\textbf{Gated FFN}} \\
& {\small(from Fig.~\ref{fig:non_indep} and \ref{fig:non_comb})} & {\small(from Fig.~\ref{fig:gate_indep} and \ref{fig:gate_comb})} \\
\midrule
{\textit{\textbf{Baselines}}} \\
Muon(All Attn, FFN) & \textbf{3.5654} & \textbf{3.5125} \\
All Adam & 3.9242 & 3.8837 \\
\midrule

% --- Independent Blocks Section ---
{\textit{\textbf{Independent Blocks}}} \\
Muon(QK Attn)\& Adam(VO Attn, FFN) & 3.8925 & 3.8518 \\
Muon(VO Attn)\& Adam(QK Attn, FFN) & 3.7644 & \textbf{3.6874} \\
Muon(V Attn)\& Adam(QKO Attn, FFN) & 3.8301 & 3.7482 \\
Muon(O Attn)\& Adam(QKV Attn, FFN) & 3.7712 & 3.7604 \\
Muon($W_{\inn}$)\& Adam(All Attn, $W_{\out}$) & 3.7170 & {---} \\
Muon($W_{\out}$)\& Adam(All Attn, $W_{\inn}$, $W_{\gate}$) & \textbf{3.7023} & 3.7843 \\
Muon($W_{\inn}$)\& Adam(All Attn, $W_{\out}$, $W_{\gate}$) & {---} & 3.7918 \\
Muon($W_{\gate}$)\& Adam(All Attn, $W_{\out}$, $W_{\inn}$) & {---} & 3.7847 \\
\midrule

% --- Combined Configuration Section ---
{\textit{\textbf{Combined Configuration}}} \\
Muon(VO Attn, FFN)\& Adam(QK Attn) & \textbf{3.5858} & \textbf{3.5312} \\
Muon(VO Attn, $W_{\inn}$)\& Adam(QK Attn, $W_{\out}$) & 3.6778 & {---} \\
Muon(VO Attn, $W_{\out}$)\& Adam(QK Attn, $W_{\inn}$) & 3.6054 & {---} \\
Muon(VO Attn, $W_{\inn}$, $W_{\gate}$)\& Adam(QK Attn, $W_{\out}$) & {---} & 3.5681 \\
Muon(VO Attn, $W_{\out}$, $W_{\gate}$)\& Adam(QK Attn, $W_{\inn}$) & {---} & 3.5833 \\
Muon(VO Attn, $W_{\inn}$, $W_{\out}$)\& Adam(QK Attn, $W_{\gate}$) & {---} & 3.5482 \\
Muon(V Attn, FFN)\& Adam(QKO Attn) & 3.6702 & 3.7185 \\
Muon(O Attn, FFN)\& Adam(QKV Attn) & 3.6042 & 3.5634 \\
\bottomrule
\end{tabular}
\end{table}

Figure~\ref{fig:fineweb_160m} and Table~\ref{tab:valloss_summary} present our results. We first examine the
independent-block experiments for attention. In both non-gated and gated \ac{ffn}
settings (Figures~\ref{fig:non_indep} and~\ref{fig:gate_indep}), the VO weights
$W_{V}, W_{O}$ (Muon on VO \& Adam on QK and FFN) show substantially larger gains
under Muon than the QK weights $W_{Q}, W_{K}$ (Muon on QK \& Adam on VO and
FFN). Notably, applying Muon to only $W_{V}$ or only $W_{O}$ already yields much
larger gains than applying it to QK. Between the two, $W_{O}$ performs comparably
in the gated \ac{ffn} setting and better in the non-gated setting. For the \ac{ffn},
we find that $W_{\inn}$, $W_{\gate}$, and $W_{\out}$ all benefit from Muon, with
$W_{\out}$ yielding stronger improvements than $W_{\inn}$. 
% {\mh{[I am curious if this is just a stepsize issue. Do we have experiments where, for a give block in question, we compare (1) switch to muon; (2) increase Adam learning rate.][I guess my question is really whether the muon's spectral normalization is the key, or just 'some' normalized learning rate is useful.]}}\fengzhuo{This is a good question. First, given our current work, the results are logical, since we search and use the best LRs for both optimizers. Second, for your question, Shuche, Jiaxiang, and I have another ongoing work focusing on the landscape, i.e., Hessian or sharpness. Currently, we have the following points for this ongoing work: 1. Muon's update is much larger than Adam's. If we force Adam's LR to be as large as Muon, it will just not converge. 2. The sharpness of Muon is much smaller than that of Adam, potentially admitting the large stepsize. 3. The sharpness of both optimizers is influenced by the heterogeneity of the training data (potentially consistent with the heavy-tailed experiments in this paper). For example, the sharpness on the data with multiple topics is larger than that of the data with a single topic. If you are interested, we can discuss this ongoing work after ICLR. I am not sure about the argument of normalized LR, since it is impractical to search LR for each layer. We currently maintain the same LR for the whole network.}

After identifying the importance of each module, the combined configurations aim
to quantify their contributions to the full Muon. Guided by the independent-block
findings, we first observe that VO+\ac{ffn} already closely tracks—and in our
runs nearly recovers—the full-Muon trajectory in Figures~\ref{fig:non_comb} and~\ref{fig:gate_comb}.
This indicates that applying Muon to QK contributes little to its overall performance. The small remaining gap between full Muon and VO+\ac{ffn} may arise because VO+\ac{ffn} uses the same learning rate as full Muon without further tuning. This gap could likely be reduced by adjusting the learning rate specifically for VO+\ac{ffn}. Importantly, the underperformance of QK is not attributable to the logit explosion observed by \citet{team2025kimi} in large Mixture-of-Experts (MoE) models; in our setting, logit values remain stable, as shown in Appendix~\ref{app:max_logit}.

To isolate the contributions of $W_{O}$ and $W_{V}$ within VO+\ac{ffn}, we perform
ablations starting from the VO+\ac{ffn} setting: we keep Muon on \ac{ffn} and on
only one of $W_{O}$ or $W_{V}$, reverting the other to Adam (i.e., V+\ac{ffn} and
O+\ac{ffn}). Both ablations degrade performance, with the V+\ac{ffn} variant dropping more,
indicating that $W_{O}$ is more influential than $W_{V}$. We apply the same analysis
to \ac{ffn}. The results reveal architectural sensitivity: in the ungated
setting (Figure~\ref{fig:non_comb}), VO+$W_{\out}$ nearly recovers the full-Muon
trajectory, whereas in the gated setting (Figure~\ref{fig:gate_comb}) the same
combination falls short. Nevertheless, both analyses underscore the central role
of $W_{\out}$ in \ac{ffn}. Overall, applying Muon to VO+\ac{ffn} is critical for
recovering full-Muon performance, though the extent of recovery still depends on
architectural design (ungated vs gated). The results from training a $0.7$B model
in Appendix~\ref{app:scaleup} show similar findings.

% Across the independent‑block results in both non-gated and gated \ac{ffn} settings (Figures~\ref{fig:non_indep}
% and~\ref{fig:gate_indep}), the VO parameters ($W_V,W_O$)
% exhibit substantially larger gains under Muon than the QK parameters ($W_Q,W_K$).
% which show markedly weaker effects. Within the \ac{ffn}, applying Muon to either
% $W_{\inn}$ or $W_{\out}$ produces similar behavior.

% The combined configurations then aim to recover the full‑Muon curve using as few Muon parameters as possible. Guided by the independent‑block findings, we observe: (i) VO+\ac{ffn} already closely tracks—and in our runs can recover—the full‑Muon trajectory; (ii) combining VO with \ac{ffn} sub‑matrices reveals architectural sensitivity: in the non‑gated setting (Figure~\ref{fig:fineweb_160m_nogated}(b)), VO+$W_{\out}$ nearly recovers the full‑Muon trajectory, whereas in the gated setting (Figure~\ref{fig:fineweb_160m_gated}(b)) the same combination does not fully recover; VO+$W_{\inn}$ and VO+$W_{\gate}$ improve over All‑Adam in both settings but do not recover; (iii) using \ac{ffn} with $W_{O}$ or $W_{V}$ alone is less reliable and often fails to match full‑Muon. Overall, VO+\ac{ffn} is a simple and strong choice. It works well across settings, though results still depend on the architecture (non‑gated vs gated). This supports the view that VO and \ac{ffn} carry most of the model’s memory.

\begin{olivebox}
    {\bf Observation 1:} Muon is most effective when applied to VO and \ac{ffn};
    in particular, applying Muon to only VO+\ac{ffn} almost recovers the full-Muon trajectory.
\end{olivebox}
We emphasize that this observation is not a trivial consequence of parameter counting; although QK and VO have the same number of parameters, VO proves substantially more influential in our results.

As introduced in Section~\ref{sec:prelim}, prior works discover that the common
role of VO and \ac{ffn} is that they both serve as the associative memories for transformers,
which store facts and knowledge. Furthermore, \cite{bietti2023birth} and
\cite{meng2022locating} show that the linear
associative memories well approximate them. Specifically, for a set of facts represented by key-value
pairs $\{(e_{s_i}, e_{o_i})\}$, the memory matrix $W$ can be constructed as a sum
of outer products, i.e., $W = \sum_{i=1}^{K}e_{o_i}e_{s_i}^{\top}$, where the
summation is taken over the index $i$ of $K$ facts. \looseness=-1 
%Specifically, a linear associative memory stores key-value pairs $\{(e_{s_i}, e_{o_i})\}$ as a sum of outer products in a matrix $W = \sum_{i}e_{o_i}e_{s_i}^{\top}$. Querying with a key $e_{s_j}$ yields $W e_{s_j} = \sum_{i}e_{o_i}(e_{s_i}^{\top}e_{s_j})$. If the keys $\{e_{s_i}\}$ are orthogonal, this sum collapses to the corresponding value $e_{o_j}$, enabling content-based retrieval.

Learning linear associative memories is particularly well-suited to Muon’s
update mechanism. Intuitively, the gradient $G\in\bbR^{d\times d}$ of the loss with respect to the
linear associative memory weight $W$ can be expressed as a sum of outer
products. Muon computes its update (without momentum) by taking the \ac{svd} of
the gradient, $G = U S V^{\top}= \sum_{i=1}^{d}s_{i}u_{i}v_{i}^{\top}$, and forming
the orthogonal factor $O = U V^{\top}= \sum_{i=1}^{d} u_{i}v_{i}^{\top}$. Comparing this
with the linear associative memory $\sum_{i=1}^{K}e_{o_i}e_{s_i}^{\top}$, we see that
Muon updates all ``orthogonal'' facts at the same rate. Later, we will see that
the singular values $S$ encode the frequencies of knowledge in the training data
in Sections~\ref{sec:knowledge} and \ref{sec:theory}. This implies that Muon can
learn both frequent and infrequent facts uniformly.

%We verify this insight from two perspectives. \mh{[what does 'parameter' mean here? 'weights'?]}
%First, from the perspective of parameter spectra, the parameters learned by Muon
%should exhibit a more isotropic singular spectrum than those learned by Adam, indicating
%that different knowledge {\mh{[different knowledge is confusing][should be more specific, like `regardless of its frequency, knowledge should be captured...'],}}
%is captured with comparable magnitude. Second, from the perspective of overall knowledge
%acquisition, Muon should achieve more balanced learning across knowledge than
%Adam. We examine these two consequences in the following sections.

We verify this insight from two perspectives. First, from the view of
weight spectra, the weight matrices learned with Muon exhibit a more isotropic
singular-value spectrum than those learned with Adam, indicating that
knowledge, regardless of its frequency, is represented with comparable magnitude. Second,
at the level of overall knowledge acquisition, Muon yields more balanced learning
across entities and frequencies (head and tail) than Adam. We examine these two
consequences in the following sections.

% This insight has two key implications that we will investigate empirically.
% First, by preventing a few dominant outer products from controlling the parameter matrix,
% Muon should encourage a more isotropic structure where spectral energy is
% distributed more evenly. This motivates a spectral analysis of the learned weight
% matrices. Second, by giving more uniform treatment to both high- and low-frequency
% facts, Muon should be particularly effective in learning from heavy-tailed,
% knowledge-intensive data distributions. We explore these two consequences in the
% following subsections.%, respectively.

\subsection{Muon Consistently Learns More Isotropic Weights Than Adam}

%\mh{why 'parameter' not 'weight'?}
To validate that Muon can shape the weight
matrices more evenly across directions, we conducted a spectral analysis of them.
For a weight matrix with $n$ non-zero singular values
$\sigma = (\sigma_{1}, \sigma_{2}, \dots , \sigma_{n})$ arranged in descending order, we define the normalized
singular energy distribution $q = (q_{1}, q_{2}, \dots, q_{n})$, where each
component $q_{i}$ is $q_{i}= \sigma_{i}^{2}/{\sum_{j=1}^{n}\sigma_{j}^{2}}$. This distribution represents the fraction of energy captured by each corresponding singular
vector. Based on this, we introduce several metrics to characterize the isotropy
of the spectrum: 
\begin{itemize}
    \item \textbf{Normalized SVD Entropy.} This metric, adapted from \cite{alter2000singular}, quantifies the uniformity of the singular energy distribution. A higher entropy value indicates a more isotropic matrix where energy is distributed evenly across many directions. It is defined as the Shannon entropy of the distribution $q$, normalized by the maximum possible entropy: $H_{\text{norm}} (\sigma) = -\frac{1}{\log n}\sum_{i=1}^{n}q_{i}\log q_{i}$.

    \item \textbf{Effective Rank.} The effective rank \citep{roy2007effective} provides a continuous measure of the number of significant singular dimensions used by the matrix. It is calculated as the exponentiation of the unnormalized Shannon entropy, which corresponds to the perplexity of the energy distribution: $\text{eRank}(\sigma) = \exp\left(-\sum_{i=1}^{n}q_{i}\log q_{i}\right)$.
    
    \item \textbf{Top-$k$ Energy Fraction.} This metric measures the concentration of energy within the Top-$k$ principal singular components. Assuming the singular values are sorted in descending order ($\sigma_{1}\ge \sigma_{2}\ge \dots \ge \sigma_{n}$), it is the cumulative sum of the first $k$ energy fractions: $\text{TopE}_{k}(\sigma) = \frac{\sum_{i=1}^{k}\sigma_{i}^{2}}{\sum_{j=1}^{n}\sigma_{j}^{2}}$.
    
    \item \textbf{Eigenvalue Quantile Ratio.} To measure the spread of the singular energy distribution while being robust to extreme outliers, we compute the ratio of the 75th percentile ($Q_{3}$) to the 25th percentile ($Q_{1}$) of the eigenvalues $\{ \sigma_{i}^{2}\}_{i=1}^{n}$: $Q_{75/25}(\sigma) = \frac{Q_{3}(\{\sigma_{i}^{2}\})}{Q_{1}(\{\sigma_{i}^{2}\})}$.
\end{itemize}

% normalized \ac{svd} entropy defined as $H_{\text{norm}}(\sigma)
% = -\frac{1}{\log n}\sum_{i=1}^{n}q_{i}\log q_{i}$, effective rank defined as
% $\text{eRank}(\sigma) = \exp\left(-\sum_{i=1}^{n}q_{i}\log q_{i}\right )$, Top-$k$
% energy fraction defined as $\text{TopE}_{k}(\sigma)= \sum_{i=1}^{k}\sigma_{i}^{2}
% /\sum_{j=1}^{n}\sigma_{j}^{2}$, and eigenvalue quantile ratio defined as
% $\{ \sigma_{i}^{2}\}_{i=1}^{n}$:$Q_{75/25}(\sigma) = Q_{3}(\{\sigma_{i}^{2}\})/Q_{1}
% (\{\sigma_{i}^{2}\})$. 

These metrics assess the isotropy of the distribution by capturing both the evenness of values (normalized SVD entropy, effective rank) and the rate of decay (Top-$k$ energy fraction, quantile ratio). Intuitively, more isotropic weights correspond to larger values of normalized
\ac{svd} entropy and effective rank, and smaller Top-$k$ energy fraction and eigenvalue
quantile ratio.

\begin{figure}[t]
    \centering
    \subfigure[VO(Non-gated \ac{ffn})]{ \includegraphics[width=0.46\textwidth]{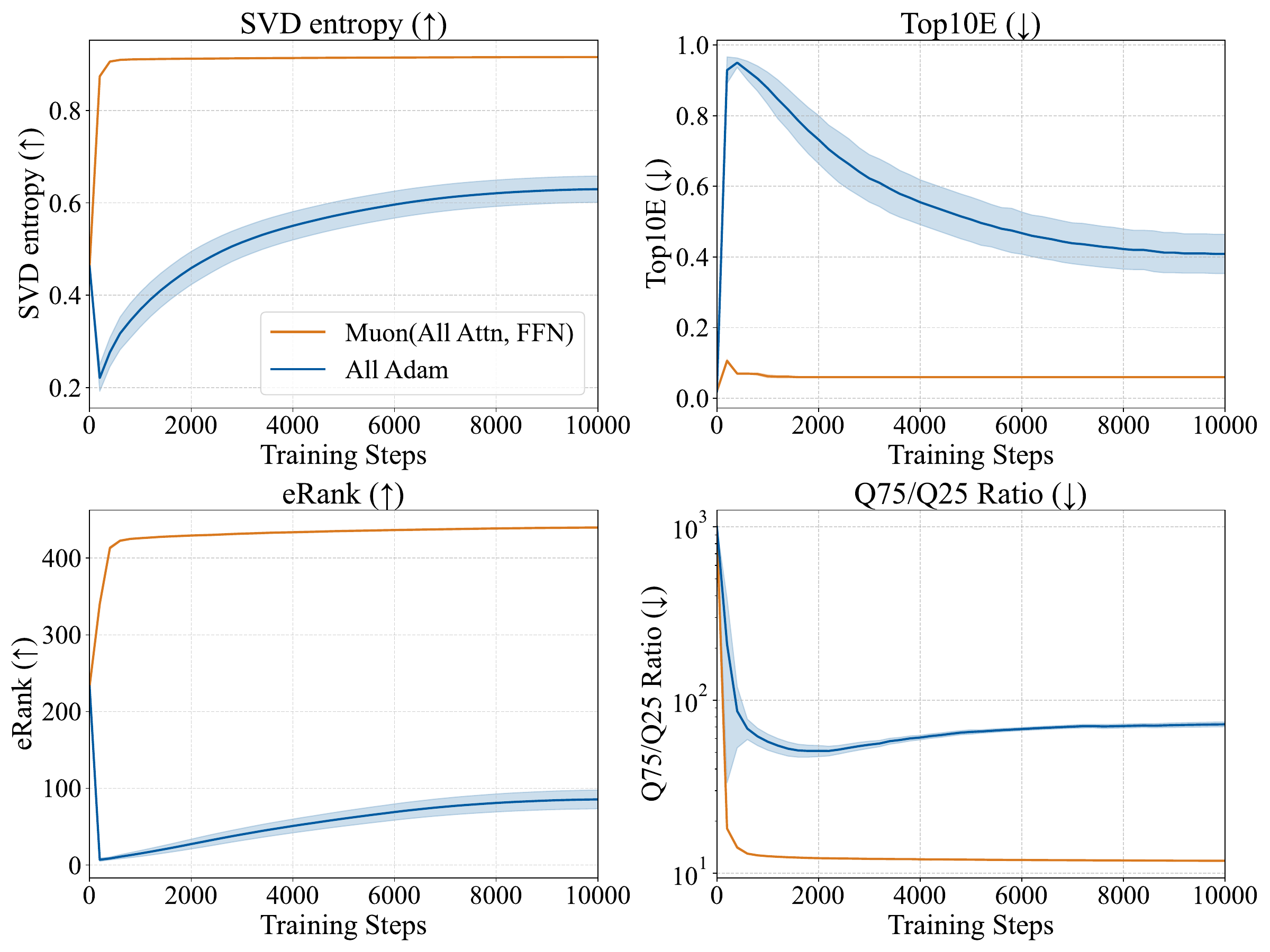}}
    \subfigure[VO(Gated \ac{ffn})]{ \includegraphics[width=0.46\textwidth]{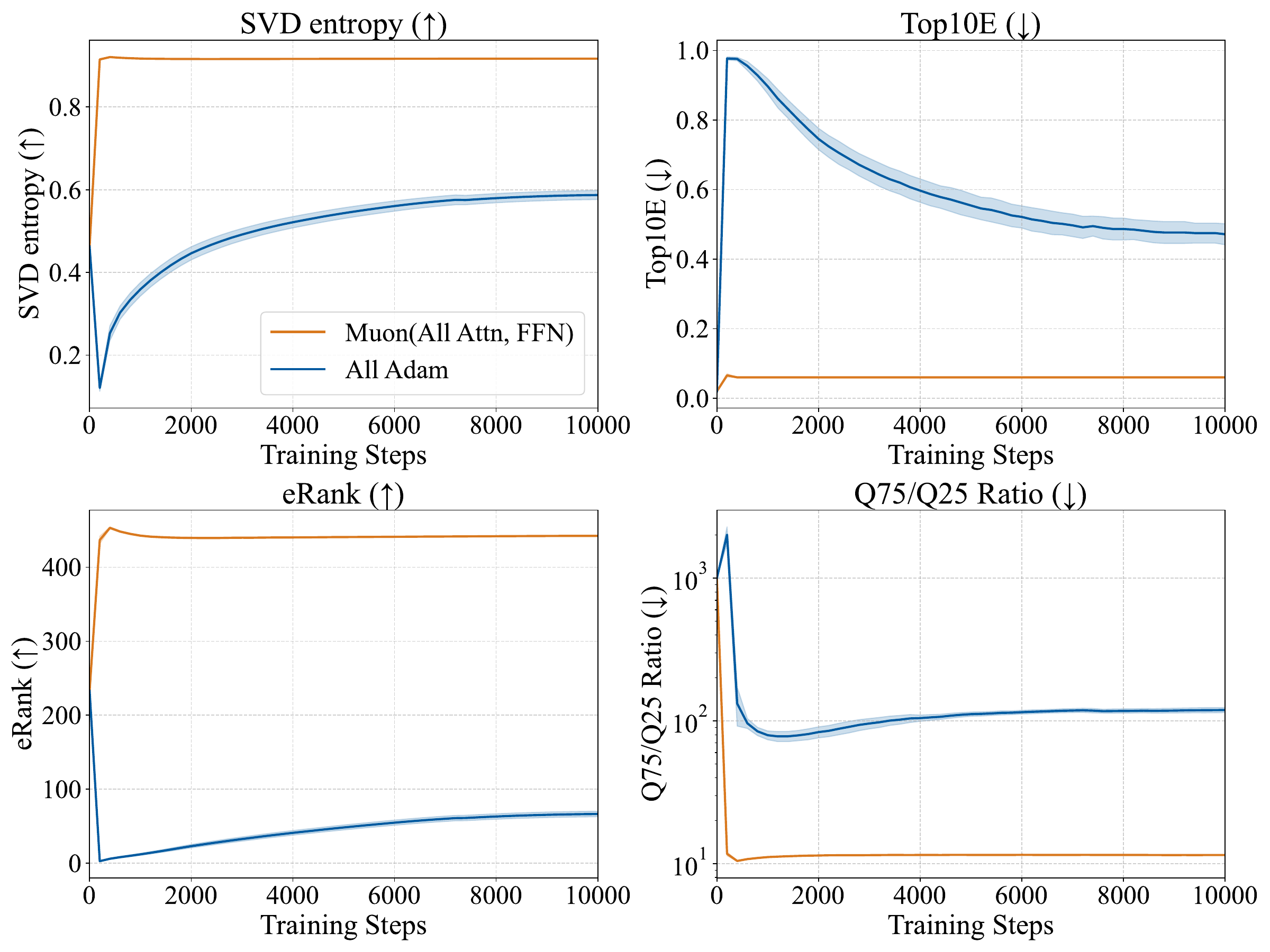}}\\[-0.3cm]
    \subfigure[$W_{\out}$(Non-gated \ac{ffn})]{ \includegraphics[width=0.46\textwidth]{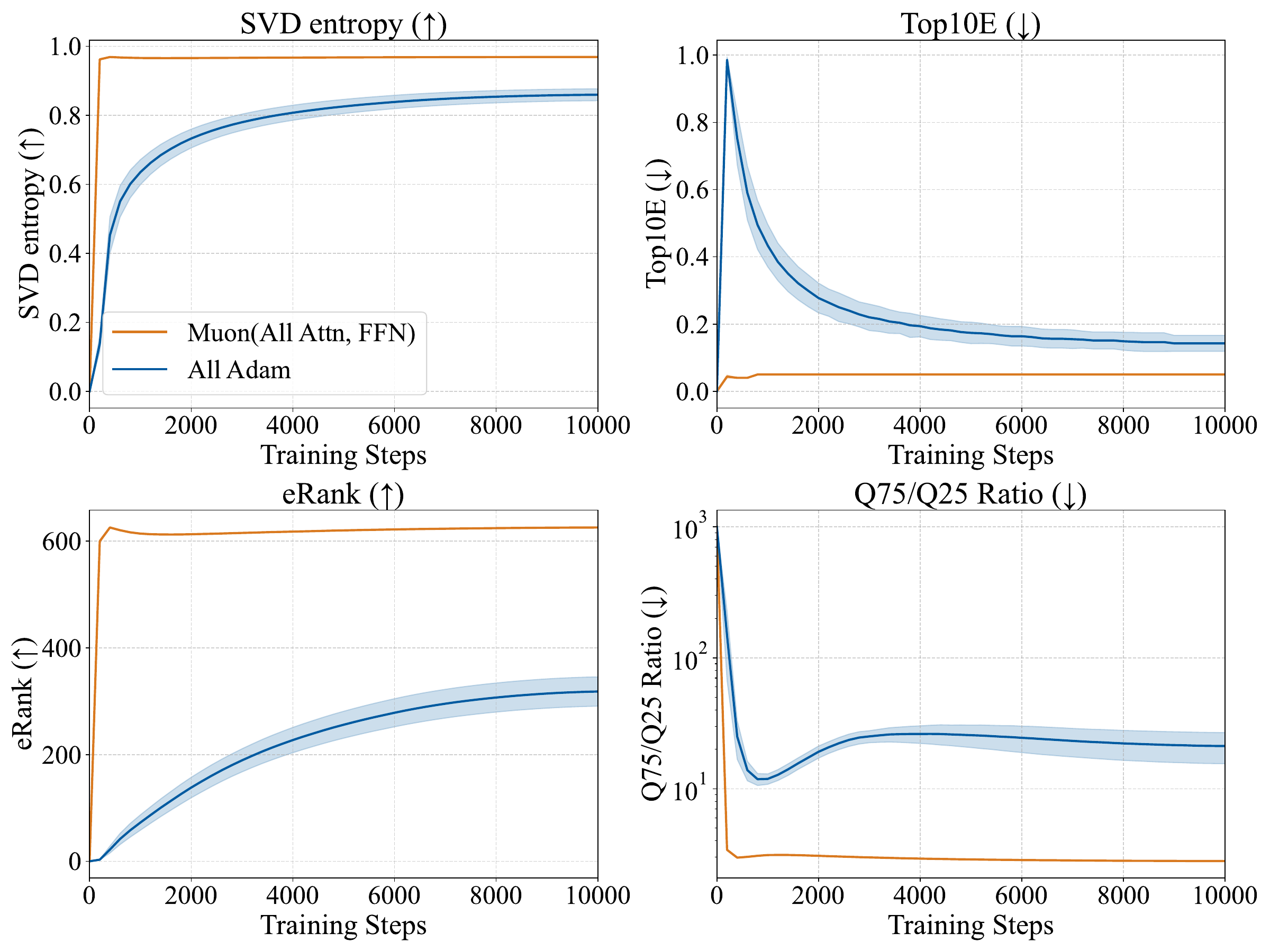}}
    \subfigure[$W_{\out}$(Gated \ac{ffn})]{ \includegraphics[width=0.46\textwidth]{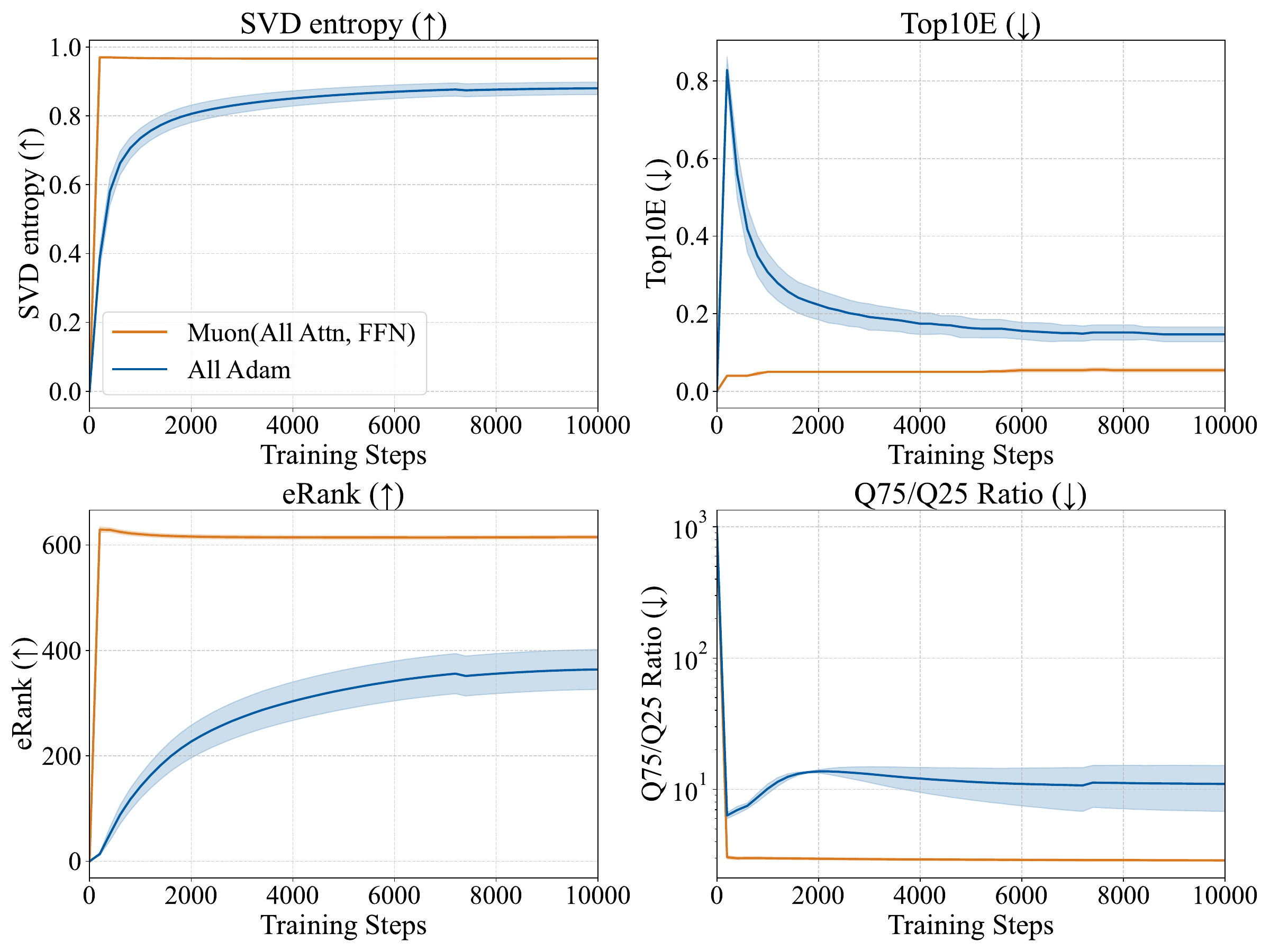}}\\
    \caption{Spectral Dynamics of Transformer Weight Matrices During Training. Each
    panel reports four metrics characterizing singular value distributions: SVD
    entropy, Top10E, eRank, and Q75/Q25 ratio. The four subplots correspond to
    different weight matrix groups: (a) VO, (b) VO (Gated \ac{ffn}), (c)
    $W_{\out}$, and (d) $W_{\out}$ (Gated \ac{ffn}).}
    \label{fig:fineweb_160m_svd}
\end{figure}

The spectral analysis in Figure~\ref{fig:fineweb_160m_svd}, focusing on the key associative
memory components from Observation~1, shows that Muon systematically reshapes
the learned weight matrices relative to Adam. The results, averaged over 10
random seeds, demonstrate that: (i) In both gated and ungated \ac{ffn} architectures,
Muon produces a much more isotropic singular spectrum than Adam from the start
of training, whereas Adam’s isotropy fluctuates significantly over the course of
optimization. (ii) The isotropy of Muon is stable across random initializations,
as indicated by the negligible error bars in Figure~\ref{fig:fineweb_160m_svd},
while Adam is highly sensitive to initialization. These findings suggest that Muon
consistently promotes richer and more diverse features in the model’s most critical
memory components, a conclusion we summarize below. The results for other weights
are in Appendix~\ref{app:svd}.
% The spectral analysis in Figure~\ref{fig:fineweb_160m_svd}, focusing on the key associative memory components from Observation 1, reveals that Muon systematically reshapes the learned weight matrices relative to Adam. In both standard and gated \ac{ffn} architectures, Muon consistently yields higher normalized SVD entropy and effective rank for both VO and $W_{\out}$ matrices, alongside a lower Top-10 energy fraction and $Q_{75/25}$ eigenvalue energy ratio. This joint pattern—higher entropy and rank, lower energy concentration—indicates a more isotropic singular spectrum, rather than one dominated by a few leading directions. These results suggest Muon encourages the development of richer, more diverse features within the model's most critical memory components, a finding we summarize below.

%\mh{[now here is 'weight matrices'..]}
%\vspace{-0.2cm}
\begin{olivebox}
    {\bf Observation 2:} Muon consistently yields more isotropic weight matrices
    with broadly distributed spectral energy than Adam, both throughout training
    and across random initializations, thereby supporting richer feature representations.
\end{olivebox}
%\vspace{-0.3cm}
Empirically, we also find that Muon learns more isotropic QK weights than Adam.
However, as discussed in Section~\ref{sec:am}, QK weights are not part of the
linear associative memory mechanism and are therefore not expected to benefit
from the isotropic property of the weight matrices. 

Our results differ fundamentally from the spectral analysis in \cite{liu2025muon} for three reasons. First, we decompose the parameters according to associative memories, whereas \cite{liu2025muon} aggregates them, obscuring the essential components driving Muon’s behavior. Second, we investigate the instability of Adam under random initialization (i.e., random seeds), which we further establish theoretically in Section~\ref{sec:theory}. Finally, our analysis focuses on dense architectures, while \cite{liu2025muon} centers on Mixture-of-Experts (MoE) models.\looseness=-1

\begin{figure}[t]
    \centering
    \subfigure[Sample/class]{ \includegraphics[width=0.32\textwidth]{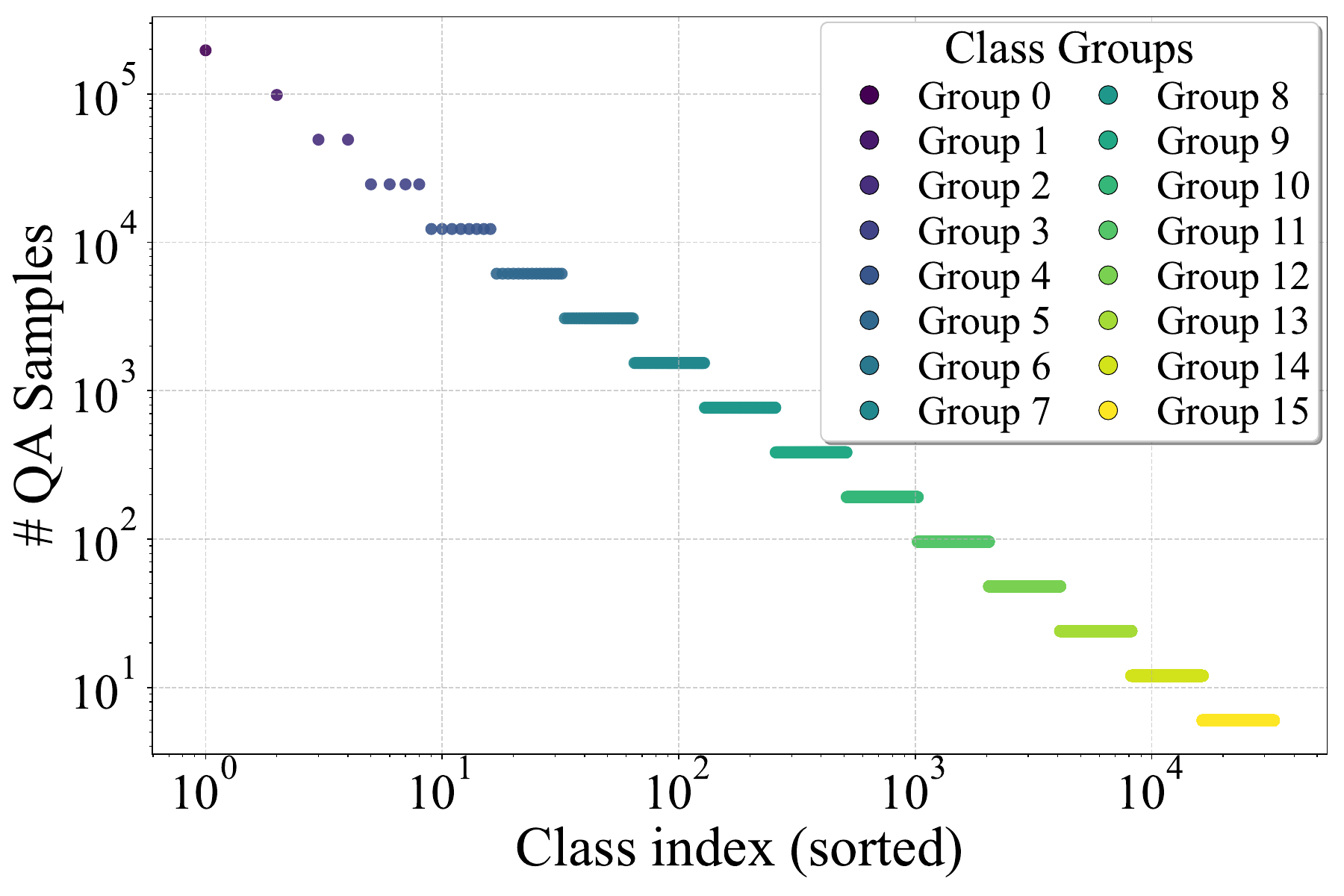}\label{fig:class}}
    \subfigure[Muon]{ \includegraphics[width=0.32\textwidth]{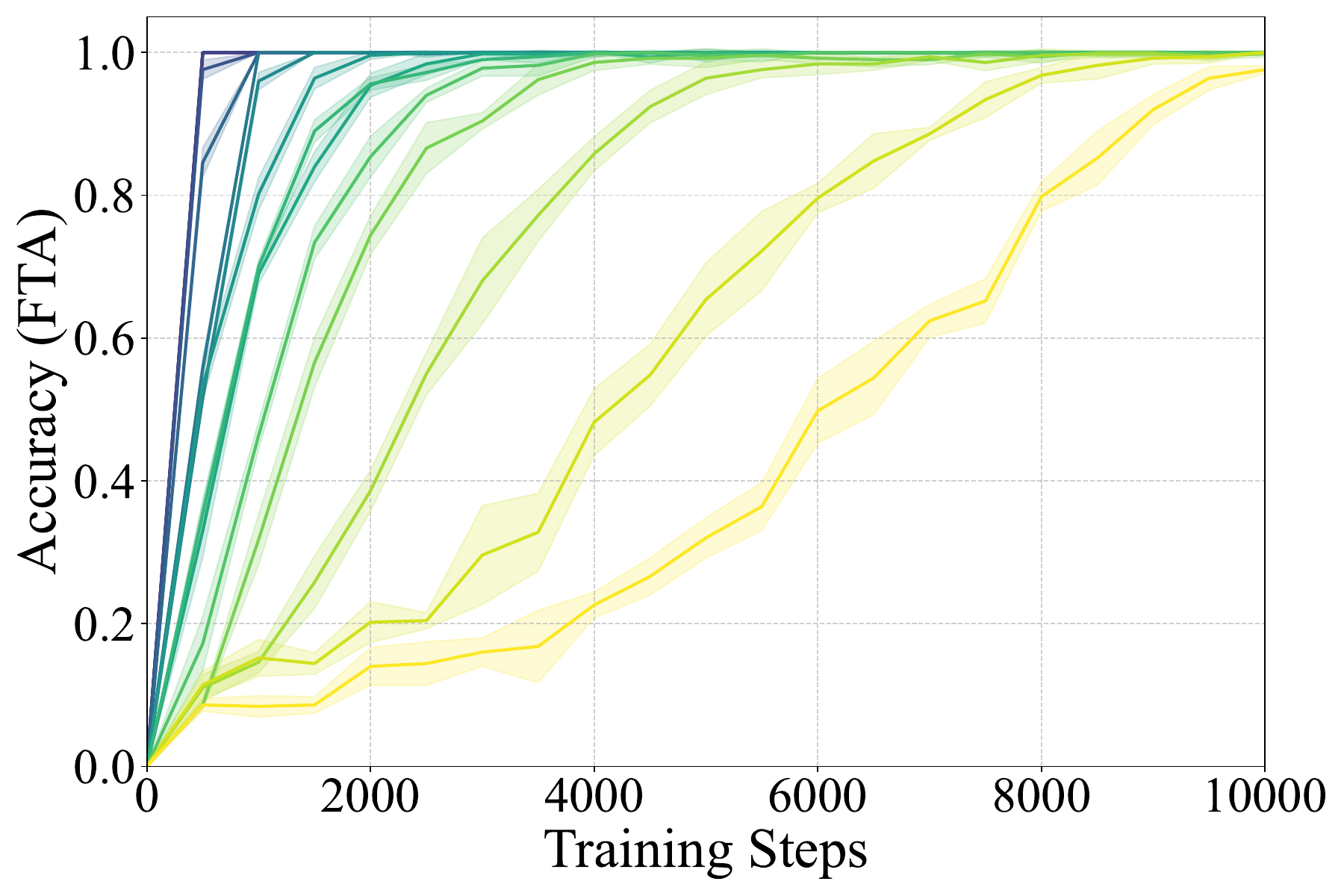}\label{fig:qa_muon}}
    \subfigure[Adam]{ \includegraphics[width=0.32\textwidth]{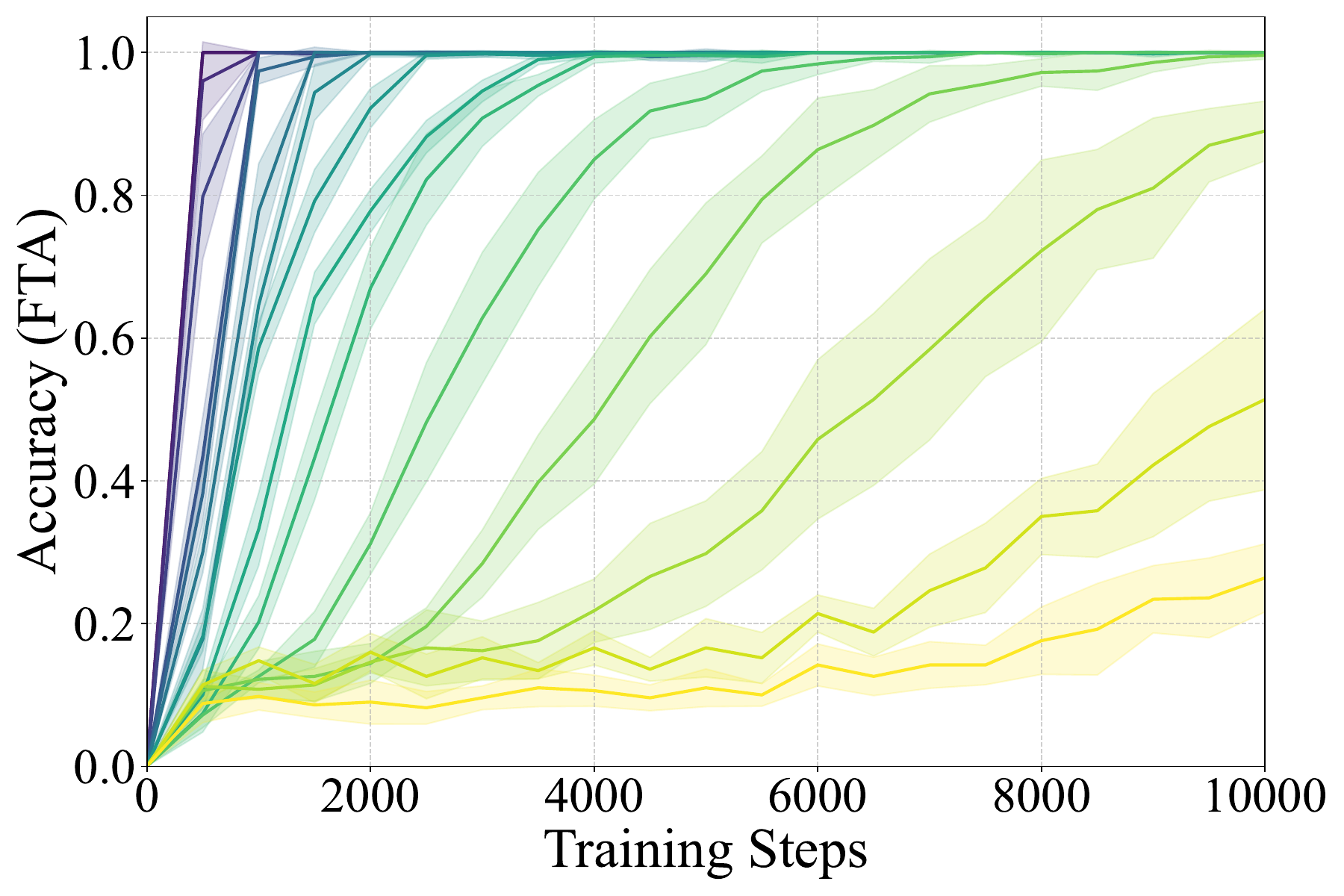}\label{fig:qa_adam}}\\[-0.3cm]
    \subfigure[SGD+Momentum]{ \includegraphics[width=0.32\textwidth]{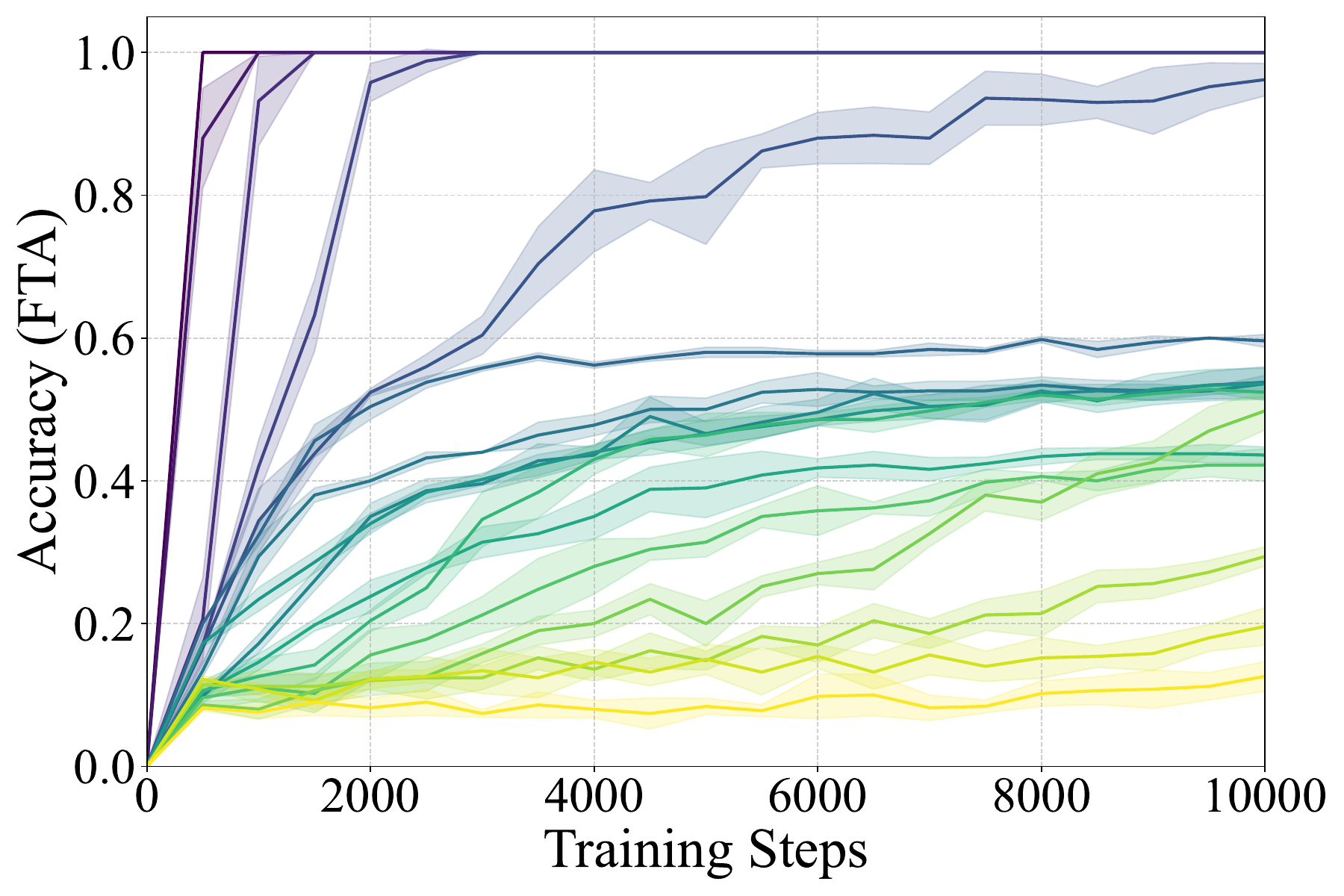}\label{fig:qa_sgd}}
    \subfigure[Muon(VO,\ac{ffn})\& Adam(QK)]{ \includegraphics[width=0.32\textwidth]{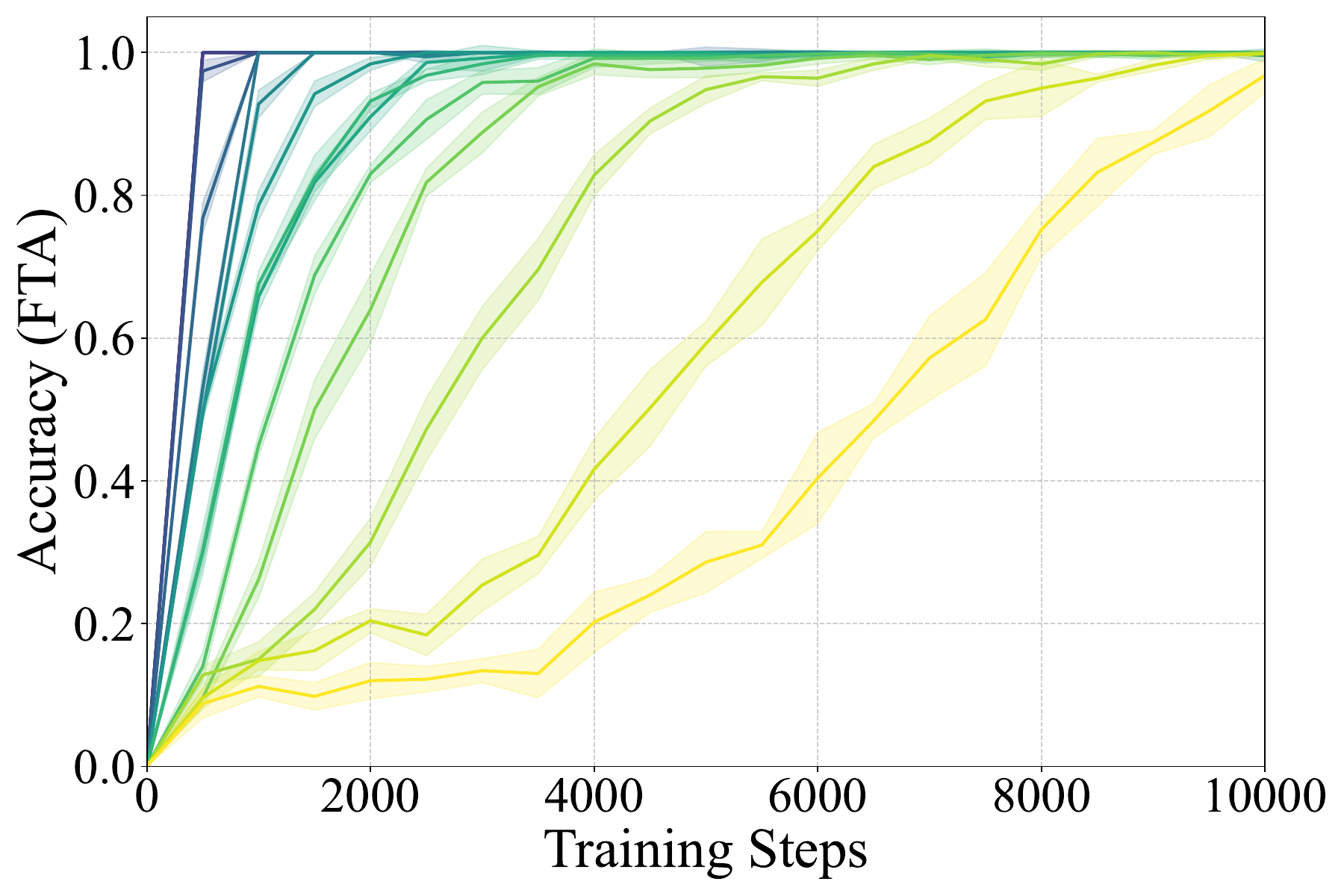}}
    \subfigure[Muon(QK)\& Adam(VO,\ac{ffn})]{ \includegraphics[width=0.32\textwidth]{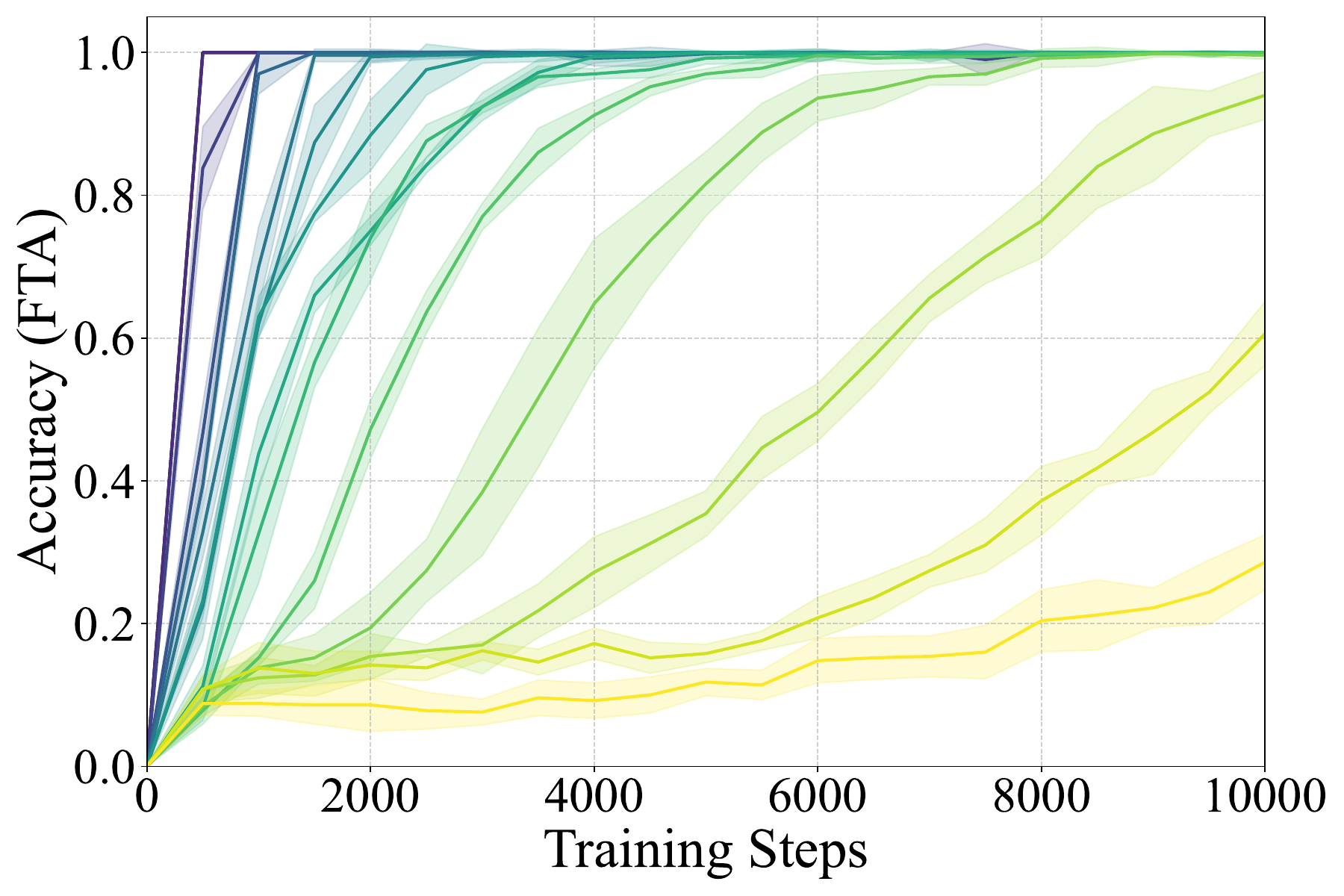}}\\
    \caption{Performance comparison of different optimizers for transformers
    with non-gated \ac{ffn} on a heavy-tailed knowledge task. (a) Sample
    distribution per class, following a power law. (b–d) Performance of Muon, Adam,
    and SGD+Momentum. (e) Muon applied to VO and \ac{ffn}, with Adam on QK. (f)
    Muon applied to QK, with Adam on VO and \ac{ffn}.}
    \label{fig:heavy_tail_qa}
\end{figure}

\subsection{Muon Acquires Knowledge More Evenly Compared To Adam}
\label{sec:knowledge} %\label{sec:heavy_tail_qa}

Our previous findings indicate that the Muon optimizer is particularly important
for the associative memory components of the model, where it learns more
isotropic weights. To examine the overall effects of learning associative memories,
we turn to a knowledge-intensive question-answering (QA) task. The task is based
on a synthetic QA dataset containing biographical information (e.g., name, birthday,
and company) for over $200{,}000$ individuals~\citep{allen2024physics}. To capture
the heavy-tailed nature of real-world knowledge, we control the frequency of
each individual’s appearance in the training set so that it follows a power-law
distribution (Figure~\ref{fig:class}), thereby inducing varying levels of difficulty
in learning knowledge about different individuals. A 160M NanoGPT model is
trained to answer questions about this biographical information. The performance
is evaluated via the First Token Accuracy (FTA) on the answers, following
\cite{allen2024physics}. Further details on the dataset are provided in Appendix~\ref{app:qadataset_details}. We include SGD as a baseline for Adam and Muon.

The results in Figure~\ref{fig:heavy_tail_qa}
lead to an unequivocal conclusion about the efficacy of different optimizers under data
imbalance. In high-frequency (head) classes, all optimizers perform well, with
Muon, Adam, and even SGD+Momentum rapidly reaching near-perfect accuracy (Figure~\ref{fig:heavy_tail_qa}(b–d)).
Consistent with prior work on heavy-tailed distributions~\citep{kunstner2024heavy},
Adam maintains a clear advantage over SGD, which struggles with tail classes. Our
key finding, however, is that Muon substantially outperforms Adam on low-frequency
(tail) data, achieving faster and more uniform convergence across all frequencies.
Moreover, the consistently tighter error bars for Muon—especially relative to Adam—reflect
lower variance and a more stable learning process.

Furthermore, the hybrid configurations in Figure~\ref{fig:heavy_tail_qa}(e–f) clarify
where Muon matters most. Applying Muon to VO+\ac{ffn} (with QK on Adam) yields strong
gains on rare classes and markedly reduces the head–tail gap, whereas applying
Muon only to QK (with VO+\ac{ffn} on Adam) yields only limited improvement. This
mirrors Observation~1: VO+\ac{ffn} is the most effective target set, as it concentrates
the model’s associative memory. Results for the gated \ac{ffn}, which show the
same pattern, are provided in Appendix~\ref{app:qa_gated}. We summarize these
findings as Observation~3.

\begin{olivebox}
    {\bf Observation 3:} In heavy-tailed, knowledge-intensive tasks, Muon matches
    Adam's strong performance in the head classes while substantially improving learning
    on tail classes, narrowing the head-tail gap and accelerating convergence.
\end{olivebox}
% \vspace{-0.3cm}
    \section{Case Study of One-Layer Models}
\label{sec:theory}
%\vspace{-0.1cm}
% We now provide an analysis of three optimizers, Adam, Muon, and \ac{gd} (as a baseline), complementing the above empirical observations. We begin with an abstraction that captures the key dynamics underlying these optimizers, and then present both empirical and theoretical results. As shown in Eqns.~\eqref{eq:attn} and~\eqref{eq:ff}, a structural property of associative memory parameters is that their outputs are added directly to the hidden states, which are subsequently processed by the language model head. Motivated by this, our abstraction retains the associative memory and language model head, replacing all preceding modules with a feature extractor.
We now analyze three optimizers—Adam, Muon, and \ac{gd} (as a baseline)—to complement the preceding empirical observations. We first introduce an abstraction that captures their key dynamics and then present both empirical and theoretical results. As shown in Eqns.~\eqref{eq:attn} and~\eqref{eq:ff}, a structural property of associative memory parameters is that their output is added directly to the hidden states, which are subsequently processed by the language model head. Motivated by this property, our abstraction retains the associative memory and language model head, while replacing all preceding modules with given feature embeddings.\looseness=-1
%with a feature extractor that supplies inputs to the associative memory.

%Concretely, we consider $K$ triplets $\{(s_{i},r_{i},o_{i})\}_{i=1}^{K}$ to be
%learned by the associative memory. The subject-relation pair $(s_{i},r_{i})$ is embedded as $E_{i}\in \mathbb{R}^{d_s}$ and the object $o_{i}$ as $\tilE_{i}\in \mathbb{R}^{d_o}$. We concatenate these into matrices $E \in \mathbb{R}^{d_s \times K}$ and $\tilE \in \mathbb{R}^{d_o \times K}$. The model predicts the object for a query $E_k$ via a linear associative memory $W \in \mathbb{R}^{d_o \times d_s}$ and a language model head $\tilE$, with output probabilities $f_{W}(E_{k}) = \sm\big(\tilE^{\top}W E_{k}\big)$. The population cross-entropy loss is $\mathcal{L}(W) = -\sum_{k=1}^{K}p_{k} \log [f_{W}(E_{k})]_{k}$, where $p_k$ is the probability of the $k$-th triplet. Our analysis extends directly to the empirical loss by using the empirical data distribution.

Consider $K$ triplets $\{(s_{i},r_{i},o_{i})\}_{i=1}^{K}$, where subject-relation pairs $(s_i, r_i)$ and objects $o_i$ are embedded into the columns of matrices $E\!\in\!\mathbb{R}^{d_s \times K}$ and $\tilE\!\in\! \mathbb{R}^{d_o \times K}$, respectively. A linear associative memory $W\!\in\!\mathbb{R}^{d_o\!\times\!d_s}$ predicts the object for a query $E_k$ with probabilities $f_{W}(E_{k}) = \sm(\tilE^{\top}W E_{k})\in\bbR^{K}$. The objective is to minimize the population cross-entropy loss $\mathcal{L}(W)\!=\!-\sum_{k=1}^{K}p_{k} \log [f_{W}(E_{k})]_{k}$, where $p_k$ is the frequency or probability of the $k$-th triplet. We consider three optimizers: \ac{gd},
Adam, and Muon. 
\begin{itemize}
    \item \ac{gd} updates the parameters according to the gradient $\nabla_{W}
\mathcal{L}(W)$ as $W_{t+1}^{\gd}=W_{t}^{\gd}-\eta_{t+1}\nabla_{W}\calL(W_{t}^{\gd}
).$
\item For Adam, we switch off the exponential moving averages (EMA), i.e., $\beta_{1}
=\beta_{2}=0$, following the practice in existing theoretical works~\citep{kunstner2024heavy,bernstein2024old}.
Under this setting, Adam reduces to sign-\ac{gd} as
$W_{t+1}^{\adam}=W_{t}^{\adam}-\eta_{t+1}\sign\big(\nabla_{W}\calL(W_{t}^{\adam})
\big),$
where $\sign(\cdot)$ denotes the element-wise sign operator.
\item For Muon, we also
disable its momentum and analyze the update
$W_{t+1}^{\muon}=W_{t}^{\muon}-\eta_{t+1}U_{t}\nor(\Sigma_{t})V_{t}^{\top},$
where $\nor(\cdot)$ normalizes all non-zero elements to $1$ (element-wise), and $U
_{t}\Sigma_{t}V_{t}^{\top}$ is the \ac{svd} of the gradient $\nabla_{W}\mathcal{L}(W_{t}^{\muon})$.
\end{itemize}
 All these optimizers adopt the zero initialization
that $W_{0}=0_{d_o, d_s}$. We then state the assumptions for our results.

\begin{figure}[t]
    \centering
    \subfigure[Average Angles Between $E_{i}$/$\tilE_{i}$]{ \includegraphics[width=0.32\textwidth]{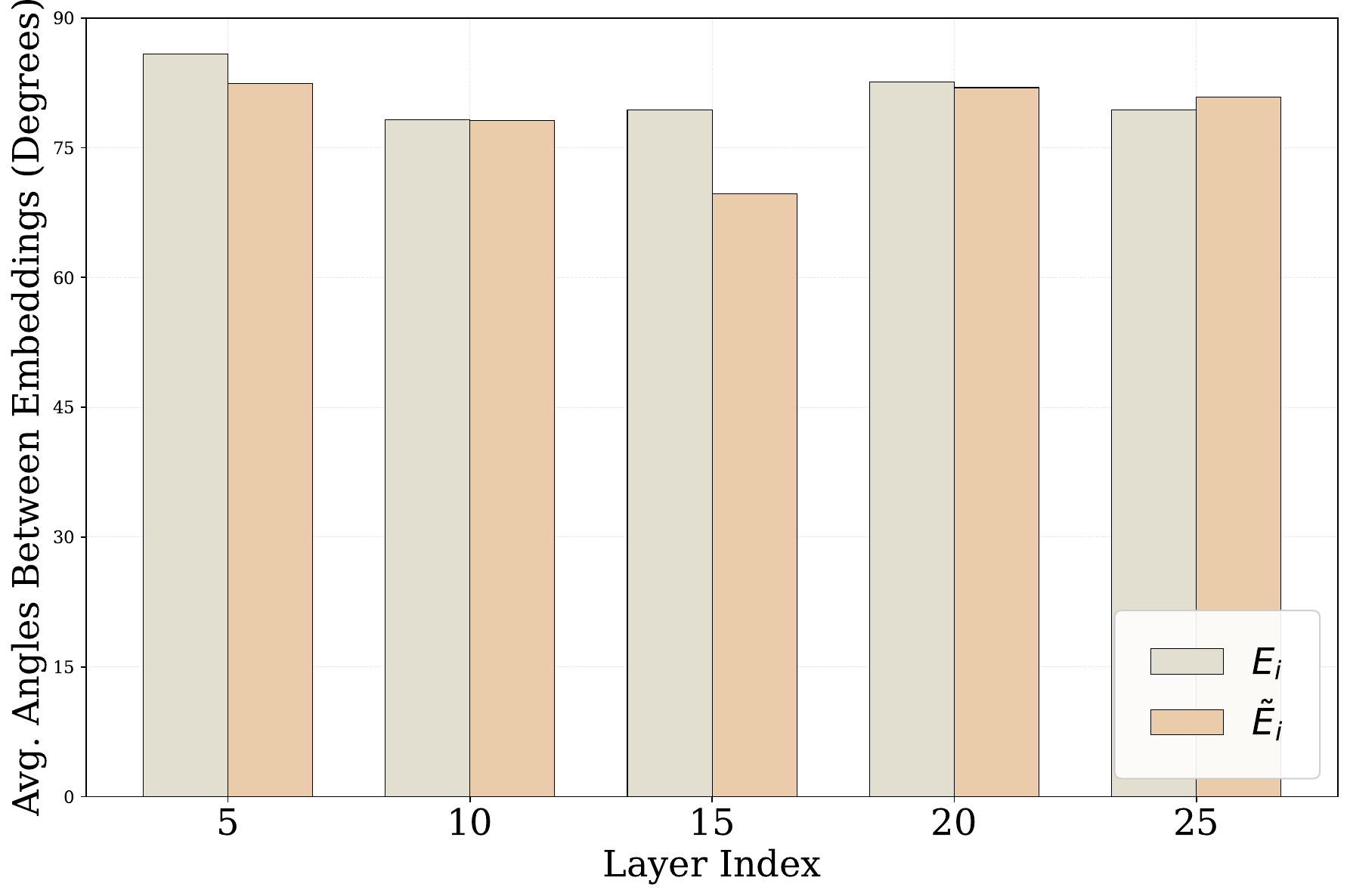}\label{fig:kv_angles_mcf}}
    \subfigure[One-step Optimization Results]{ \includegraphics[width=0.32\textwidth]{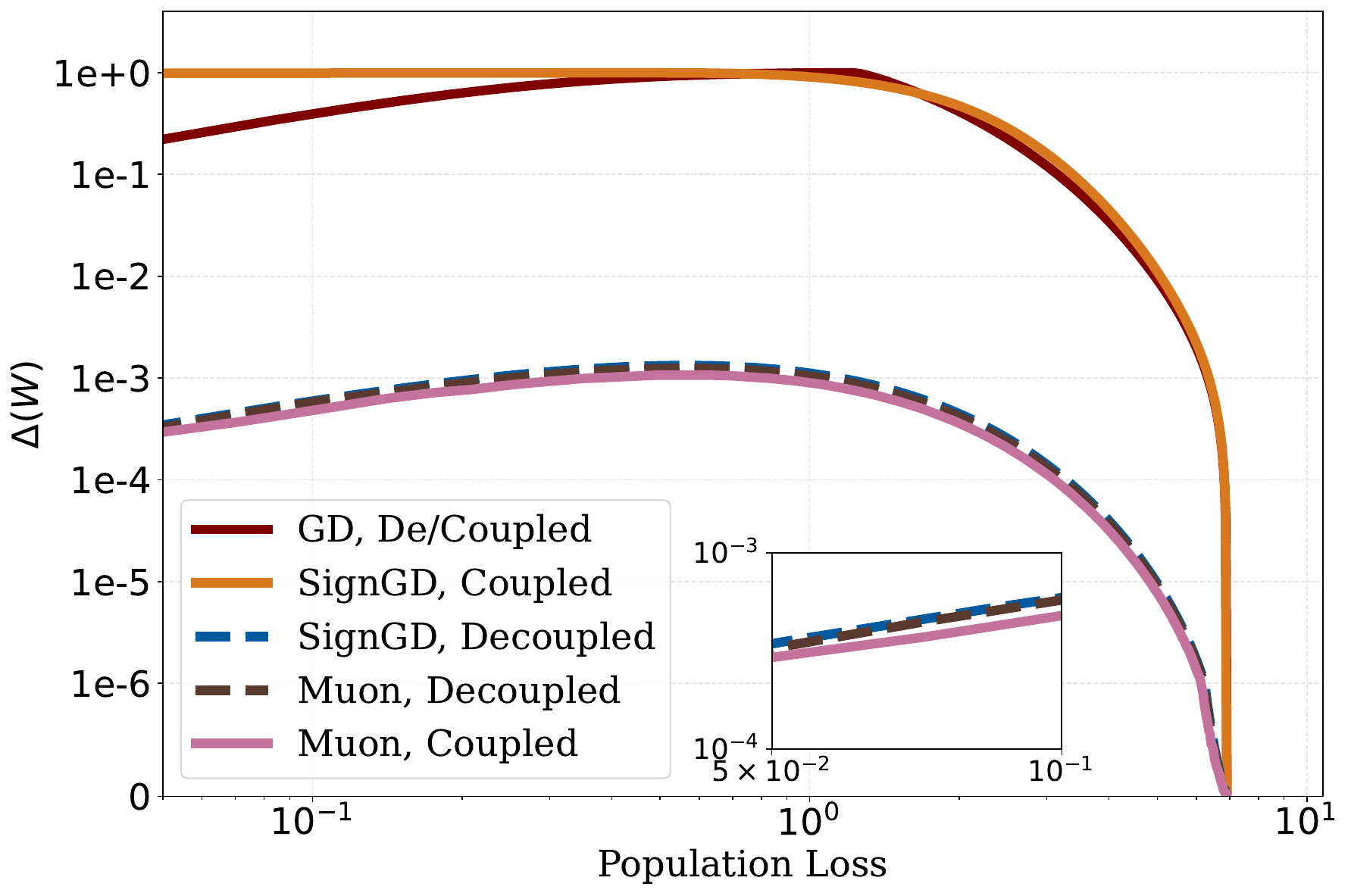}\label{fig:toy_one_step}}
    \subfigure[Multi-step Optimization Results]{ \includegraphics[width=0.32\textwidth]{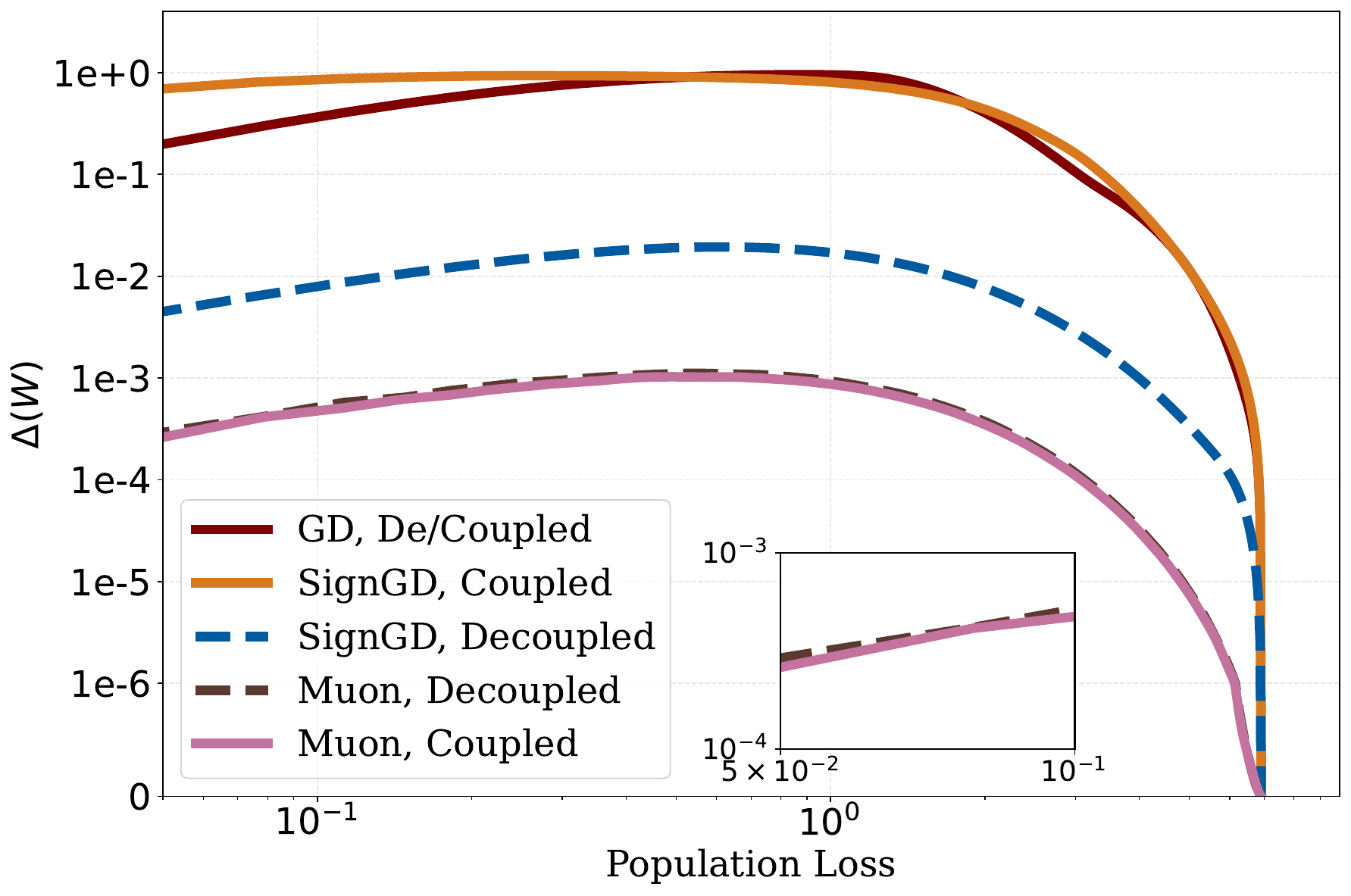}\label{fig:toy_multi_step}}
    \caption{(a) Average angles between $E_{i}$ or $\tilE_{i}$ in \ac{ffn} at layers
    $5$, $10$, $15$, $20$, $25$ of Llama3-8b-instruct. (b) Results of one-step
    \ac{gd}, Sign\ac{gd}, and Muon with both coupled and decoupled embeddings.
    For \ac{gd}, the outcomes under the two embedding types coincide. (c) Results
    of multi-step \ac{gd}, Sign\ac{gd}, and Muon with both coupled and decoupled
    embeddings.} %For \ac{gd}, the outcomes under the two embedding types coincide.
\end{figure}
\vspace{-0cm}
\begin{assumption}
    \label{assump:ortho} The embeddings $E$ and $\tilE$ are orthonormal, i.e., $E^{\top}E = \tilE^{\top}\tilE =I_{K,K}$.  %, i.e.,
    %$E^{\top}E = \tilE^{\top}\tilE = I_{K,K}\in \mathbb{R}^{K \times K}$, where $I
    %_{K,K}$ is the identity matrix.
\end{assumption}

%This assumption states that the feature embeddings are orthonormal. 
The unit-norm
requirement rules out feature-level imbalance, which would otherwise couple with
the imbalance induced by $p_{k}$ and complicate the analysis. Our techniques can
be directly applied even without this unit-norm requirement. The orthogonality
assumption is intuitively plausible, as different concepts are independent and do not influence one another. We empirically verify this on Llama3-8b-instruct~\citep{dubey2024llama}.
Following \cite{fang2024alphaedit}, we extract $E_{i}$ and $\tilE_{i}$ in \ac{ffn}
across layers for $3,000$ knowledge items of Counterfact~\citep{meng2022locating}
and compute average angles between them (see Appendix~\ref{app:exp_setting_angles} for details). As shown in Figure~\ref{fig:kv_angles_mcf}, these angles are near $9
0^{\circ}$, confirming approximate orthogonality. For $K$ independent concepts,
orthogonality requires $d_{r}, d_{s}\geq K$. For simplicity, we set
$d_{r}= d_{s}= K$ in what follows.

\begin{assumption}
    \label{assump:two_class} The first $L$ triplets share the same probability and
    together contribute a total mass of $\alpha$, i.e., $p_{k}= \alpha / L$ for
    $k \in [L]$. The remaining triplets also share the same probability and together
    contribute a total mass of $1 - \alpha$, i.e.,
    $p_{k}= (1 - \alpha) / (K - L)$ for $k > L$.
\end{assumption}
%\vspace{-0.2cm}
This assumption states that the data imbalance is between two classes among the
$K$ triplets. Defining $\beta = L/K$, the ratio $\alpha / \beta$ quantifies the
degree of balance: if $\alpha > \beta$, the first $L$ triplets appear more
frequently during learning, and vice versa. This simplified two-class setting is
sufficient to capture the primary differences between optimizers; the multi-class case follows directly from our proof by extending the \ac{svd} calculation.
%\vspace{-0.2cm}

\subsection{Experimental Results}
Under Assumptions~\ref{assump:ortho} and \ref{assump:two_class}, we   evaluate
\ac{gd}, Sign\ac{gd}, and Muon for $\alpha=0.8$, $\beta=0.2$,  considering two
embeddings for $E$ and $\tilde E$: (i) support-decoupled: the supports (indices
of non-zero entries) of different $E_{i}$ or $\tilE_{i}$ are disjoint; (ii) support-coupled:
supports may overlap. We study two optimization protocols,   initializing
$W_{0}= 0_{d_o\times d_s}$: (i) one-step: take a single update with a scaled
step size to obtain a range of $\mathcal{L}(W)$ values; (ii) multi-step: run
multiple updates to reduce $\mathcal{L}(W)$, varying the number of steps. Experimental details are in Appendix~\ref{app:toy_details}. To quantify \emph{learning
imbalance} across $K$ knowledge items, we examine the relationship between  population
loss $\mathcal{L}(W)$ and \emph{maximal probability gap}
$\Delta(W) := \max_{i,j\in[K]}[f_{W}(E_{i})]_{i}- [f_{W}(E_{j})]_{j},$ where
$[f_{W}(E_{i})]_{i}$ denotes the probability assigned to the correct item
$i$. A larger $\Delta(W)$ indicates greater imbalance.

Across both optimization-step protocols and embeddings (Figures~\ref{fig:toy_one_step},
\ref{fig:toy_multi_step}), we observe that
\begin{itemize}
    \item For all optimizers, $\Delta(W)$
first \emph{increases} and then \emph{decreases} as $\mathcal{L}(W)$ decreases. Early
in training, when correct probabilities are near $0$, imbalance is pronounced; later, when all items are well learned (e.g., probabilities $\geq 0.9$), imbalance diminishes.
    \item For both embedding regimes, \ac{gd} and Muon behave
consistently: \ac{gd} exhibits a substantial imbalance, whereas Muon remains much
more balanced across items.
    \item Sign\ac{gd} also demonstrates unstable
behavior; its imbalance resembles \ac{gd} in the coupled embedding case and Muon
in the decoupled embedding case.
\end{itemize}   

Because one-step and multi-step experiments align qualitatively, we first analyze
the \textbf{one-step} setting for clarity. This simplification is common in theoretical
studies of neural network dynamics~\citep{ba2022high,dandi2023two}, and our techniques
extend directly---albeit with more algebra---to the multi-step case. As a
demonstration, Theorem~\ref{thm:multistep} provides a multi-step analysis of Muon.

\subsection{Theoretical Results}

For the one-step analysis, define the smallest  correct-class probability
across all knowledge items, under the condition that at least one item achieves the correct-class probability of at least $1-\epsilon$ as
\begin{align}
    \maxd_{\opt}^{\epsilon}= \inf_{\eta \ge 0}\Big\{ \min_{k\in[K]}[f_{W_\eta}(E_{k})]_{k}\,\Big|\, \max_{k\in[K]}[f_{W_\eta}(E_{k})]_{k}\geq 1-\epsilon,\  W_{\eta}= W_{0}- \eta \cdot G_{\opt}(W_{0}) \Big\}, \label{eq:inf_correct_prob}
\end{align}
%\textcolor{red}{this is bad notation. Since $\Delta(W)$ was used for the difference between log probabilities previously, it conflicts with the use here in which $\maxd_{\opt}^\epsilon$ represents a probability. The definition would benefit from a better discussion. what is the infimum over? is it possible to provide an example starting from the softmax?}
where $\opt \in \{\gd, \adam, \muon\}$ and $G_{\opt}(W_{0})$ denotes the
parameter update of the optimizer ``$\opt$'' at $W_{0}$; and $W_{\eta}$ denotes the parameter obtained after one step of optimizer ``$\opt$'' with step size $\eta$ starting from $W_{0}$, i.e., $W_{\eta}= W_{0}- \eta \cdot G_{\opt}(W_{0})$. %\mh{[the above description does not mention where $W_{\eta}$ is coming from, i.e., by running one step of a given algorithm][this is important since you are evaluating the performance of certain very 'special' matrices, need to say where this 'special' matrics is coming from.] }
Specifically, we denote 
\begin{align*}
    G_{\gd}(W_{0}) = \nabla_{W}\calL(W_{0}), \qquad  G_{\adam}(W_{0}) = \sign(\nabla_{W}\calL (W_{0})), \qquad G_{\muon}(W_{0})= U_{0}
\nor(\Sigma_{0})V_{0}^{\top}, 
\end{align*}
 where $U_{0}\Sigma_{0}V_{0}^{\top}$ is the \ac{svd}
of $\nabla_{W}\mathcal{L}(W_{0})$. %Intuitively, $\maxd_{\opt}^{\epsilon}\in[0,1-\varepsilon]$
%measures the \emph{minimum correct-class probability across all knowledge items,
%under the condition that at least one item is learned with probability at least $1
%-\epsilon$}. Roughly, 
Note that $\maxd_{\opt}^{\epsilon}\in[0,1-\varepsilon]$ and   $\Delta(W)$  are related as $\Delta(W)=1-\epsilon-\maxd_{\opt}
^{\epsilon}\geq 0$. %\textcolor{red}{can anyone please explain this equality $\Delta(W)=1-\epsilon-\maxd_{\opt}
%^{\epsilon}$ to me? if it's not precise, remove. } 
When $\maxd_{\opt}^{\epsilon}\approx 1-\epsilon$, $\opt$ achieves balanced learning across facts; in contrast, when $\maxd_{\opt}^{\epsilon}\approx 0$,  imbalanced learning ensues. %\textcolor{red}{do you mean ``under the condition that at least one item is learned with probability at least $1-\epsilon $?} 

\begin{theorem}
    \label{thm:comp} If Assumptions~\ref{assump:ortho} and \ref{assump:two_class}
    hold, with fixed $\alpha,\beta$ such that $\alpha\neq\beta$, and $K$ goes to
    infinity, %{\color{red}[jl: since below we use big O notation for $K$, do we need to say "$K$ goes to infinity"?]},
    we obtain the following results for one-step \ac{gd}, Muon, and Adam.
    \begin{itemize}
        \item  For \ac{gd},  for any $\tilE$ and $E$ satistifying
    Assumption~\ref{assump:ortho}, we have 
    \begin{align*}
        \maxd_{\gd}^{\epsilon}= O(\epsilon^{-r(\alpha,\beta)}K^{r(\alpha,\beta)-1}), \text{ where }r(\alpha,\beta) = \min \bigg\{\frac{\alpha(1-\beta)}{\beta(1-\alpha)},\frac{\beta(1-\alpha)}{\alpha(1-\beta)}\bigg\}< 1.
    \end{align*}

    % \vspace{-1em}
    \item For Muon,  for any $\tilE$ and $E$ satistifying
    Assumption~\ref{assump:ortho}, we have
    \begin{align*}
        \maxd_{\muon}^{\epsilon}\geq 1- \epsilon\bigg(1+O\bigg(\frac{\log K}{K}\bigg)\bigg),\text{ and }G_{\muon}(W_{0})= -\tilE E^{\top}+O\bigg(\frac{1}{K}\tilE J_{K,K}E^{\top}\bigg),
    \end{align*}
    %\textcolor{red}{remove the $ G_{\muon}(W_0)$ equation }
    where $J_{K,K}\in\bbR^{K\times K}$ is the matrix with all elements equal to
    $1$. The big-$O$ notation for matrices means that for $A=O(B)$, each entry
    satisfies $A_{ij}= O(B_{ij})$ for all $i,j$.

    \item For Adam, there exist $\tilE$ and $E$ satisfying Assumption~\ref{assump:ortho}
    such that $\maxd_{\adam}^{\epsilon}\geq 1-\epsilon.$ There also exist $\tilE^{\prime}$ and
    $E^{\prime}$ satisfying Assumption~\ref{assump:ortho} such that
    \begin{align*}
        \maxd_{\adam}^{\epsilon}=O( \epsilon^{-0.7}K^{-0.3}),\text{ and }\frac{\sigma_{\min}\big(G_{\adam}(W_{0})\big)}{\sigma_{\max}\big(G_{\adam}(W_{0})\big)}\leq 25\%,
    \end{align*}
    where $\sigma_{\max}$ and $\sigma_{\min}$ are the largest and smallest singular
    values, respectively.
    \end{itemize}
\end{theorem}

\textbf{Interpretation of Theorem~\ref{thm:comp}.} The proof of Theorem~\ref{thm:comp} is provided in
Appendix~\ref{app:comp}. We now explain the results for the three optimizers separately.
For \ac{gd}, the quantity $r(\alpha,\beta) \leq 1$ measures the imbalance of the
data distribution: $r(\alpha,\beta) = 1$ corresponds to perfectly balanced data,
while $r(\alpha,\beta) \ll 1$ indicates severe imbalance. The results show that if
one set of $(s,r,o)$ triplets is learned with the correct-class probability $[f_{W}
(E_{k})]_{k}$ of at least $1-\epsilon$, then there exists another triplet whose
correct-class probability is
$O(\epsilon^{-r(\alpha,\beta)}K^{r(\alpha,\beta)-1})$. Thus, \ac{gd} is highly sensitive
to data imbalance: as the training distribution becomes more imbalanced, the dispersion
of correct-class probabilities across items increases, i.e., the maximal probability gap
$\Delta(W)$ grows and $\min_{k\in[K]} [f_{W}(E_{k})]_{k}$ decreases. This mirrors the message in Figure~\ref{fig:toy_one_step}, \ref{fig:toy_multi_step},
and Figure~\ref{fig:qa_sgd} in Section~\ref{sec:knowledge}.

In contrast, Muon learns in a balanced fashion, unaffected by data imbalance for
any embeddings $\tilE$ and $E$. Our results show that when the best-learned triplet
achieves a correct-class probability of at least $1-\epsilon$, the worst-learned
triplet has a comparable correct-class probability at least
$1-\epsilon(1 + O(\log K / K))$. This justifies Observation 3. Furthermore,
consistent with Observation 2, Muon’s update $G_{\muon}$ rule allocates equal strength
to all update directions; equivalently, the singular values of $G_{\muon}(W_{0})$
are nearly identical.

Our analysis shows that Adam’s performance is \emph{unstable} with respect to
the embeddings $\tilE$ and $E$, as reflected by the large error bars in Observations~2
and~3. Adam’s element-wise normalization disrupts the inherent matrix structure of
the gradient. When embeddings of different triplets have disjoint supports (e.g.,
$\tilE = E = I_{K,K}$), Adam can optimize parameters in a balanced manner.
However, when embeddings overlap, the sign operator in Adam can introduce
imbalance. In particular, the worst-optimized triplet may then have correct-class
probability $O(\epsilon^{-0.7}K^{-0.3})$. These exponents ($0.3,0.7$) are intrinsic
to Adam’s update under certain embeddings and are independent of $\alpha$ or $\beta$.
Moreover, the Adam update $G_{\adam}(W_{0})$ exhibits pronounced spectral decay—for
example, its smallest singular value can be less than $25\%$ of the
largest—unlike the nearly uniform singular values of Muon. This spectral decay explains
the poor isotropy reported in Observation~2.

In the following, we extend our techniques of one-step analysis to the multi-step
analysis of Muon. Parallel to \eqref{eq:inf_correct_prob}, we define the infimum
correct-class probability for the multi-step optimizer as 
\begin{align*}
\maxd_{\opt}^{\epsilon}
= \inf_{t}\Big\{ \min_{k\in[K]}[f_{W_t}(E_{k})]_{k}\,\Big|\, \max_{k\in[K]}[f_{W_t}(E_{k}
)]_{k}\geq 1-\epsilon, \text{ where }W_{t}=W_{t-1}-\eta_{t}\cdot G_{\opt}(W_{t-1}
) \Big\}.
\end{align*}
Here, we assume that the learning rates $\{\eta_{t}\}_{t\geq 1}$ are
determined by a fixed schedule prior to optimization. Although the quantity implicitly
depends on this schedule, we omit it from the notation for $\maxd_{\opt}^{\epsilon}$ for brevity. We emphasize
that different schedules may affect the value of $t$ that attains the infimum in
$\maxd_{\opt}^{\epsilon}$, but they do not influence the balance behavior that
we present.

\begin{theorem}
    \label{thm:multistep} If Assumptions~\ref{assump:ortho} and \ref{assump:two_class}
    hold, then multi-step Muon achieves 
    \begin{align*}
        \maxd_{\muon}^{\epsilon}\geq 1- \epsilon\bigg(1+O\bigg(\frac{\log K}{K}\bigg)\bigg),\text{ and }G_{\muon}(W_{t})= -\tilE E^{\top}+O\bigg(\frac{1}{K}\tilE J_{K,K}E^{\top}\bigg)\text{ for any }t\geq 0.
    \end{align*}
\end{theorem}
%\textcolor{red}{is the first inequality involving $\maxd_{\muon}^{\epsilon} $ also affected by $t$? if not, why?}
The proof is provided in Appendix~\ref{app:multistep}. We note that the multi-step analysis of Muon exhibits similar properties to the one-step case presented in Theorem~\ref{thm:comp}. Specifically, for any embedding, Muon achieves balanced learning across all items, and its update at each step remains nearly isotropic.
    \section{Conclusion}
Our work takes the first step toward unveiling why and how Muon outperforms Adam. Through ablations of Muon’s effect on different Transformer components and by relating these results to the balanced learning of associative memories, we conclude that the Muon update rule is aligned with the outer-product structure of linear associative memories, enabling more balanced and effective learning of tail classes in heavy-tailed distributions. Intuitively, this property of Muon may extend beyond outer products to higher-order tensor products, an exciting direction for future work.

    \clearpage
    \bibliography{ref}
    \bibliographystyle{iclr2026_conference}
    \newpage
    \appendix

    \section{Steepest Descent View Understanding Muon and Adam}\label{app:dgd_understand}

\citet{bernstein2024old} showed that many popular deep learning optimizers can be understood through the unifying framework of \emph{steepest descent}, once their exponential moving averages (EMAs) are disabled. This perspective shifts the focus from heuristic or second-order motivations to a more fundamental, geometric view: the choice of an optimizer is equivalent to choosing a specific \emph{norm} to measure the ``size'' of the weight update.

\textbf{The Steepest Descent Framework.} The core idea is to find a weight update, $\Delta w$, that minimizes a local quadratic approximation of the loss function. This is formulated as the following optimization problem:
\begin{equation*}
    \Delta w^* = \argmin_{\Delta w} \left[ g^\top \Delta w + \frac{\lambda}{2} \|\Delta w\|^2 \right],
    \label{eq:steepest_descent_problem}
\end{equation*}
where $g$ is the gradient of the loss, $\lambda > 0$ is a ``sharpness'' parameter that controls the step size, and $\|\cdot\|$ is a chosen norm.

The solution to this problem can be expressed as:
\begin{equation*}
    \Delta w^* = -\eta \cdot d,
    \label{eq:steepest_descent_solution}
\end{equation*}
where the step size $\eta = \frac{\|g\|_*}{\lambda}$ and the update direction $d = \arg\max_{\|t\|=1} g^\top t$. Here, $\|\cdot\|_*$ denotes the \emph{dual norm} of $\|\cdot\|$ (defined as $\|y\|_* = \sup_{\|x\|\le 1} y^\top x$). The key insight is that different choices of the norm $\|\cdot\|$ lead to different update directions $d$, recovering the update rules of well-known optimizers.

\textbf{Muon as Steepest Descent under Spectral Norm.} The update rule of the Muon optimizer is derived by applying the steepest descent framework to weight matrices equipped with the \emph{spectral norm}, denoted in the paper as the $\|\cdot\|_{\ell_2 \to \ell_2}$ operator norm (defined as its largest singular value, $\|A\|_{\ell_2 \to \ell_2} = \sigma_{\max}(A) = \sup_{\|x\|_2=1} \|Ax\|_2$). For a gradient matrix $G$, the problem is to find the update $\Delta W$ that solves:
\begin{equation*}
    \Delta W^* = \argmin_{\Delta W} \left[ \langle G, \Delta W\rangle_F + \frac{\lambda}{2} \|\Delta W\|_{\ell_2 \to \ell_2}^2 \right].
\end{equation*}

The solution to this problem is directly determined by the Singular Value Decomposition (SVD) of the gradient, $G = U\Sigma V^\top$. The resulting update direction, which maximizes alignment with the gradient under the spectral norm constraint, is shown to be $UV^\top$. The corresponding dual norm of the gradient, $\|G\|^*_{\ell_2 \to \ell_2}$, which scales the step size, is found to be $\text{tr}(\Sigma)$, the sum of the singular values. Combining these components yields the final steepest descent update rule:
\begin{equation*}
    \Delta W^* = -\frac{\text{tr}(\Sigma)}{\lambda} \cdot UV^\top.
\end{equation*}
This demonstrates that Muon's core operation is a principled descent step where the singular vectors of the gradient determine the direction, and the sum of its singular values scales the step size.

\textbf{Adam as Steepest Descent under $\ell_\infty$ Norm.} Adam can be understood as steepest descent on the flattened parameter vector $w$ when the space is equipped with the vector \emph{infinity norm} ($\ell_\infty$) (defined as the maximum absolute value of its elements, $\|x\|_\infty = \max_i |x_i|$). For a gradient vector $g$, the optimization problem is to find the update $\Delta w$ that solves:
\begin{equation*}
    \Delta w^* = \argmin_{\Delta w} \left[ g^\top \Delta w + \frac{\lambda}{2} \|\Delta w\|_\infty^2 \right].
\end{equation*}

The update direction that maximizes alignment with the gradient $g$ under the infinity norm constraint is the sign of the gradient, $\text{sign}(g)$. The corresponding dual norm of the gradient, $\|g\|^*_\infty$, which scales the step size, is the $\ell_1$ norm, $\|g\|_1$ (the sum of the absolute values of its elements, $\|x\|_1 = \sum_i |x_i|$). Combining these components yields the final steepest descent update rule:
\begin{equation*}
    \Delta w^* = -\frac{\|g\|_1}{\lambda} \cdot \text{sign}(g).
\end{equation*}
This reveals that Adam's fundamental operation corresponds to a descent step where each parameter moves with the same magnitude, determined only by its gradient's sign. 

    \section{Experimental Details}\label{app:exp_details}
\subsection{Experimental Details of Training on FineWeb}
When training 160M models on FineWeb, we disable weight decaying and Nesterov acceleration for both Adam and Muon. Thus, we only compare their performance along. To set the learning rate, we conduct a grid search on $1\times 10^{-1}, 5\times 10^{-2},2\times 10^{-2},1\times 10^{-2},5\times 10^{-3},2\times 10^{-3},1\times 10^{-3},5\times 10^{-4},2\times 10^{-4}$. When conducting the ``Independent Blocks '' and ``Combined Configuration '' experiments in Section~\ref{sec:am}, we just fix the learning rate of Muon. We set $\beta_1=0.8$, $\beta_2=0.95$ for Adam and set $\beta=0.95$ for Muon. When training 0.7B models on FineWeb, we conduct a grid search of learning rate on $2\times 10^{-3},1\times 10^{-3},5\times 10^{-4},2\times 10^{-4}$. We set $\beta_1=0.9$, $\beta_2=0.95$ for Adam and set $\beta=0.95$ for Muon. We do not adopt group query attention in the structure; thus, the parameter sizes of $W_Q$, $W_K$, $W_V$, and $W_O$ are the same. We conduct experiments on 8 A100 with 80 GB memory. 

% \subsection{Isotropicity Metrics Explanations}\label{app:metrics}

\subsection{Dataset Details for the Heavy-Tail Knowledge Task}
\label{app:qadataset_details} Following \cite{allen2024physics}, the foundation
of our knowledge-intensive task is a set of question-answering (QA) pairs
derived from synthetically generated biographies. Each biography is constructed from
a combination of seven key attributes: name, birthdate, birthplace, educational institution,
major, employer, and workplace. The attribute values are sampled from predefined
lists, creating a diverse set of entities. Specifically, we use approximately
400 first names, 1000 surnames, 300 educational institutions, 100 majors, and
300 employers. Each synthetic individual is assigned a unique combination of these
attributes, forming a distinct biographical profile. For example, a generated
biography might look like this:
\begin{quote}
    \textit{\textbf{Ashton Hilda Older} has a birthday that falls on \textbf{February
    01, 2063}. \textbf{Miami, FL} is the birthplace of he. He is an alumnus of \textbf{Saddleback
    College}. He has a \textbf{General Literature} education. He works closely with
    \textbf{BlockFi}. For professional growth, he chose to relocate to \textbf{Jersey
    City}.}
\end{quote}
This text is generated by combining the \textbf{structured attributes} (name,
date, location, etc.) with a set of sentence templates.

A predefined set of QA templates is then used to generate the final training
data. These templates contain placeholders corresponding to the biographical attributes.
By formatting these templates with the information from each synthetic biography,
we generate a collection of concrete QA pairs for each entity. For example, for the
entity ``Ashton Hilda Older'', we can generate the following six QA pairs:

\vspace{0.3cm}
\noindent
\begin{minipage}[t]{0.5\textwidth}
    \begin{enumerate}
        \item What is the birth date of Ashton Hilda Older? \\
            \textbf{Answer: February 01, 2063.}

        \item What is the birth city of Ashton Hilda Older? \\
            \textbf{Answer: Miami, FL.}

        \item Which university did Ashton Hilda Older study? \\
            \textbf{Answer: Saddleback College.}
    \end{enumerate}
\end{minipage}%
\begin{minipage}[t]{0.5\textwidth}
    \begin{enumerate}
        \setcounter{enumi}{3}

        \item What major did Ashton Hilda Older study? \\
            \textbf{Answer: General Literature.}

        \item Which company did Ashton Hilda Older work for? \\
            \textbf{Answer: BlockFi.}

        \item Where did Ashton Hilda Older work? \\
            \textbf{Answer: Jersey City.}
    \end{enumerate}
\end{minipage}
\vspace{0.3cm}

To evaluate the optimizers on a knowledge-intensive task with data imbalance, we
constructed a synthetic dataset where the number of question-answering (QA) samples
per class follows a power-law distribution. This is designed to simulate real-world
scenarios where a few entities (the ``head'') are highly represented, while most
entities (the ``tail'') are rare.

The generation process is controlled by an integer parameter, $m$. The classes are
organized into $m+1$ groups, indexed from $g=0$ to $m$.
\begin{itemize}
    \item Group $g$ contains $N_{g}$ classes, where $N_{0}= 1$ and
        $N_{g}= 2^{g-1}$ for $g > 0$.

    \item Each class within group $g$ is allocated a specific number of ``selections,''
        $S_{g}= 2^{m-g}$.

    \item For each selection, we generate $n_{qa}$ unique QA pairs by formatting
        templates with biographical information corresponding to that class.
\end{itemize}
Thus, the total number of QA samples for any given class in group $g$ is
$S_{g}\times n_{qa}$. This structure ensures that the single class in group 0 has
the most samples, while the numerous classes in group $m$ have the fewest.

In our experiment, we set the parameters to $m=15$ and $n_{qa}=6$. This results
in a dataset with a total of $2^{15}= 32,768$ classes. The number of samples per
class ranges from $196,608$ for the head class (group 0) down to just $6$ for each
of the $16,384$ tail classes (group 15). The final distribution is visualized in
Figure~\ref{fig:heavy_tail_qa}(a) in the main text.

To evaluate the model's performance on this pure memory task, we measure the First
Token Accuracy (FTA) on the answers. This metric assesses the model's ability to
correctly recall information by checking if the first generated token of the
answer matches the ground truth. Furthermore, to understand how optimizers
handle data imbalance, we analyze the FTA across different data frequency groups,
from high-frequency (head) to low-frequency (tail) data.

\subsection{Experimental Details About Angles Between Associative Memories
Embeddings}
\label{app:exp_setting_angles} Following \cite{fang2024alphaedit}, we analyze
the associative memories in the \ac{ffn} modules. To
obtain $E_i$, we use the activations within the feed-forward modules, and for
$\tilE_i$, we take the corresponding module outputs. We evaluate knowledge items
from two widely used datasets: Counterfact~\citep{meng2022locating} and ZsRE~\citep{levy2017zero}.
Results on Counterfact are shown in Figure~\ref{fig:kv_angles_mcf}, while results
on ZsRE are provided in Figure~\ref{fig:kv_angles_zsre} in Appendix~\ref{app:angles}.

\subsection{Experimental Details of One-layer Models}\label{app:toy_details}

We set the hyperparameters as $K=d=999$, $\alpha=0.8$, $\beta=0.2$. For the support-decoupled setting, we set $E$ and $\tilE$ as identity matrices. For the support-coupled setting, we set $E$ and $\tilE$ according to the construction presented in the proof of Theorem~\ref{thm:comp} in Appendix~\ref{app:comp}.

\section{Additional Experimental Results}

\subsection{MaxLogit per Layer on the 160M NanoGPT model via Muon Optimizer}
\label{app:max_logit}

In this subsection, we present the MaxLogit values for each layer of the 160M NanoGPT
model trained using the Muon Optimizer. Following Gemma~3~\citep{team2025gemma},
we introduce RMSNorm to the attention mechanism. The attention mechanism in our model
is defined as follows:
\begin{align*}
    O = \text{softmax}(\tilde{Q}\tilde{K}^{T})V, \quad \tilde{Q}= \text{RMSNorm}(Q), \quad \tilde{K}= \text{RMSNorm}(K)
\end{align*}
where RMSNorm is defined as
$\text{RMSNorm}(x) = \frac{x}{\sqrt{\frac{1}{d}\sum_{i=1}^{d}x_{i}^{2}}}$, with $d$
being the dimension of $x$. MaxLogit is defined as:
\begin{align*}
    S_{\max}= \max_{i,j}\tilde{q}_{i}\cdot \tilde{k}_{j}
\end{align*}
representing the maximum value in the attention scores before softmax normalization.

The MaxLogit values for each layer are summarized in Table~\ref{tab:max_logit}.
\begin{table}[H]
    \centering
    \caption{MaxLogit values per layer on the 160M NanoGPT model via Muon
    Optimizer.}
    \label{tab:max_logit} \small
    \setlength{\tabcolsep}{4pt}
    \begin{tabular}{c|cccccccccccc}
        \hline
        \textbf{Layer}    & 1     & 2     & 3     & 4     & 5     & 6     & 7     & 8     & 9     & 10    & 11    & 12    \\
        \hline
        \textbf{MaxLogit} & 8.396 & 6.880 & 6.009 & 7.676 & 6.349 & 5.890 & 7.688 & 6.314 & 6.205 & 5.613 & 6.033 & 6.371 \\
        \hline
    \end{tabular}
\end{table}

Recent reports~\cite{team2025kimi} have shown a potential ``MaxLogit explosion''
phenomenon, where $S_{\max}$ grows steadily (often near-linearly) during
training, leading to overly peaked attention, gradient spikes, and degraded optimizer
comparisons. We included this measurement to rule out the possibility that Muon's
comparatively smaller impact on the QK blocks (relative to VO/FFN) is simply due
to suppressing such an instability. In our 160M setting, with RMSNorm applied to
both $Q$ and $K$ (following Gemma~3), the per-layer MaxLogit values remain moderate
and show no runaway growth. Thus, for this model size and normalization scheme,
differences in Muon's effectiveness across components cannot be attributed to
avoiding a MaxLogit explosion in attention.

\subsection{Scaling to the 0.7B NanoGPT Model}\label{app:scaleup}

To evaluate the scalability of our findings, we extend our experiments from the 160M
model to a larger 0.7B parameter model. This section presents the results of this
scaled-up analysis, examining whether the advantages of Muon observed in the
smaller model persist at a larger scale. 

\begin{figure}[H]
    \centering
    \subfigure[Non-gated FFN]{ \includegraphics[width=0.46\textwidth]{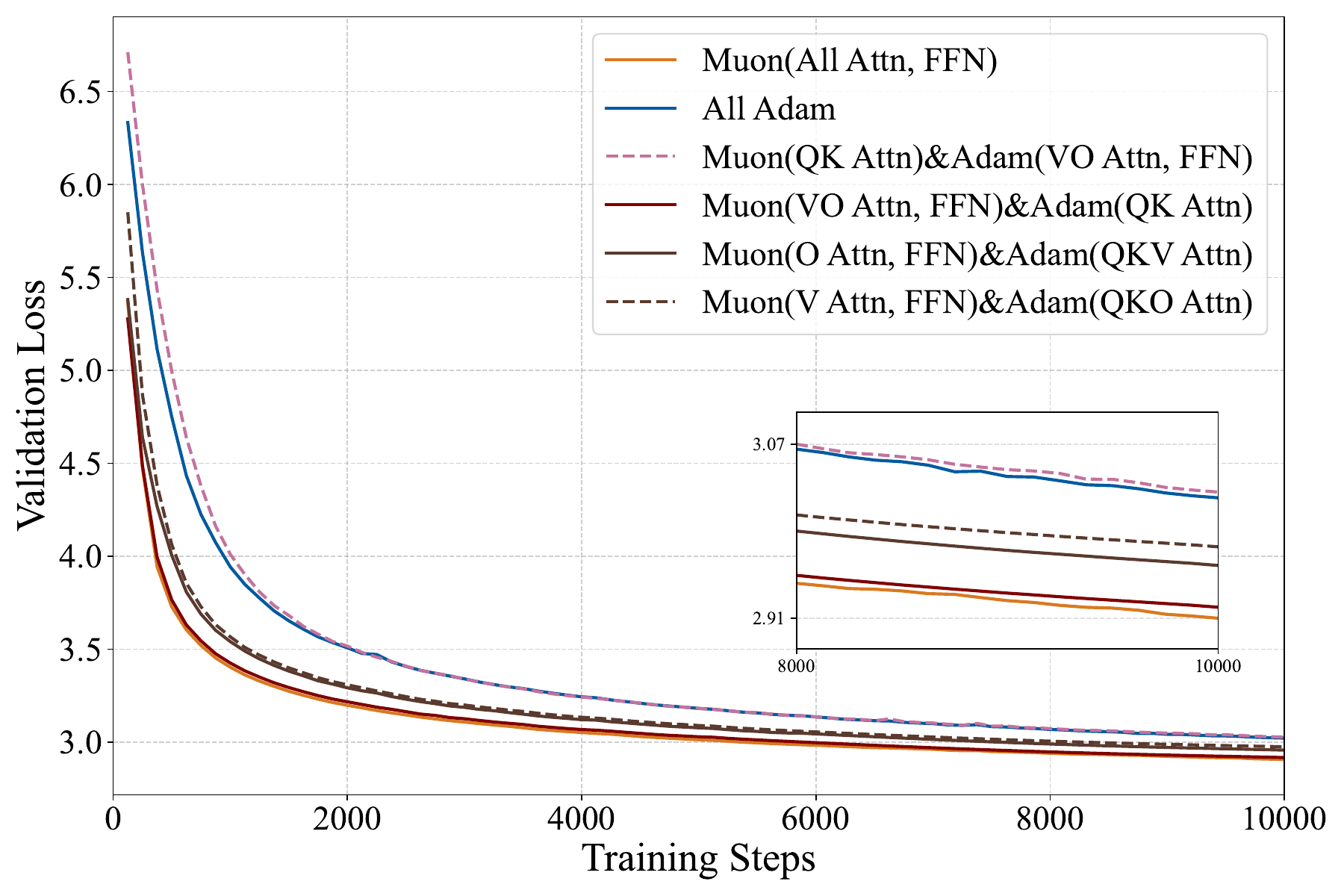}}
    \subfigure[Gated FFN]{ \includegraphics[width=0.46\textwidth]{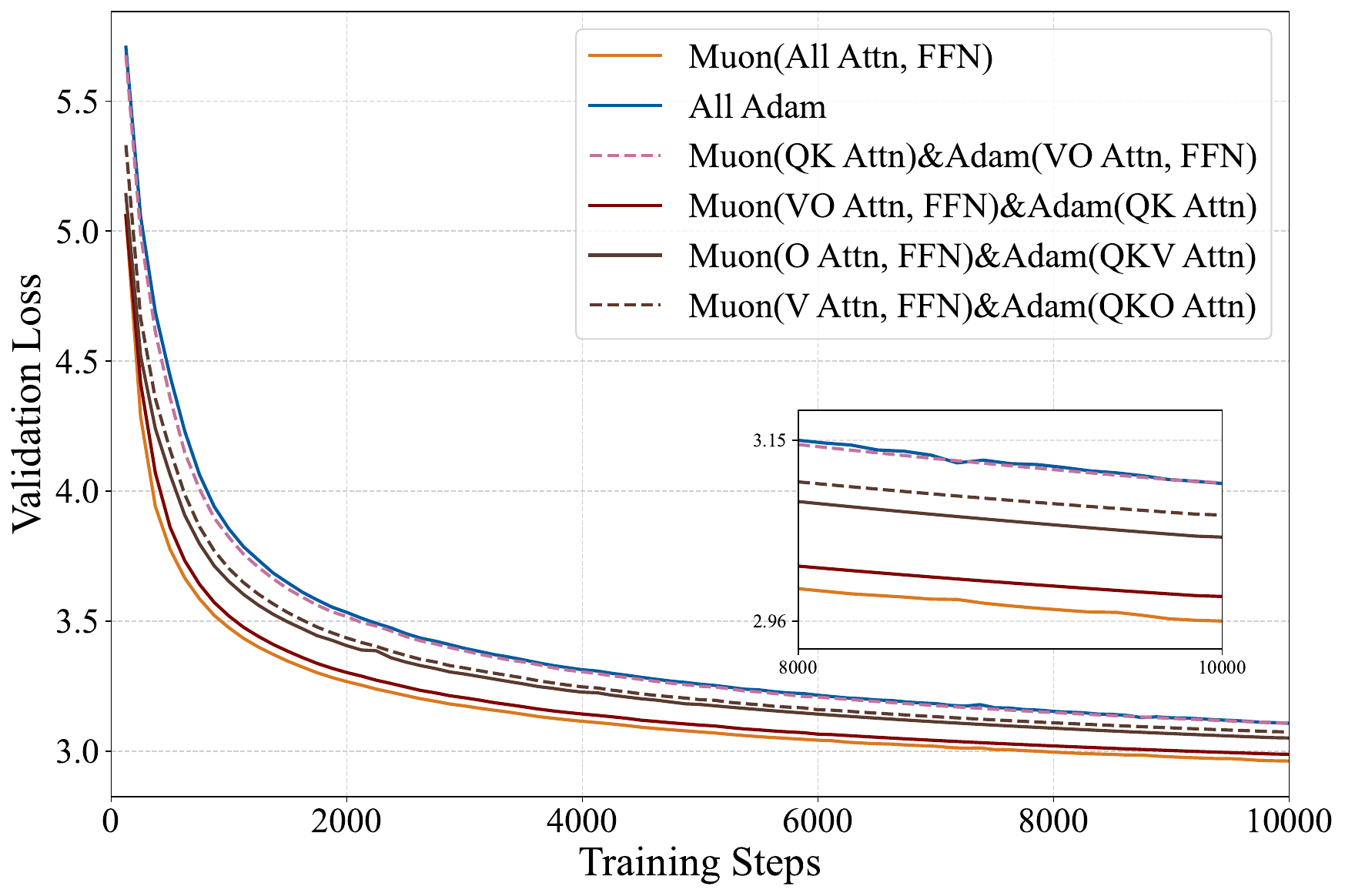}}
    \caption{Validation loss comparison on the 0.7B NanoGPT model. (a) Combined configuration with non-gated feed-forward networks.(b) Combined configuration with gated feed-forward networks.}
    \label{fig:valloss_large_nogated}
\end{figure}

\begin{figure}[H]
    \centering
     \includegraphics[width=1.0\textwidth]{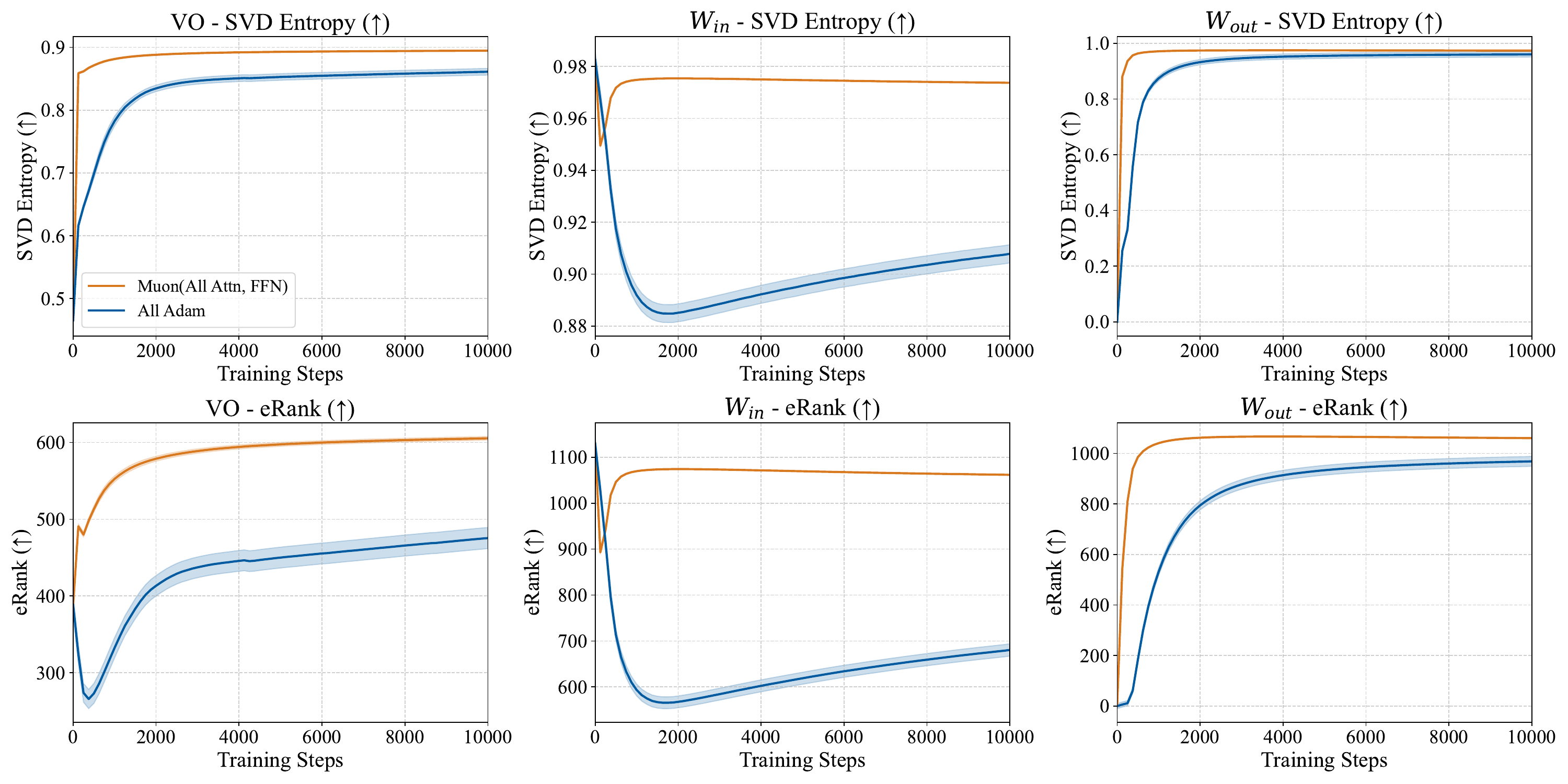}
    %\subfigure[Gated FFN]{ \includegraphics[width=0.49\textwidth]{figures/large/valloss_large_gated.pdf}}
    \caption{Spectral Dynamics of Weight Matrices During Training on the
    0.7B NanoGPT model.}
    \label{fig:svd_large_nogated}
\end{figure}

\begin{figure}[H]
    \centering
     \includegraphics[width=1.0\textwidth]{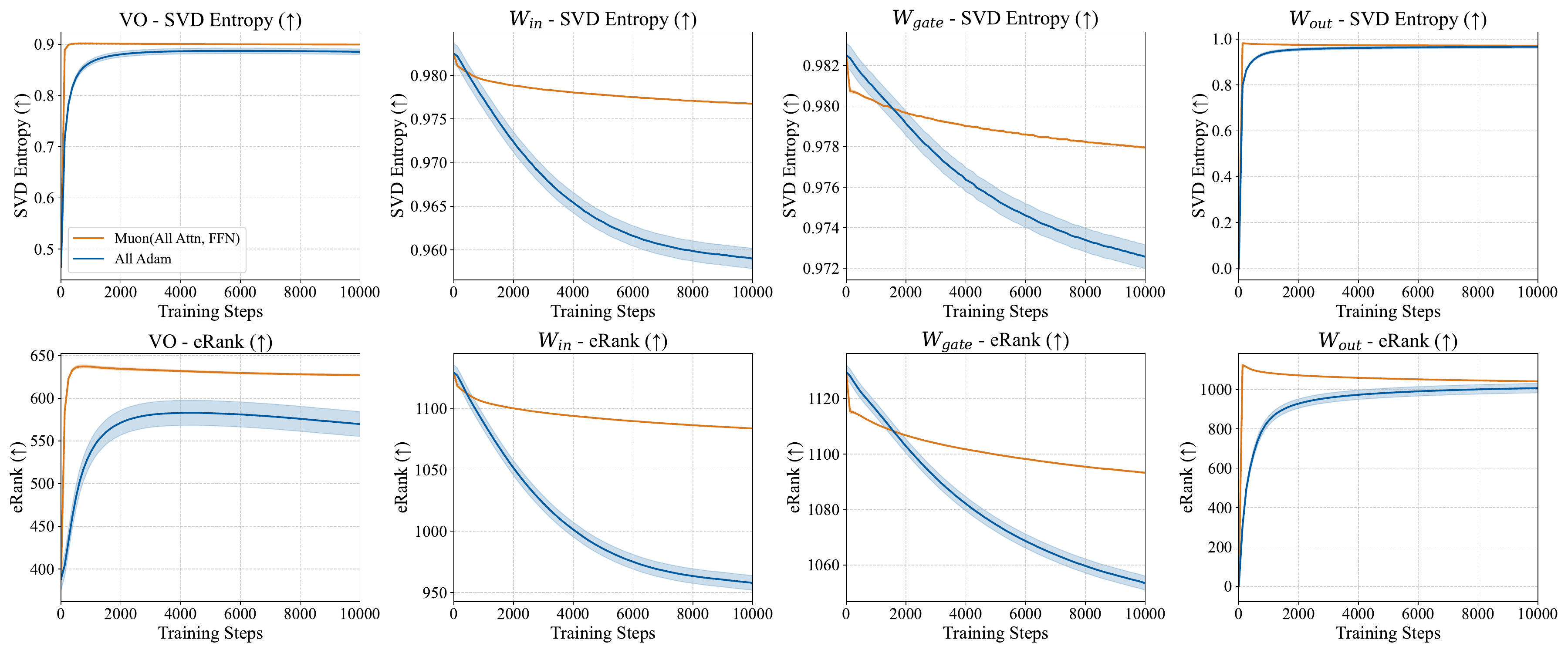}
    %\subfigure[Gated FFN]{ \includegraphics[width=0.49\textwidth]{figures/large/valloss_large_gated.pdf}}
    \caption{Spectral Dynamics of Weight Matrices During Training on the
    0.7B NanoGPT model with the Gated FFN.}
    \label{fig:svd_large_gated}
\end{figure}

Figure~\ref{fig:valloss_large_nogated}
shows the validation loss curves for various optimizer configurations. Consistent
with our findings on the 160M model, applying Muon to all components achieves the lowest validation loss, outperforming Adam
baseline. The hybrid experiments further reinforce our earlier conclusions: applying
Muon to only the VO and FFN components yields performance nearly identical to that
of the full Muon optimizer, whereas applying it only to the QK components offers
little advantage over Adam.

The spectral dynamics, shown in Figures~\ref{fig:svd_large_nogated} and~\ref{fig:svd_large_gated}, also align
with Observation~2. For the VO, $W_{\text{in}}$, $W_{\text{gate}}$ (in model with Gated FFN) and $W_{\text{out}}$ matrices,
Muon leads to higher SVD entropy and eRank compared to Adam, indicating
that it encourages the learning of more distributed, higher-dimensional representations. Overall, these results
demonstrate that the benefits of Muon and the underlying mechanisms scale to larger models.

\subsection{Additional Results about Spectral Dynamics of Transformer Weight
Matrices During Training}\label{app:svd}

To complement the main-text analysis (Fig.~\ref{fig:fineweb_160m_svd}), we also evaluate
spectral dynamics during training for the 160M NanoGPT model with both non-gated
and gated feed-forward networks (Fig.~\ref{fig:fineweb_160m_gated_svd}). The analysis
includes $W_{\inn}$ for both configurations, as well as the gate matrix
$W_{\gate}$ for the gated version. The conclusions are consistent across all three
matrices and mirror the non-gated setting: with Muon, SVD entropy and eRank increase,
while Top-$k$ energy and the $Q_{75/25}$ ratio decrease, consistent with Observation~2
in the main text.

\begin{figure}[H]
    \centering
    \subfigure[$W_{\inn}$ (Non-Gated FFN)]{ \includegraphics[width=0.32\textwidth]{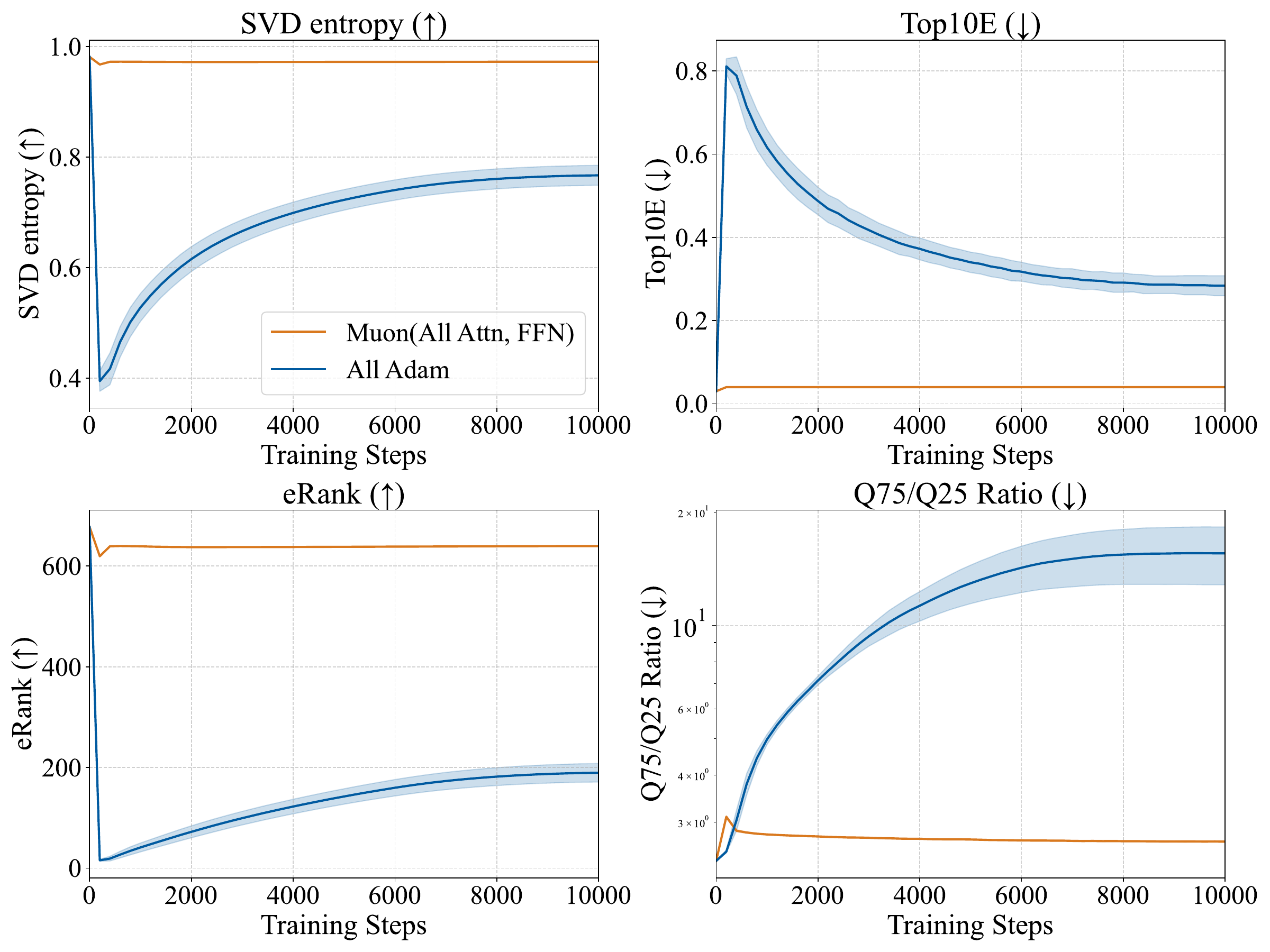}}
    \subfigure[$W_{\inn}$ (Gated FFN)]{ \includegraphics[width=0.32\textwidth]{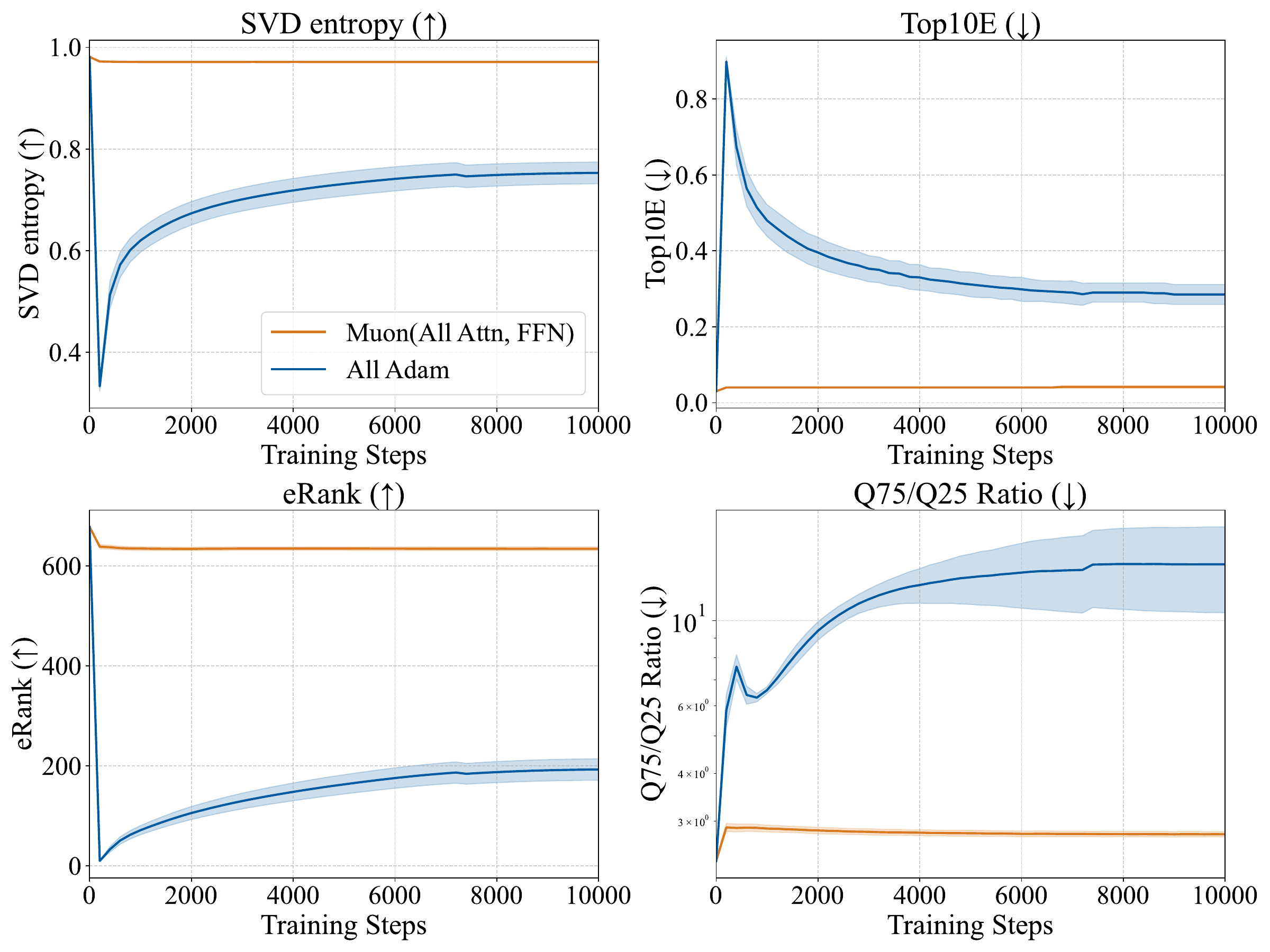}}
    \subfigure[$W_{\gate}$]{ \includegraphics[width=0.32\textwidth]{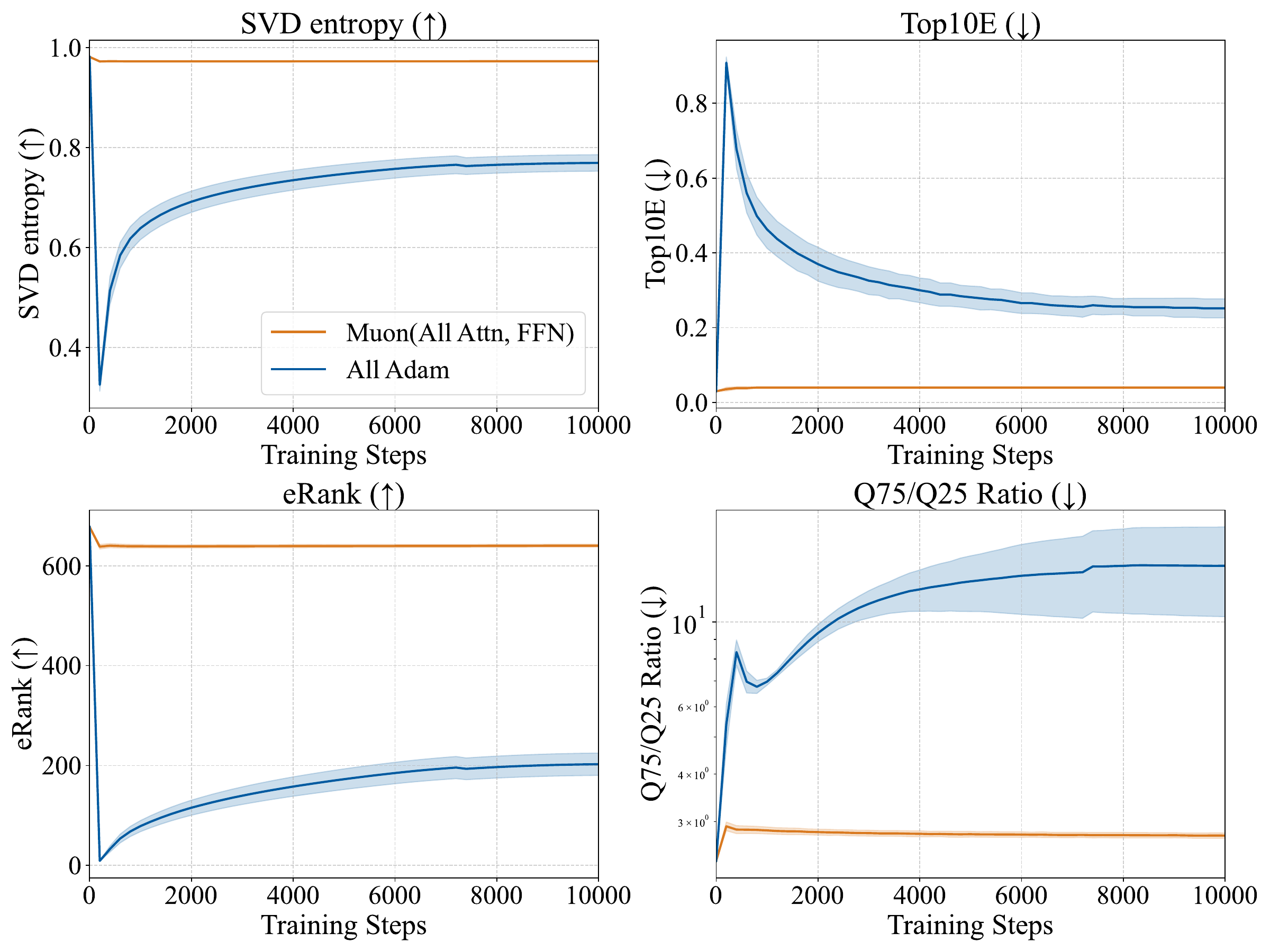}}
    \caption{Spectral Dynamics of FFN Weight Matrices During Training on the
    160M NanoGPT model. Each panel reports four metrics characterizing singular value
    distributions: SVD entropy, Top10E, eRank, and Q75/Q25 ratio. The subplots
    correspond to different weight matrices: (a) $W_{\text{in}}$ (non-gated), (b)
    $W_{\text{in}}$ (gated), and (c) $W_{\text{gate}}$ (gated).}
    \label{fig:fineweb_160m_gated_svd}
\end{figure}

\subsection{Detailed Experiment Results about Heavy-Tail Imbalance Knowledge
Task}

To complement the qualitative trends shown in Section~\ref{sec:knowledge} (Fig.~\ref{fig:heavy_tail_qa}),
we report the exact First Token Accuracy (FTA) for selected tail groups at three
training checkpoints (2k, 5k, 10k steps). We focus on groups $g=11,13,15$, which
represent increasingly rare (mid–tail, tail, extreme tail) frequency bands in the
power-law distribution (recall that larger $g$ implies fewer samples per class).
The tables contrast full Muon, Adam, SGD+Momentum, and two hybrid configurations
(Muon applied only to VO\&FFN or only to QK). The numbers highlight: (i) Muon's rapid
convergence on rare groups (already strong by 2k, near-saturated by 5k), (ii) Adam's
persistent head–tail gap, and (iii) the dominant contribution of applying Muon
to VO\&FFN for tail generalization (the VO\&FFN hybrid closely tracks full Muon,
whereas the QK-only hybrid lags). These quantitative results substantiate
Observation~3 that Muon delivers more balanced learning.

\begin{table}[H]
    \centering
    \caption{Heavy-tail knowledge task: Group performance by optimizer (2{,}000
    steps)}
    \label{tab:optimizer_2000}
    \small
    \begin{tabular}{c ccccc}
        \toprule \multirow{2}{*}{\textbf{Group}} & \multicolumn{5}{c}{\textbf{Optimizer}} \\
        \cmidrule(lr){2-6}                       & \textbf{Muon}                         & \textbf{Adam} & \textbf{SGD+Mom.} & \textbf{Muon(VO, FFN)} & \textbf{Muon(QK)} \\
        \midrule \textbf{11}                     & \textbf{0.854 ± 0.029}                & 0.312 ± 0.043 & 0.156 ± 0.037     & 0.814 ± 0.022         & 0.472 ± 0.041     \\
        \textbf{13}                              & \textbf{0.386 ± 0.029}                & 0.146 ± 0.015 & 0.120 ± 0.012     & 0.256 ± 0.030         & 0.154 ± 0.032     \\
        \textbf{15}                              & \textbf{0.140 ± 0.027}                & 0.090 ± 0.031 & 0.082 ± 0.013     & 0.114 ± 0.023         & 0.086 ± 0.037     \\
        \bottomrule
    \end{tabular}
\end{table}

\begin{table}[H]
    \centering
    \caption{Heavy-tail knowledge task: Group performance by optimizer (5{,}000
    steps)}
    \label{tab:optimizer_5000}
    \small
    \begin{tabular}{c ccccc}
        \toprule \multirow{2}{*}{\textbf{Group}} & \multicolumn{5}{c}{\textbf{Optimizer}} \\
        \cmidrule(lr){2-6}                       & \textbf{Muon}                         & \textbf{Adam} & \textbf{SGD+Mom.} & \textbf{Muon(VO, FFN)} & \textbf{Muon(QK)} \\
        \midrule \textbf{11}                     & \textbf{0.996 ± 0.006}                & 0.936 ± 0.039 & 0.314 ± 0.021     & 0.992 ± 0.005         & 0.970 ± 0.007     \\
        \textbf{13}                              & \textbf{0.964 ± 0.023}                & 0.298 ± 0.074 & 0.148 ± 0.013     & 0.934 ± 0.015         & 0.354 ± 0.032     \\
        \textbf{15}                              & \textbf{0.320 ± 0.028}                & 0.110 ± 0.027 & 0.084 ± 0.011     & 0.254 ± 0.026         & 0.118 ± 0.019     \\
        \bottomrule
    \end{tabular}
\end{table}

\begin{table}[H]
    \centering
    \caption{Heavy-tail knowledge task: Group performance by optimizer (10{,}000
    steps)}
    \label{tab:optimizer_10000}
    \small
    \begin{tabular}{c ccccc}
        \toprule \multirow{2}{*}{\textbf{Group}} & \multicolumn{5}{c}{\textbf{Optimizer}} \\
        \cmidrule(lr){2-6}                       & \textbf{Muon}                         & \textbf{Adam} & \textbf{SGD+Mom.} & \textbf{Muon(VO, FFN)} & \textbf{Muon(QK)} \\
        \midrule \textbf{11}                     & 1.000 ± 0.000                         & 1.000 ± 0.000 & 0.422 ± 0.023     & 1.000 ± 0.000         & 1.000 ± 0.000     \\
        \textbf{13}                              & \textbf{1.000 ± 0.000}                & 0.890 ± 0.042 & 0.294 ± 0.013     & 0.998 ± 0.002         & 0.940 ± 0.034     \\
        \textbf{15}                              & \textbf{0.976 ± 0.006}                & 0.264 ± 0.048 & 0.126 ± 0.021     & 0.954 ± 0.021         & 0.286 ± 0.039     \\
        \bottomrule
    \end{tabular}
\end{table}

\subsection{Additional Experiment Results about Heavy-Tail Imbalance Knowledge
Task with Gated Feed-Forward Networks}\label{app:qa_gated}

This subsection complements the main heavy-tail results in Section~\ref{sec:knowledge}
by studying the gated feed-forward networks (Gated FFN) variant. We follow the same
presentation order as in the main text: first an overview figure (sample distribution
and learning curves under different optimizers), then tables reporting the exact
First-Token Accuracy (FTA) for tail groups $g\in\{11,13,15\}$ at three training checkpoints
(2k, 5k, 10k steps). The findings mirror the non-gated setting: (i) full Muon consistently
outperforms Adam and SGD+Momentum on rare classes and reaches high accuracy earlier;
(ii) the VO\&FFN-hybrid (Muon applied to VO and FFN while Adam is used for QK)
closely tracks full Muon, indicating that VO\&FFN are the primary levers for tail
generalization; (iii) the QK-only hybrid offers limited gains. Overall, the
gated FFN does not change the qualitative conclusions about where Muon helps
most. See Fig.~\ref{fig:heavy_tail_qa_gated} and Tables~\ref{tab:gated_optimizer_2000}--\ref{tab:gated_optimizer_10000}
for details.

\begin{figure}[h]
    \centering
    \subfigure[Sample/class]{ \includegraphics[width=0.32\textwidth]{figures/qa/powerlaw_samples_class.pdf}}
    \subfigure[Muon]{ \includegraphics[width=0.32\textwidth]{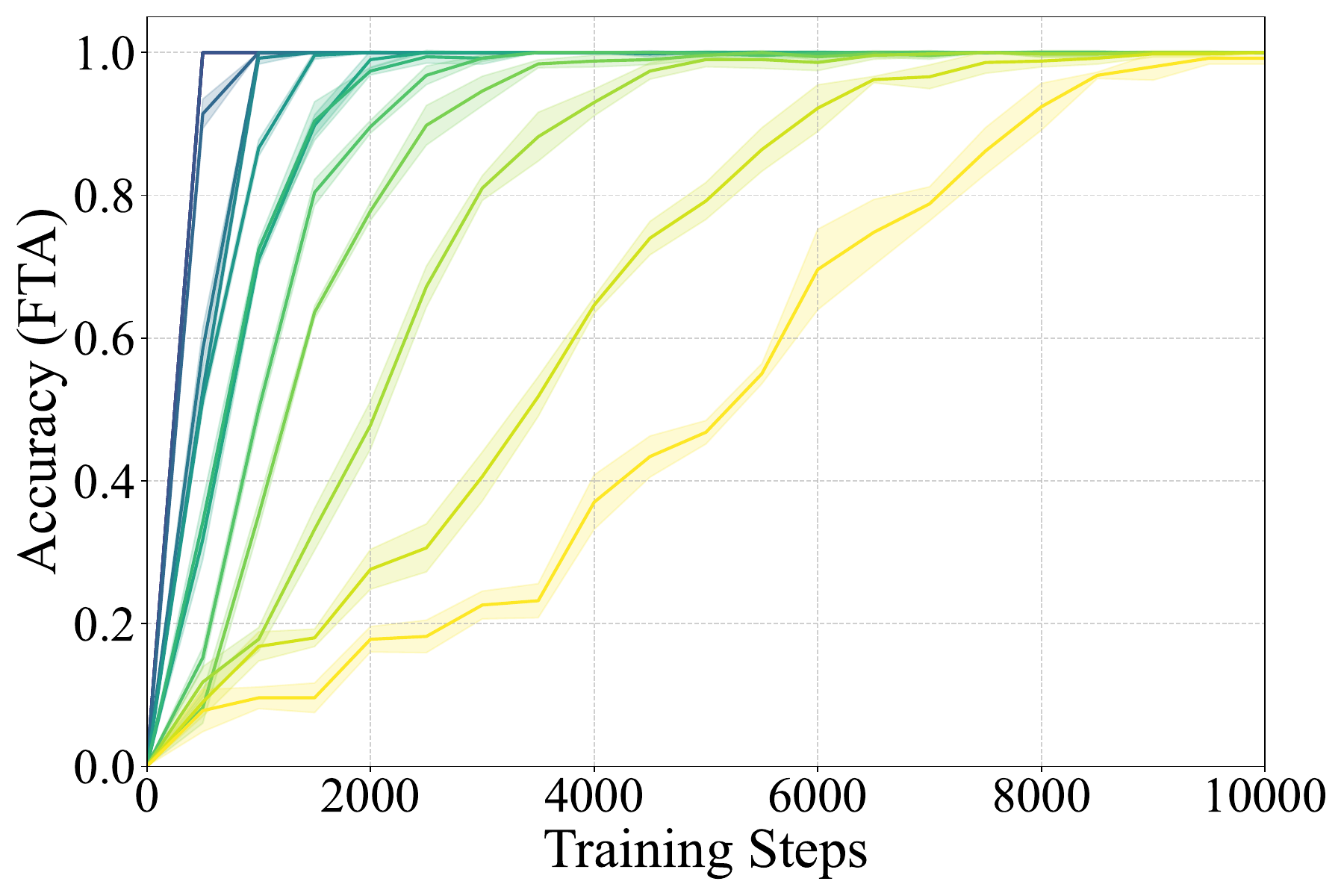}}
    \subfigure[Adam]{ \includegraphics[width=0.32\textwidth]{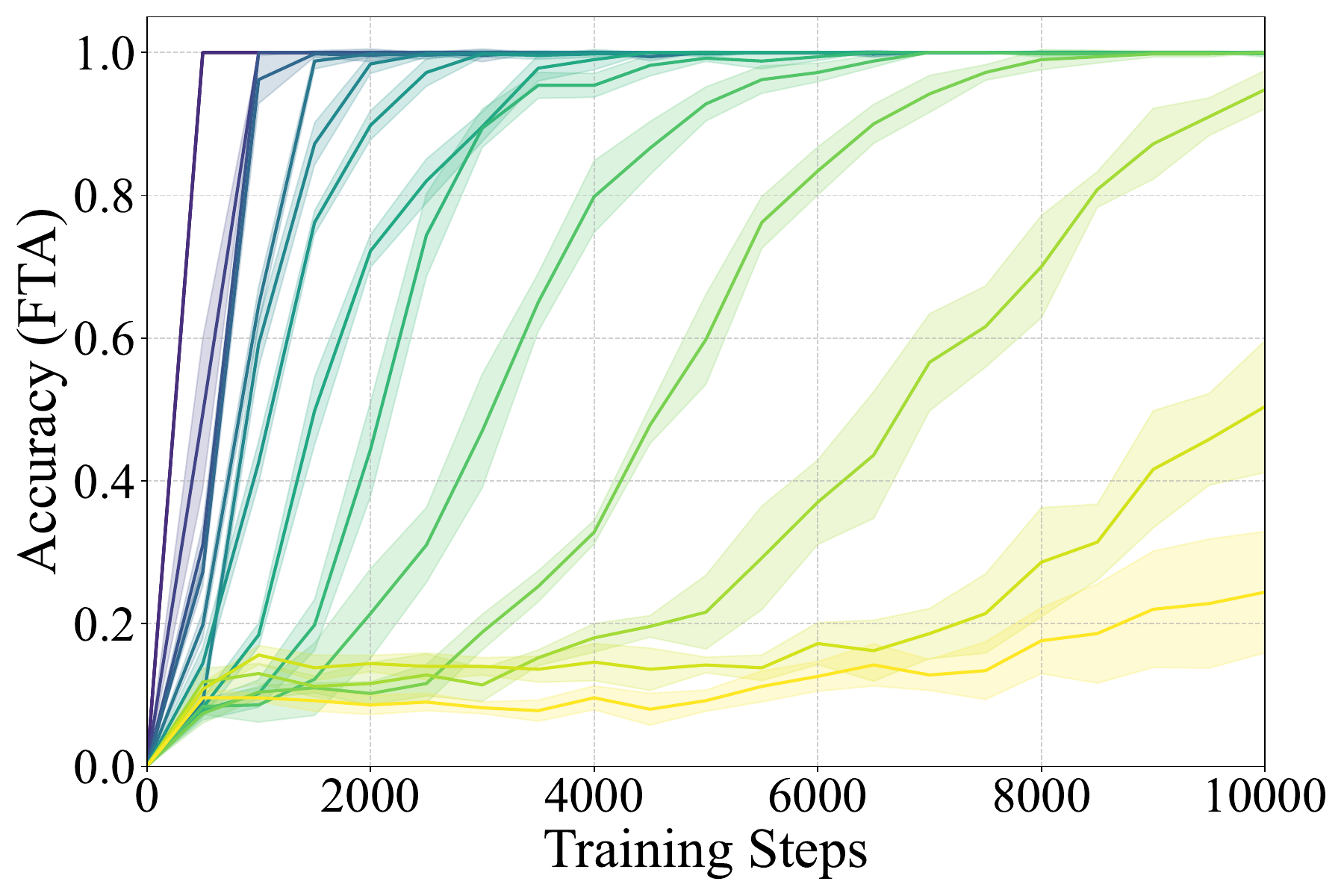}}
    \subfigure[SGD+Momentum]{ \includegraphics[width=0.32\textwidth]{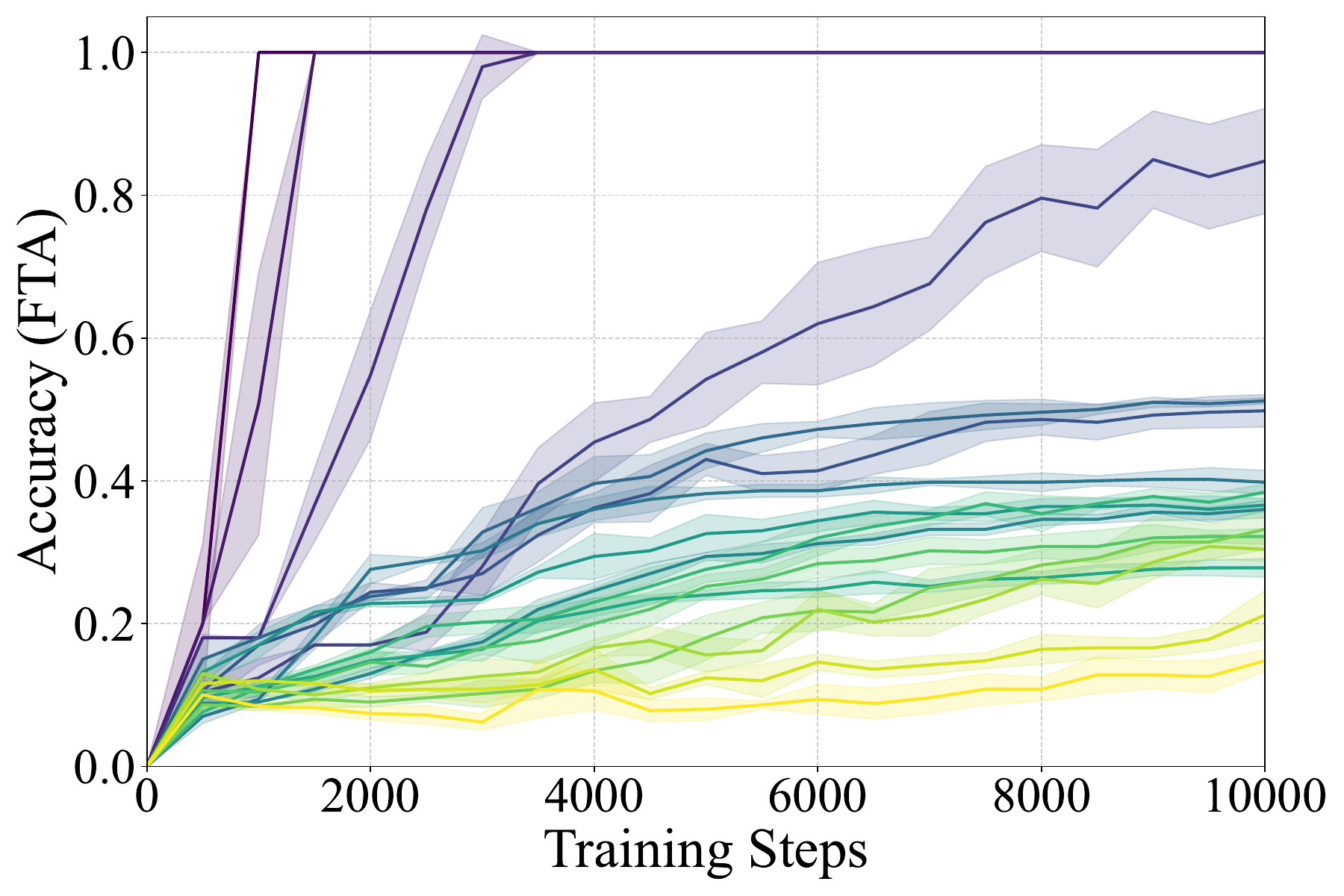}}
    \subfigure[Muon(VO, FFN)\&Adam(QK)]{ \includegraphics[width=0.32\textwidth]{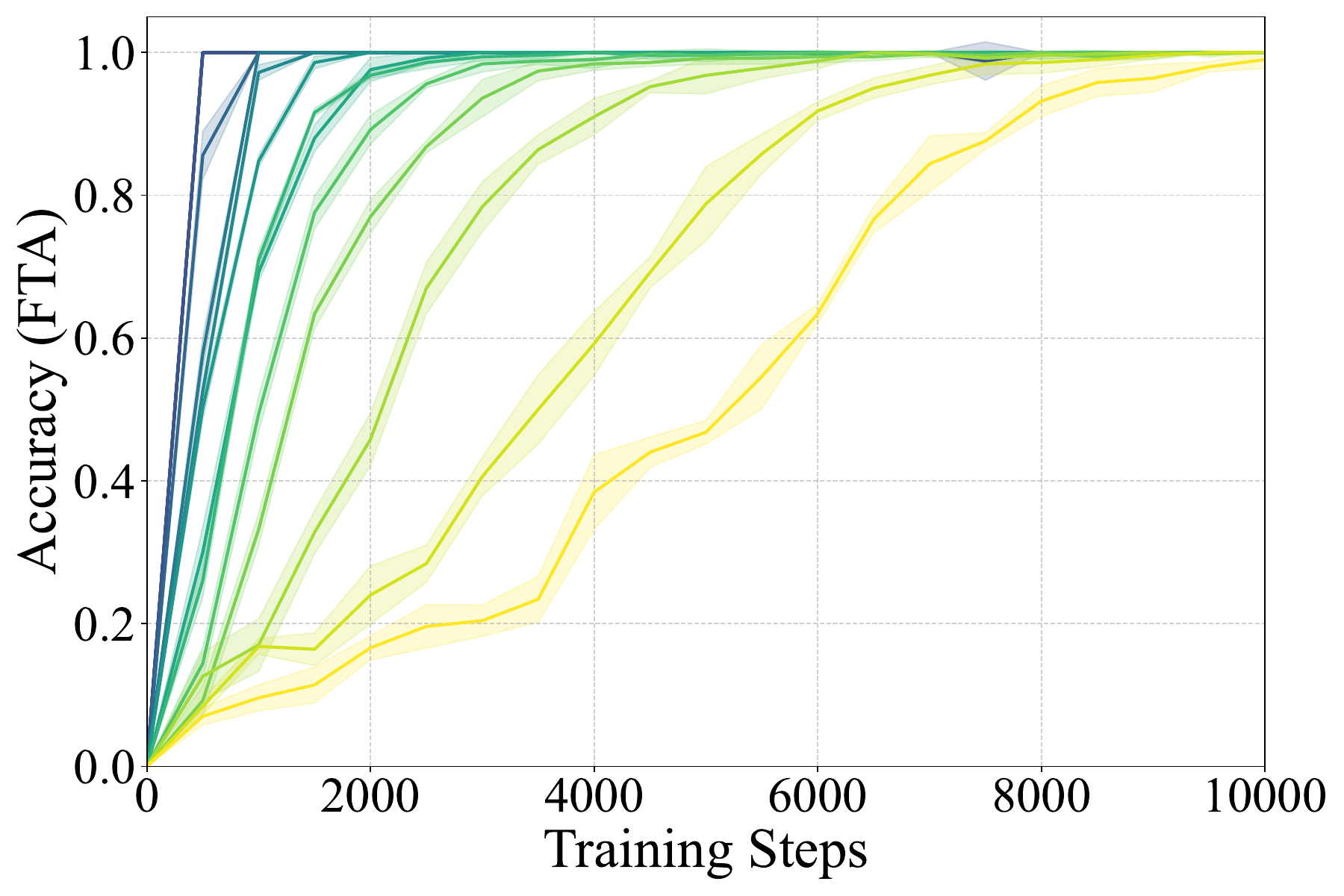}}
    \subfigure[Muon(QK)\&Adam(VO,FFN)]{ \includegraphics[width=0.32\textwidth]{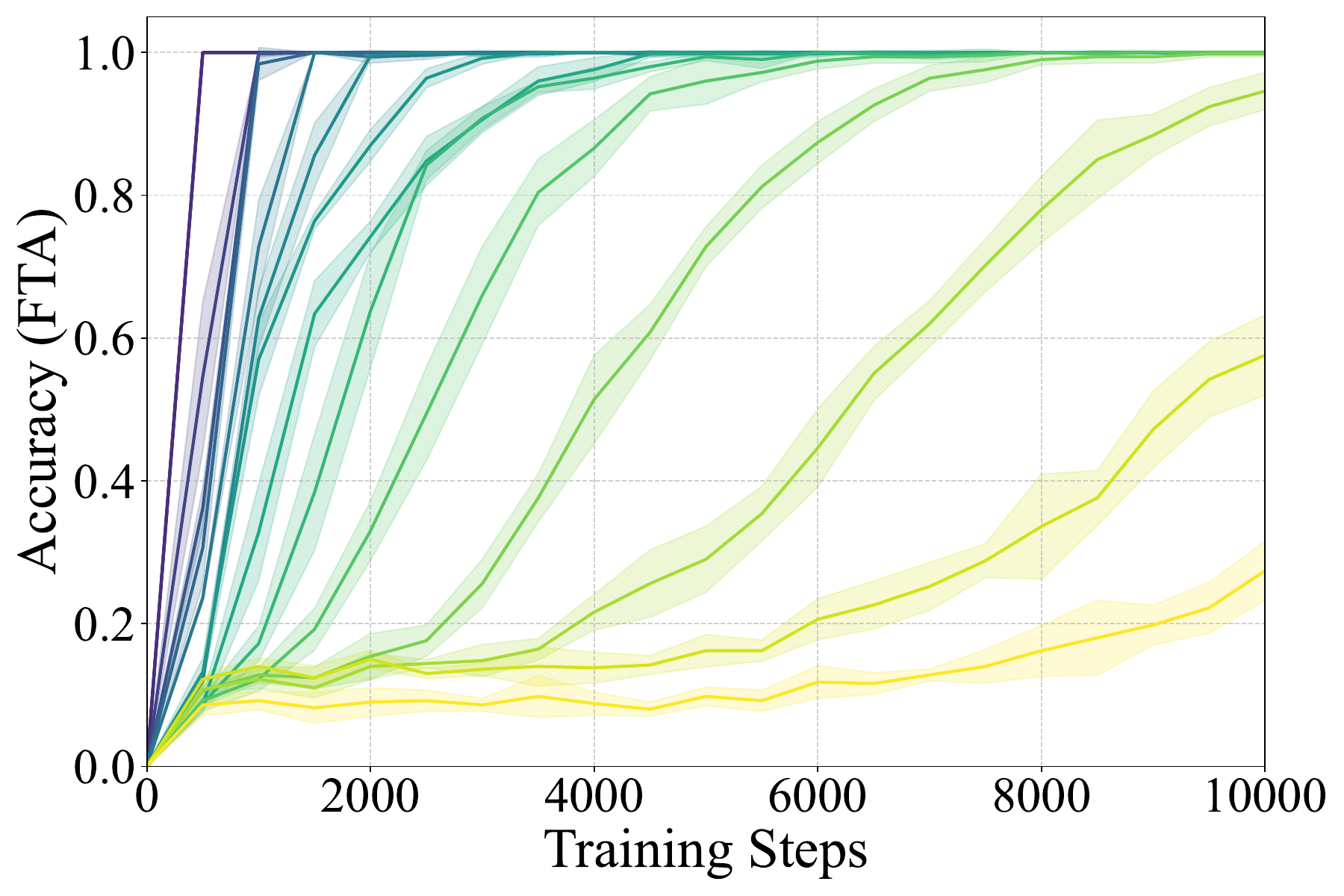}}
    \caption{Performance comparison of different optimizers on a heavy-tailed
    knowledge task with gated feed-forward networks. (a) The distribution of samples
    per class follows a power law. (b-d) Performance of Muon, Adam, and SGD+Momentum
    optimizers. (e) Muon (VO, FFN)/Adam (QK). (f) Muon (QK)/Adam (VO, FFN).}
    \label{fig:heavy_tail_qa_gated}
\end{figure}

\begin{table}[H]
    \centering
    \caption{Heavy-tail knowledge task with the Gated FFN: Group performance by optimizer
    (2{,}000 steps)}
    \label{tab:gated_optimizer_2000}
    \small
    \begin{tabular}{c ccccc}
        \toprule \multirow{2}{*}{\textbf{Group}} & \multicolumn{5}{c}{\textbf{Optimizer}} \\
        \cmidrule(lr){2-6}                       & \textbf{Muon}                         & \textbf{Adam} & \textbf{SGD+Mom.} & \textbf{Muon(VO, FFN)} & \textbf{Muon(QK)} \\
        \midrule \textbf{11}                     & \textbf{0.896 ± 0.009}                & 0.214 ± 0.063 & 0.146 ± 0.018     & 0.892 ± 0.021          & 0.330 ± 0.042     \\
        \textbf{13}                              & \textbf{0.478 ± 0.034}                & 0.116 ± 0.030 & 0.110 ± 0.007     & 0.458 ± 0.037          & 0.140 ± 0.019     \\
        \textbf{15}                              & \textbf{0.178 ± 0.018}                & 0.086 ± 0.013 & 0.074 ± 0.009     & 0.166 ± 0.017          & 0.090 ± 0.020     \\
        \bottomrule
    \end{tabular}
\end{table}

\begin{table}[H]
    \centering
    \caption{Heavy-tail knowledge task with the Gated FFN: Group performance by optimizer
    (5{,}000 steps)}
    \label{tab:gated_optimizer_5000}
    \small
    \begin{tabular}{c ccccc}
        \toprule \multirow{2}{*}{\textbf{Group}} & \multicolumn{5}{c}{\textbf{Optimizer}} \\
        \cmidrule(lr){2-6}                       & \textbf{Muon}                         & \textbf{Adam} & \textbf{SGD+Mom.} & \textbf{Muon(VO, FFN)} & \textbf{Muon(QK)} \\
        \midrule \textbf{11}                     & \textbf{0.998 ± 0.002}                & 0.928 ± 0.024 & 0.252 ± 0.016     & 0.990 ± 0.010          & 0.960 ± 0.032     \\
        \textbf{13}                              & \textbf{0.990 ± 0.010}                & 0.216 ± 0.052 & 0.156 ± 0.024     & 0.968 ± 0.028          & 0.290 ± 0.046     \\
        \textbf{15}                              & \textbf{0.510 ± 0.039}                & 0.092 ± 0.015 & 0.080 ± 0.016     & 0.468 ± 0.016          & 0.098 ± 0.013     \\
        \bottomrule
    \end{tabular}
\end{table}

\begin{table}[H]
    \centering
    \caption{Heavy-tail knowledge task with the Gated FFN: Group performance by optimizer
    (10{,}000 steps)}
    \label{tab:gated_optimizer_10000}
    \small
    \begin{tabular}{c ccccc}
        \toprule \multirow{2}{*}{\textbf{Group}} & \multicolumn{5}{c}{\textbf{Optimizer}} \\
        \cmidrule(lr){2-6}                       & \textbf{Muon}                         & \textbf{Adam} & \textbf{SGD+Mom.} & \textbf{Muon(VO, FFN)} & \textbf{Muon(QK)} \\
        \midrule \textbf{11}                     & 1.000 ± 0.000                         & 0.998 ± 0.002 & 0.322 ± 0.011     & 1.000 ± 0.000          & 1.000 ± 0.000     \\
        \textbf{13}                              & \textbf{1.000 ± 0.000}                & 0.948 ± 0.027 & 0.304 ± 0.017     & \textbf{1.000 ± 0.000} & 0.946 ± 0.026     \\
        \textbf{15}                              & \textbf{0.994 ± 0.006}                & 0.244 ± 0.085 & 0.148 ± 0.015     & 0.990 ± 0.010          & 0.274 ± 0.042     \\
        \bottomrule
    \end{tabular}
\end{table}

\subsection{Additional Results about Angles Between Associative Memories
Embeddings}
\label{app:angles}
\begin{figure}[H]
    \centering
    \includegraphics[width=0.5\linewidth]{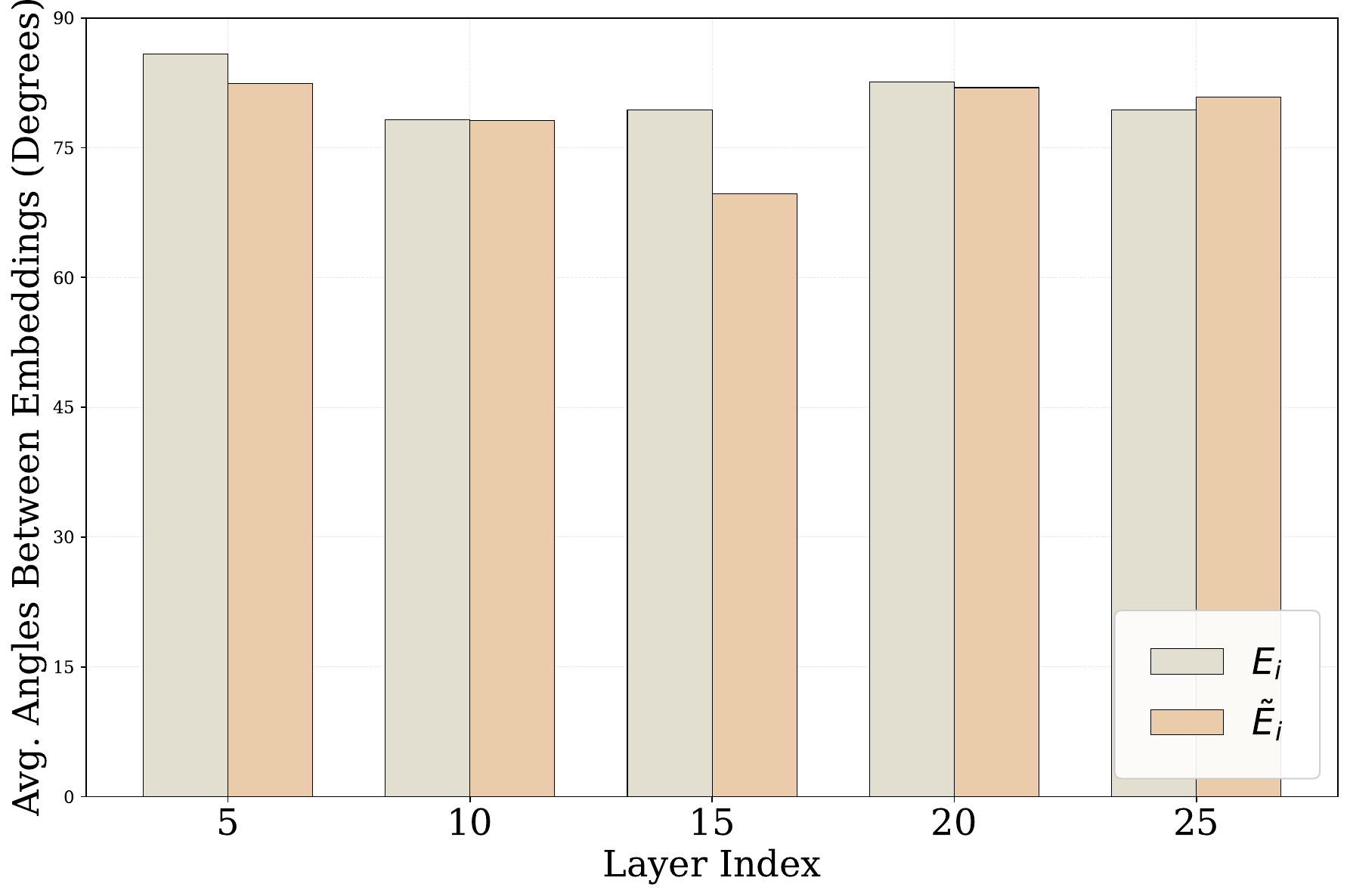}
    \caption{Average angles between $e_{s}$ or $e_{o}$ for items in ZsRE at
    layers $5$, $10$, $15$, $20$, $25$ of Llama3-8b-instruct.}
    \label{fig:kv_angles_zsre}
\end{figure}

\section{Proof of Theorem~\ref{thm:comp}}\label{app:comp}
We separately derive the results for GD, Muon, and Adam in the following proof. For all of them, we define
\begin{align}
    \eta_{\opt}^{\epsilon} = \inf \Big\{\eta\geq 0 \,\Big|\,1-\max_{k\in[K]}\big[f_W(E_k)\big]_{k}\leq\epsilon, \text{ where }W=W_0-\eta\cdot G_{\opt}(W_0) \Big\}.\label{eq:minimal_stepsize}
\end{align}
 The quantity $\eta_{\opt}^{\epsilon}$ represents the minimal step size for at least one triplet to be learned with error probability less than $\epsilon$. From the definition, we have that
 \begin{align*}
     \maxd_{\opt}^{\epsilon}\leq \min_{k\in[K]}\big[f_{-\eta_{\opt}^{\epsilon} G_{\opt}}(E_k)\big]_{k}.
 \end{align*}

\textbf{Step 1: Calculations of GD.}

We define the score of $k^{\prime}$-th object for the $k$-th subject-relation pair with the parameter $W$ as
\begin{align*}
    s(k^{\prime},k,W) = \frac{\exp(\tilE_{k^{\prime}}^{\top} W E_{k} )}{\sum_{k^{\prime\prime}=1}^{K}\exp(\tilE_{k^{\prime\prime}}^{\top} W E_{k} )}.
\end{align*}
At $W_0=0_{d_o,d_s}$, we have that
\begin{align*}
    s(k^{\prime},k,W_0)=\frac{1}{K} \text{ for all }k,k^{\prime}\in[K].
\end{align*}
Proposition~\ref{prop:grad} shows that the gradient is
\begin{align}
    -\nabla_{W} \mathcal{L}(W_0)= \frac{\alpha}{L}\tilE_{1:L}E_{1:L}^{\top}+\frac{1-\alpha}{K-L}\tilE_{L+1:K}E_{L+1:K}^{\top} - \frac{\alpha}{LK}\tilE J_{K,L}E_{1:L}^{\top}\nonumber\\-\frac{1-\alpha}{(K-L)K}\tilE J_{K,K-L}E_{L+1:K}^{\top}.\label{eq:gd}
\end{align}
From the gradient, it is easy to see that the first $L$ triplets $(s,r,o)$ share the same learning behavior, and the last $K-L$ triplets also share the same behavior. Thus, we calculate the results for $k=1$ and $k=L+1$. The calculation for $k=1$ depends on evaluating its score function, which takes the form $\eta \cdot \tilE_{k''}^\top [-\nabla_W\calL(W_0)] E_1$, for $k'' \in \{1, \dots, K\}$. Based on the gradient in \eqref{eq:gd} and the orthonormality of the embeddings, it evaluates to $\alpha/L$ for the case $k''=1$, and to $0$ for all $k'' \neq 1$.

This leads to a numerator in the softmax score of $\exp(\eta \cdot \alpha/L)$, while the denominator sum consists of one term $\exp(\eta \cdot \alpha/L)$ and $K-1$ terms of $\exp(0)=1$. A similar calculation for $k=L+1$ shows that the argument of the exponent for the correct object, $\eta \cdot \tilE_{L+1}^\top [-\nabla_W\calL(W_0)] E_{L+1}$, evaluates to $\eta \cdot (1-\alpha)/(K-L)$. By defining $\gamma_1 = \alpha/(\beta K)$ and $\gamma_2 = (1-\alpha)/((1-\beta)K)$ based on the problem setup ($L=\beta K$), we have that
%In fact, we have that
\begin{align*}
    \big[f_{-\eta \nabla_W\calL}(E_1)\big]_{1}= \frac{\exp(\eta \gamma_1)}{\exp(\eta \gamma_1)+K-1},\quad \big[f_{-\eta \nabla_W\calL}(E_{L+1})\big]_{L+1}= \frac{\exp(\eta \gamma_2)}{\exp(\eta \gamma_2)+K-1},
\end{align*}
where $\gamma_1$ and $\gamma_2$ are defined as
\begin{align*}
    \gamma_1 = \frac{\alpha}{\beta K},\quad \gamma_2 = \frac{1-\alpha}{(1-\beta) K}.
\end{align*}
Then we derive that
\begin{align}
    \eta_{\gd}^{\epsilon} = \frac{1}{\max\{\gamma_1,\gamma_2\}} \log\big[(\epsilon^{-1} -1)(K-1)\big].\label{eq:gd_time}
\end{align}
To calculate the desired quantity, we define the quantity $r(\alpha,\beta)$ to evaluate the balance of data as
\begin{align*}
    r(\alpha,\beta) = \min\{\gamma_1/\gamma_2,\gamma_2/\gamma_1\} = \min \bigg\{\frac{\alpha(1-\beta)}{\beta(1-\alpha)},\frac{\beta(1-\alpha)}{\alpha(1-\beta)}\bigg\}.
\end{align*}
Some basic calculations show that
\begin{align}
    1-\min_{k\in[K]}\big[f_{-\eta_{\gd}^{\epsilon} G_{\gd}}(E_k)\big]_{k} = \frac{\epsilon}{\epsilon+(1-\epsilon)^{r(\alpha,\beta)}\epsilon^{1-r(\alpha,\beta)}(K-1)^{r(\alpha,\beta)-1}}.\label{eq:gd_result}
\end{align}
When $r<1$, with the fact that $\frac{1}{x+1}=1-x+O(x^2)$ as $x\rightarrow 0$, we have that
\begin{align*}
    \min_{k\in[K]}\big[f_{-\eta_{\gd}^{\epsilon} G_{\gd}}(E_k)\big]_{k} = O(\epsilon^{-r(\alpha,\beta)} K^{r(\alpha,\beta)-1}).
\end{align*}
Thus, the proof for the convergence of GD has been established.

\textbf{Step 2: Calculations of Muon.}

For Muon, we first calculate the \ac{svd} of the gradient. In fact, we can write the gradient in Eqn.~\eqref{eq:gd} as
\begin{align*}
    -\nabla_{W} \mathcal{L}(W_0)&= \tilE\bigg\{\diag\bigg(\frac{\alpha}{L}\bbI_{L},\frac{1-\alpha}{K-L}\bbI_{K-L}\bigg)-\frac{1}{K}\bbI_{K}\cdot\bigg[\frac{\alpha}{L}\bbI_{L}^{\top},\frac{1-\alpha}{K-L}\bbI_{K-L}^{\top}\bigg]^{\top}\bigg\} E^{\top}\\
    & = \tilE X E^{\top}.
\end{align*}
The \ac{svd} calculation of $X=U\Sigma V^{\top}$ can be directly derived from Proposition~\ref{prop:svd_simp}. Thus, the \ac{svd} of the gradient is $-\nabla_{W} \mathcal{L}(W_0)= (\tilE \cdot U)\Sigma (E\cdot V)^{\top}$. The update quantity $G_{\muon}(W_0)= U_0 \nor(\Sigma_0)V_0^{\top}$ of Muon is 
\begin{align}
    &- G_{\muon}(W_0) \nonumber\\
    &\quad = \tilE_{1:L}R_{L,L-1}R_{L,L-1}^{\top}E_{1:L}^{\top}+\tilE_{L+1:K}R_{K-L,K-L-1}R_{K-L,K-L-1}^{\top}E_{L+1:K}^{\top}\nonumber\\
    &\quad\qquad  +\frac{1}{\sqrt{K\big[\alpha^2(K-L)^3+(1-\alpha)^2L^3\big]}}\big((K-L)\tilE_{1:L}\bbI_{L}-L\tilE_{L+1:K}\bbI_{K-L}\big)\nonumber\\
    &\qquad\qquad\qquad\qquad\qquad\qquad\qquad\qquad\qquad\cdot\bigg(\frac{(K-L)\alpha}{L}\bbI_{L}^{\top}E_{1:L}^{\top}-\frac{L(1-\alpha)}{K-L}\bbI_{K-L}^{\top}E_{L+1:K}^{\top}\bigg)\nonumber\\
    &\quad = \tilE_{1:L}E_{1:L}^{\top}+\tilE_{L+1:K}E_{L+1:K}^{\top}\nonumber\\
    &\qquad+\frac{1}{K}\bigg\{\frac{1}{\beta}\bigg(\frac{(1-\beta)^2\alpha}{\lambda}-1\bigg)\tilE_{1:L}J_{L,L}E_{1:L}^{\top}\nonumber\\
    &\qquad\qquad+\frac{1}{1-\beta}\bigg(\frac{\beta^2(1-\alpha)}{\lambda}-1\bigg)\tilE_{L+1:K}J_{K-L,K-L}E_{L+1:K}^{\top}\nonumber\\
    &\qquad\qquad\qquad\qquad -\beta(1-\alpha)\tilE_{1:L}J_{L,K-L}E_{L+1:K}^{\top}-\alpha(1-\beta)\tilE_{L+1:K}J_{K-L,L}E_{1:L}^{\top}\bigg\},\label{eq:muon}
\end{align}
where   $\lambda = \sqrt{\alpha^2(1-\beta)^3+(1-\alpha)^2\beta^3}$, the matrices $R_{L,L-1}$ and $R_{K-L,K-L-1}$ are defined in Proposition~\ref{prop:svd_simp}, and the second equality results from the following facts
\begin{align*}
    R_{L,L-1}R_{L,L-1}^{\top} &= I_{L,L}-\frac{1}{L}\bbI_{L}\bbI_{L}^{\top}, \\ R_{K-L,K-L-1}R_{K-L,K-L-1}^{\top} &= I_{K-L,K-L}-\frac{1}{K-L}\bbI_{K-L}\bbI_{K-L}^{\top}.
\end{align*}
Although the gradient is composed of heterogeneous components from $\tilE_{1:L}, E_{1:L}$ and $\tilE_{L+1:K}, E_{L+1:K}$, we can bound the convergence rate of $[f_{-\eta G_{\muon}}(E_k)]_k$ for any $k$: an upper (resp. lower) bound is obtained by increasing (resp. decreasing) the coefficient of $\tilE_k E_k^{\top}$ while decreasing (resp. increasing) that of $\tilE_{k^{\prime}} E_{k}^{\top}$ for $k^{\prime} \neq k$. In fact, Eqn.~\eqref{eq:muon} implies that there exists a constant $C > 0$ such that the dynamics of the fastest- and slowest-learning triplets are bounded by those along the following two update directions.
\begin{align*}
    - G_{\muon}^{+}(W_0) & = \bigg(1+\frac{2C}{K}\bigg) (\tilE_{1:L}E_{1:L}^{\top}+\tilE_{L+1:K}E_{L+1:K}^{\top})-\frac{C}{K}\cdot \tilE J_{K,K}E^{\top}\nonumber\\
    - G_{\muon}^{-}(W_0) & = \bigg(1-\frac{2C}{K}\bigg) (\tilE_{1:L}E_{1:L}^{\top}+\tilE_{L+1:K}E_{L+1:K}^{\top})+\frac{C}{K}\cdot\tilE J_{K,K}E^{\top}.
\end{align*}
%Concretely, the rate of score increase for the fastest triplet is lower than that of $G_{\muon}^{+}(W_0)$, while the rate for the slowest triplet exceeds that of $G_{\muon}^{-}(W_0)$. Thus, we only need to focus on $G_{\muon}^{+}(W_0)$ and $G_{\muon}^{-}(W_0)$ to calculate the desired quantity. Following the similar procedures for GD to derive Eqn.~\eqref{eq:gd_result}, we have that for any $\eta$ such that $\max_{k\in[K]}\big[f_{W_\eta}(E_k)\big]_{k}\geq 1- \epsilon$ (where $W_\eta=W_0-\eta\cdot G_{\muon}(W_0)$), the following holds
Concretely, the rate of score increase for the correct object of the $k$-th triplet, which is given by the term $\tilE_{k}^\top [-G_{\muon}(W_0)] E_k$ in the exponent of the softmax score, is bounded. The rate for the fastest-learning triplet is lower-bounded by the corresponding rate derived from $-G_{\muon}^{+}(W_0)$, while the rate for the slowest-learning triplet is upper-bounded by that from $-G_{\muon}^{-}(W_0)$. Thus, we only need to focus on $G_{\muon}^{+}(W_0)$ and $G_{\muon}^{-}(W_0)$ to calculate the desired quantity. Following the similar procedures for GD to derive Eqn.~\eqref{eq:gd_result}, we have that for any $\eta$ such that $\max_{k\in[K]}\big[f_{W_\eta}(E_k)\big]_{k}\geq 1- \epsilon$ (where $W_\eta=W_0-\eta\cdot G_{\muon}(W_0)$), the following holds
\begin{align}
    1-\min_{k\in[K]}\big[f_{W_\eta}(E_k)\big]_{k} \leq  \frac{\epsilon}{\epsilon+(1-\epsilon)^{r(K)}\epsilon^{1-r(K)}(K-1)^{r(K)-1}},\label{ieq:muon_result}
\end{align}
where $r(K)= (K-2C)/(K+2C)$. We further have that
\begin{align}
    &(1-\epsilon)^{r(K)}\epsilon^{1-r(K)}(K-1)^{r(K)-1}\nonumber\\
    &\quad = (1-\epsilon)\exp\bigg(\frac{4C}{K+2C}\bigg(\log\frac{\epsilon}{1-\epsilon}-\log (K-1)\bigg)\bigg)\nonumber\\
    &\quad = (1-\epsilon)\bigg[1+\frac{4C}{K+2C}\bigg(\log\frac{\epsilon}{1-\epsilon}-\log (K-1)\bigg)+O\bigg(\frac{(\log K)^2}{K^2}\bigg)\bigg]\nonumber\\
    &\quad = (1-\epsilon)+O\bigg(\frac{\log K}{K}\bigg),\label{eq:largeK}
\end{align}
where the first equality results from the basic calculations, the second equality results from that $\exp(x)=1+x+O(x^2)$ when $x\rightarrow 0$. Combining Eqn.~\eqref{ieq:muon_result} and \eqref{eq:largeK}, we have that
\begin{align*}
    \maxd_{\muon}^{\epsilon} \geq 1-  \epsilon\bigg(1+O\bigg(\frac{\log K}{K}\bigg)\bigg).
\end{align*}
Thus, we prove the desired results for Muon.

\textbf{Step 3: Calculations of Adam.}

The proof of the results for Adam is conducted under two cases. We will construct different embeddings $\tilE$ and $E$ in these two cases. In the first case, we set $\tilE=E=I_{K,K}$. With such embedding and sufficiently large $K$, we have that
\begin{align*}
    -G_{\adam}(W_0) = -\sign(\nabla_{W}\calL(W_0)) = 2I_{K,K} - J_{K,K}.
\end{align*}
Under such a setting, all triplets share the same dynamic. Thus, we have that
\begin{align*}
     \maxd_{\adam}^{\epsilon} = 1-\epsilon.
\end{align*}

In the second case, we set $\tilE$ and $E$ as block-wise diagonal matrices. Here we set the block size as $3$, i.e., requiring that $K\, \text{mod} \, 3=0$. Such a requirement can be satisfied  infinitely often  when $K\rightarrow\infty$. Then the sufficient and necessary condition of Assumption~\ref{assump:ortho} is that each $3\times 3$ block contains an orthonormal basis. To achieve this, we define the following matrix.
\begin{align*}
    R(a,b,c) = \begin{bmatrix}
          \cos a\cos b\cos c-\sin a\sin c      & - \cos a\cos b\sin c-\sin a \cos c &\cos a\sin b \\
          \sin a\cos b\cos c+\cos a\sin c      & - \sin a\cos b\sin c+ \cos a \cos c &\sin a\sin b \\
          -\sin b\cos c& \sin b\sin c& \cos b
        \end{bmatrix}.
\end{align*}
It is obvious that $R(a,b,c)^{\top}R(a,b,c)=I_{3,3}$. Then we set $\tilE$ and $E$ as
\begin{align*}
    \tilE= I_{ K/3, K/3}\otimes R(3.638, 2.949, 5.218),\quad E= I_{K/3, K/3}\otimes R(1.715, 0.876, 3.098),
\end{align*}
where $\otimes$ is the Kronecker product. With these specifications and sufficiently large $K$, the Adam update matrix is
\begin{align*}
    -G_{\adam}(W_0) = I_{K/3,K/3}\otimes A + J_{K/3,K/3}\otimes B,
\end{align*}
where $A$ and $B$ are specified as
\begin{align*}
    A = \begin{bmatrix}
          2      & 0 & 0\\
          2      & 0 & 2 \\
          -2 & -2& -2
        \end{bmatrix},\quad
    B = \begin{bmatrix}
          -1     & -1 & -1 \\
          -1     & -1 & -1 \\
          1     & 1 & 1 \\
        \end{bmatrix}.
\end{align*}
These show that the diagonal block of $-G_{\adam}(W_0)$ is
\begin{align*}
    A+B = \begin{bmatrix}
          1     & -1 & -1 \\
          1     & -1 & 1 \\
          -1     & -1 & -1 \\
        \end{bmatrix}.
\end{align*}
Since the $k$-th and $(k+3)$-th triplets share the same learning dynamics for all $k\in[K-3]$, we focus on the learning dynamics of $k=1,2,3$. We have that
\begin{align*}
    &R(3.638, 2.949, 5.218)^{\top}\cdot(A+B)\cdot R(1.715, 0.876, 3.098) \\
    &\qquad\qquad\qquad\qquad= \begin{bmatrix}
        1.46552253   & 1.0132908 & -0.11179563 \\
        -0.0732561 &  1.00709257& -1.26935805 \\
        0.0544114  & 0.89611102 & 1.54147329 
        \end{bmatrix},\\
    &R(3.638, 2.949, 5.218)^{\top}\cdot B\cdot R(1.715, 0.876, 3.098) \\
    &\qquad\qquad\qquad\qquad= \begin{bmatrix}
          -0.19288146 & -1.24460331 & -1.4058011 \\
 -0.20112175 & -1.2977753 & -1.46585978\\
 -0.12780259 & -0.82466989 & -0.93147899
        \end{bmatrix}.
\end{align*}
From the last columns of these two matrices, following the similar procedures for GD to derive Eqn.~\eqref{eq:gd_time}, we have that
\begin{align*}
    \eta_{\adam}^{\epsilon}\leq \frac{1}{1.541+0.930}\log\big[(\epsilon^{-1} -1)(K-1)\big]= \frac{1}{2.471}\log\big[(\epsilon^{-1} -1)(K-1)\big].
\end{align*}
Then, from the first columns of these matrices, we have that
\begin{align*}
    1-\min_{k\in[K]}\big[f_{-\eta_{\adam}^{\epsilon} G_{\adam}}(E_k)\big]_{k} \geq  \frac{\epsilon}{\epsilon+(1-\epsilon)^{r}\epsilon^{1-r}(K-1)^{r-1}}, 
\end{align*}
where $r=\frac{1.466+0.202}{2.471}=\frac{1.668}{2.471}$.

Thus, we have that
\begin{align*}
     \maxd_{\adam}^{\epsilon}\leq O(\epsilon^{-r}K^{r-1})\leq O(\epsilon^{-0.7}K^{-0.3}).
\end{align*}
Then we calculate the singular values of $-G_{\adam}(W_0)$. We define the eigen vectors of $I_{K,K}$ as $\tilU$, i.e., $\tilU^{\top}I_{K/3,K/3}\tilU=\diag(K/3,0\cdots,0)$. Using the orthogonal invariance of singular values, $-G_{\adam}(W_0)$ shares the singular values with the following matrix
\begin{align*}
    &(\tilU^{\top}\otimes I_{3,3})\big(-G_{\adam}(W_0)\big)(\tilU\otimes I_{3,3}) \\
    &= I_{K/3,K/3}\otimes A+(\tilU^{\top}I_{K/3,K/3}\tilU)\otimes B \\
    &= \diag(A-KB/3,A,\cdots,A).
\end{align*}
Thus, the singular values of $A$ are also the singular values of $G_{\adam}(W_0)$. We have that
\begin{align*}
    \frac{\sigma_{\min}\big(G_{\adam}(W_0)\big)}{\sigma_{\max}\big(G_{\adam}(W_0)\big)}\leq \frac{\sigma_{\min}(A)}{\sigma_{\max}(A)}\leq 25\%.
\end{align*}
Thus, we conclude the proof of Theorem~\ref{thm:comp}.

\section{Proof of Theorem~\ref{thm:multistep}}\label{app:multistep}

The proof of Theorem~\ref{thm:multistep} takes two steps. In the first step, we derive the share form of $W_t$ along the whole optimization trajectory. In the second step, we build the desired results on the basis of step 1. Throughout the proof, we will write $W_{t}^{\muon}$ as $W_t$ for the ease of presentation.

\textbf{Step 1: Derive the shared forms of $W_t$ and $G_\muon$.}

We will derive the forms of $W_t$ along the optimization trajectory via the induction method. We first state our hypothesis and then prove it.
\begin{hypothesis}
    For any optimization step index $t\in[T]$, the parameters $W_t$ can be expressed as
        \begin{align*}
            W_t = \tilE X_t E,\quad X_t = \Lambda_t +C_t,
        \end{align*}
        where $\Lambda_t$ and $C_t$ are
        \begin{align*}
            \Lambda_t = \diag(a_t \cdot \bbI_{L},b_t\cdot \bbI_{K-L}),\quad C_t = \begin{bmatrix}
          c_{t}^{11}\cdot J_{L,L}      & c_{t}^{12}\cdot J_{L,K-L}\\
          c_{t}^{21}\cdot J_{K-L,L}      & c_{t}^{22}\cdot J_{K-L,K-L}
        \end{bmatrix},
        \end{align*}
        where $a_t,b_t,c_{t}^{11},c_{t}^{12},c_{t}^{21},c_{t}^{22}\in\bbR$ are real numbers such that (1) $a_t = b_t\geq 0$, and (2) $c_{t}^{ij}=O(a_t/K)$ for $i,j\in[2]$.

\end{hypothesis}
When $t=0$, it is obvious to verify that $W_0=0_{d_o, d_s}$ satisfying this hypothesis with $a_t=b_t=c_{t}^{11}=c_{t}^{12}=c_{t}^{21}=c_{t}^{22}=0$. Then we assume that this hypothesis holds for $\{1,\cdots,t\}$, and we will prove that it holds for $t+1$. Since $W_{t+1}=W_{t}-\eta_{t+1}U_t \nor(\Sigma_t)V_t^{\top}$, we need to show that $-\eta_{t+1}U_t \nor(\Sigma_t)V_t^{\top}$ satisfies the hypothesis. We define the score of $k^{\prime}$-th object for the $k$-th subject-relation pair with the parameter $W$ as
\begin{align*}
    s(k^{\prime},k,W) = \frac{\exp(\tilE_{k^{\prime}}^{\top} W E_{k} )}{\sum_{k^{\prime\prime}=1}^{K}\exp(\tilE_{k^{\prime\prime}}^{\top} W E_{k} )}.
\end{align*}
According to the symmetry of $W_t$, we have that
\begin{itemize}
    \item $s(k,k,W_t)=s(1,1,W_t)\text{ for all }k\leq L$.
    \item $s(k,k,W_t)=s(K,K,W_t)\text{ for all }k>L$.
    \item $s(k^{\prime},k,W_t)=s(2,1,W_t)\text{ for all }k,k^{\prime}\leq L, k^{\prime}\neq k$.
    \item $s(k^{\prime},k,W_t)=s(K,1,W_t)\text{ for all }k\leq L, k^{\prime}>L$.
    \item $s(k^{\prime},k,W_t)=s(K-1,K,W_t)\text{ for all }k,k^{\prime}>L, k^{\prime}\neq k$.
    \item $s(k^{\prime},k,W_t)=s(1,K,W_t)\text{ for all }k> L, k^{\prime}\leq L$.
\end{itemize}
Thus, Proposition~\ref{prop:grad} shows that the gradient of $W_t$ is
\begin{align*}
    -\nabla_{W} \mathcal{L}(W_t) = \tilE(\Gamma_t+B_t) E^{\top},
\end{align*}
where $\Gamma_t$ and $B_t$ are defined as
\begin{align*}
    \Gamma_t &= \diag\bigg(\frac{\alpha}{L}\big(1+s(2,1,W_t)-s(1,1,W_t)\big)\bbI_{L},\\
    &\qquad\qquad\qquad\qquad\frac{1-\alpha}{K-L}\big(1+s(K-1,K,W_t)-s(K,K,W_t)\big)\bbI_{K-L}\bigg),\\
    B_t &= \begin{bmatrix}
          -\frac{\alpha}{L}s(2,1,W_t)\cdot J_{L,L}      & -\frac{1-\alpha}{K-L}s(1,K,W_t)\cdot J_{L,K-L}\\
          -\frac{\alpha}{L}s(K,1,W_t)\cdot J_{K-L,L}      & -\frac{1-\alpha}{K-L}s(K-1,K,W_t)\cdot J_{K-L,K-L}
        \end{bmatrix}.
\end{align*}
Thus, Proposition~\ref{prop:svd} shows that
\begin{align*}
    -G_\muon(W_t) = \tilE\bigg(\diag(\bbI_{K})+\begin{bmatrix}
         C_{11}\cdot J_{L,L}      & C_{12}\cdot J_{L,K-L}\\
         C_{21}\cdot J_{K-L,L}      & C_{22}\cdot J_{K-L,K-L}
        \end{bmatrix}\bigg) E^{\top},
\end{align*}
where
\begin{align*}
    C_{11} &= \frac{\tilU_{1,1}\tilV_{1,1}+\tilU_{1,2}\tilV_{1,2}-1}{\beta K}, & C_{12} &= \frac{\tilU_{1,1}\tilV_{2,1}+\tilU_{1,2}\tilV_{2,2}}{\sqrt{\beta(1-\beta)} K}, \\
    C_{21} &= \frac{\tilU_{2,1}\tilV_{1,1}+\tilU_{2,2}\tilV_{1,2}}{\sqrt{\beta(1-\beta)} K}, & C_{22} &= \frac{\tilU_{2,1}\tilV_{2,1}+\tilU_{2,2}\tilV_{2,2}-1}{(1-\beta) K}.
\end{align*}
where $\tilU,\tilV\in\bbR^{2\times 2}$ are the orthonormal matrices defined in Proposition~\ref{prop:svd}. Since $W_{t+1}=W_{t}-\eta_{t+1}G_\muon(W_t)$, it is obvious that $a_{t+1}=b_{t+1}$. The orthonormality of $\tilU$ and $\tilV$ implies that $|\tilU_{i,j}|,|\tilV_{i,j}|\leq 1$. Thus, we have 
\begin{align*}
    \frac{\tilU_{1,1}\tilV_{1,1}+\tilU_{1,2}\tilV_{1,2}-1}{\beta K} = O\bigg(\frac{1}{K}\bigg).
\end{align*}
This further implies that $c_{t+1}^{1,1}=O(a_{t+1}/K)$. The proofs for other $c_{t+1}^{ij}$ are similar. This completes the proof.

\textbf{Step 2: Establish the convergence results.}

We note that this analysis is very similar to the proof of Muon in Theorem~\ref{thm:comp}. Concretely, for $W_t$, the coefficients $a_t,b_t,c_{t}^{11},c_{t}^{12},c_{t}^{21},c_{t}^{22}$ from multiple-step optimization share the same property with those of the one-step results. It means that there exists a constant $C > 0$ such that the dynamics of the fastest- and slowest-learning triplets are bounded by those along the following two update directions in only one step.
\begin{align*}
    - G_{\muon}^{+} & = \bigg(1+\frac{2C}{K}\bigg) (\tilE_{1:L}E_{1:L}^{\top}+\tilE_{L+1:K}E_{L+1:K}^{\top})-\frac{C}{K}\cdot \tilE J_{K,K}E^{\top}\nonumber\\
    - G_{\muon}^{-} & = \bigg(1-\frac{2C}{K}\bigg) (\tilE_{1:L}E_{1:L}^{\top}+\tilE_{L+1:K}E_{L+1:K}^{\top})+\frac{C}{K}\cdot\tilE J_{K,K}E^{\top}.
\end{align*}
The remaining analysis is then exactly the same as that of Theorem~\ref{thm:comp}. Thus, we conclude the proof of Theorem~\ref{thm:multistep}.

    \section{Supporting Propositions}\label{app:supp_prop}

\begin{proposition}\label{prop:grad}
    We define the score of $k^{\prime}$-th object for the $k$-th subject-relation pair with the parameter $W$ as
    \begin{align*}
        s(k^{\prime},k,W) = \frac{\exp(\tilE_{k^{\prime}}^{\top} W E_{k} )}{\sum_{k^{\prime\prime}=1}^{K}\exp(\tilE_{k^{\prime\prime}}^{\top} W E_{k} )}.
    \end{align*}
    When the parameter $W$ is trained with loss
    \begin{align*}
        \mathcal{L}(W) = -\sum_{k=1}^{K}p_{k} \cdot \log \big[f_{W}(E_{k})\big]_{k},
    \end{align*}
    the gradient of $W$ is 
    \begin{align*}
        \nabla_{W} \mathcal{L}(W)= -\sum_{k=1}^{K}p_k \Big\{\big[1-s(k,k,W)\big]\tilE_{k}E_{k}^{\top} - \sum_{k^{\prime}\neq k} s(k^{\prime},k,W)\tilE_{k^{\prime}}E_{k}^{\top}\Big\}.
    \end{align*}
\end{proposition}
\begin{proof}[Proof of Proposition~\ref{prop:grad}]
    The proof just follows from the basic calculus. Thus, we omit them here.
\end{proof}
\begin{proposition}\label{prop:svd}
    Let $X = \Lambda + C \in \bbR^{K\times K}$. The matrix $\Lambda = \diag(a \cdot \bbI_L, b\cdot \bbI_{K-L}) $ is a diagonal matrix whose first $L$ diagonal elements are $a$ and the last $K-L$ elements are $b$ with $a,b>0$. The matrix $C$ is a block-wise constant matrix defined as
    \begin{align*}
        C = \begin{bmatrix}
          c_{11}\cdot J_{L,L}      & c_{12}\cdot J_{L,K-L}\\
          c_{21}\cdot J_{K-L,L}      & c_{22}\cdot J_{K-L,K-L}
        \end{bmatrix}.
    \end{align*}
    Then $X=U\Sigma V^{\top}$. Here $\Sigma,V,U$ are defined as follows. All of them can be decomposed into three blocks, each corresponding to a subspace. The first subspace is 
    \begin{align*}
        \mathcal{S}_1 = \left\{ \begin{bmatrix} x \\ 0_{K-L} \end{bmatrix} \bigg|\, x^\top \bbI_L = 0,\text{ and }x\in\bbR^{L} \right\}.
    \end{align*}
    %The dimension of this space is $L-1$. The singular value of this space is $a$. The corresponding bases in $U$ and $V$ are 
    The dimension of this space is $L-1$. The singular value of $X$ corresponding to this subspace is $a$. The block of columns in both $U$ and $V$ that forms an orthonormal basis for this subspace is given by
    \begin{align*}
        \begin{bmatrix} R_{L,L-1} \\ 0_{K-L,L-1} \end{bmatrix},
    \end{align*}
    where the columns of the matrix $R_{L,L-1} \in \mathbb{R}^{L \times (L-1)}$ form an orthonormal basis for the subspace $\{x \in \mathbb{R}^L | x^\top \bbI_L = 0\}$. The second subspace is  
    \begin{align*}
        \mathcal{S}_2 = \left\{ \begin{bmatrix} 0_{L} \\ y \end{bmatrix} \bigg|\, y^\top \bbI_{K-L} = 0,\text{ and }y\in\bbR^{K-L} \right\}.
    \end{align*}
    The dimension of this space is $K-L-1$. The singular value of $X$ corresponding to this subspace is $b$. The block of columns in both $U$ and $V$ that forms an orthonormal basis for this subspace is given by
    \begin{align*}
        \begin{bmatrix} 0_{L,K-L-1} \\ R_{K-L,K-L-1} \end{bmatrix},
    \end{align*}
    where the columns of the matrix $R_{K-L,K-L-1} \in \mathbb{R}^{(K-L) \times (K-L-1)}$ form an orthonormal basis for the subspace $\{y \in \mathbb{R}^{K-L} | y^\top \bbI_{K-L} = 0\}$. The remaining $2$-dimensional  subspace is induced by a $2\times 2$ matrix $M$ defined as
    \begin{align*}
        M = \begin{bmatrix}
          \alpha       & \beta\\
          \gamma      & \delta
        \end{bmatrix}=\tilU\diag(s_1,s_2)\tilV^{\top},
    \end{align*}
    where the elements of $M$ are defined as
    \begin{align*}
        \alpha = a + L c_{11}, \quad
        \beta = \sqrt{L(K-L)} \, c_{12}, \quad
        \gamma = \sqrt{L(K-L)} \, c_{21}, \quad
        \delta = b + (K-L) c_{22}.
    \end{align*}
    The singular values $s_1,s_2$ are 
    \begin{align*}
        s_{1,2} = \sqrt{\frac{T \pm \sqrt{T^2 - 4\Delta}}{2}}, \quad T = \alpha^2 + \beta^2 + \gamma^2 + \delta^2, \quad \Delta = (\alpha \delta - \beta \gamma)^2.
    \end{align*}
    %The singular values of $X$ in this subspace are $s_1$ and $s_2$. The corresponding bases in $U$ and $V$ as
    The singular values of $X$ in this subspace are $s_1$ and $s_2$. The corresponding right singular vectors ($v_i$) and left singular vectors ($u_i$), which form columns of $V$ and $U$ respectively, are given by:
    \begin{align*}
        v_i = \tilV_{1,i}e_1 + \tilV_{2,i}e_2, u_i = \tilU_{1,i}e_1 + \tilU_{2,i}e_2 \text{ for }i=1,2,
    \end{align*}
    where the vectors $e_1$ and $e_2$ are defined as
    \begin{align*}
        e_1 = \begin{bmatrix} \frac{1}{\sqrt{L}}\bbI_{L} \\ 0_{K-L} \end{bmatrix},\quad e_2 = \begin{bmatrix} 0_{L} \\ \frac{1}{\sqrt{K-L}}\bbI_{K-L} \end{bmatrix}.
    \end{align*}
    In summary, the \ac{svd} of $X$ is 
    \begin{align*}
        \Sigma &= \diag(a \cdot \bbI_{L-1}, b\cdot \bbI_{K-L-1}, s_1, s_2),\\
        V & = \bigg[\begin{bmatrix} R_{L,L-1} \\ 0_{K-L,L-1} \end{bmatrix}, \begin{bmatrix} 0_{L,K-L-1} \\ R_{K-L,K-L-1} \end{bmatrix}, v_1,v_2\bigg],\\
        U & = \bigg[\begin{bmatrix} R_{L,L-1} \\ 0_{K-L,L-1} \end{bmatrix}, \begin{bmatrix} 0_{L,K-L-1} \\ R_{K-L,K-L-1} \end{bmatrix}, u_1,u_2\bigg].\\
    \end{align*}
\end{proposition}

\begin{proof}[Proof of Proposition~\ref{prop:svd}]
    We first prove the results for $\calS_1$. For any vector $v$ in $\calS_1$, it is direct to verify that
    \begin{align*}
        X^{\top}X \begin{bmatrix} v \\ 0_{K-L} \end{bmatrix} = a^2 \begin{bmatrix} v \\ 0_{K-L} \end{bmatrix}.
    \end{align*}
    Thus, the singular value of $X$ corresponding to the subspace spanned by the vector $[ v^{\top}, 0_{K-L}^{\top} ]^{\top}$ is $a$, and the corresponding columns of $V$ form an orthonormal basis for $\calS_1$. For the $U$ calculation, we have that
    \begin{align*}
        X \begin{bmatrix} v \\ 0_{K-L} \end{bmatrix} = a \begin{bmatrix} v \\ 0_{K-L} \end{bmatrix}.
    \end{align*}
    Thus, the corresponding left singular vectors (columns of U) are identical to the right singular vectors for this subspace. A similar calculation can be done for $\calS_2$. The remaining vectors are orthogonal to both $\calS_1$ and $\calS_2$ and thus take the form of 
    \begin{align*}
        v_i = p_1 e_1 + p_2 e_2,\quad  u_i = p_3 e_1 + p_4 e_2 \text{ for }i=1,2 \text{ with }p_1, p_2, p_3, p_4\in\bbR.
    \end{align*}
    By solving the equation $X^{\top}X v_i = \lambda v_i$, we can show that the corresponding singular values and coefficients $p_1, p_2, p_3, p_4$ coincide with those in the \ac{svd} of $M$, as can be verified by simple calculations. Thus, we conclude the proof of Proposition~\ref{prop:svd}.
\end{proof}

\begin{proposition}\label{prop:svd_simp}
    Let $x = [a\cdot\bbI_{L}^{\top},b\cdot\bbI_{K-L}^{\top}]^{\top}\in\bbR^{K}$, and $X = \diag(x)- K^{-1}\bbI_{K}\cdot x^{\top}\in\bbR^{K\times K}$, where $a,b>0$. Then the \ac{svd} of $X=U\Sigma V^{T}$ is that
    \begin{align*}
        \Sigma &= \diag\bigg(a \cdot \bbI_{L-1}, b\cdot \bbI_{K-L-1}, \sqrt{\frac{a^2\cdot (K-L)+b^2\cdot L}{K}}, 0\bigg),\\
        V & = \bigg[\begin{bmatrix} R_{L,L-1} \\ 0_{K-L,L-1} \end{bmatrix}, \begin{bmatrix} 0_{L,K-L-1} \\ R_{K-L,K-L-1} \end{bmatrix}, v_1,v_2\bigg],\\
        U & = \bigg[\begin{bmatrix} R_{L,L-1} \\ 0_{K-L,L-1} \end{bmatrix}, \begin{bmatrix} 0_{L,K-L-1} \\ R_{K-L,K-L-1} \end{bmatrix}, u_1,u_2\bigg].\\
    \end{align*}
    %Here $R_{L,L-1}$ and $R_{K-L,K-L-1}$ are any orthonormal bases of the following $\calS_1$ and $\calS_2$, respectively. 
    Here, the columns of the matrix $R_{L,L-1} \in \mathbb{R}^{L \times (L-1)}$ form an orthonormal basis for the subspace of vectors in $\mathbb{R}^L$ orthogonal to $\bbI_L$. Similarly, the columns of $R_{K-L,K-L-1} \in \mathbb{R}^{(K-L) \times (K-L-1)}$ form an orthonormal basis for the subspace of vectors in $\mathbb{R}^{K-L}$ orthogonal to $\bbI_{K-L}$. These correspond to the subspaces $\mathcal{S}_1$ and $\mathcal{S}_2$ defined as:
    \begin{align*}
        \mathcal{S}_1 = \left\{ \begin{bmatrix} x \\ 0_{K-L} \end{bmatrix} \bigg|\, x^\top \bbI_L = 0,\text{ and }x\in\bbR^{L} \right\},\quad \mathcal{S}_2 = \left\{ \begin{bmatrix} 0_{L} \\ y \end{bmatrix} \bigg|\, y^\top \bbI_{K-L} = 0,\text{ and }y\in\bbR^{K-L} \right\}.
    \end{align*}
    The vectors $v_1,v_2,u_1,u_2$ are
    \begin{align*}
        v_1 &= \frac{1}{\sqrt{a^2(K-L)+b^2 L}}\bigg(\frac{a\sqrt{K-L}}{\sqrt{L}}\begin{bmatrix} \bbI_{L} \\ 0_{K-L} \end{bmatrix}-\frac{b\sqrt{L}}{\sqrt{K-L}}\begin{bmatrix} 0_{L} \\ \bbI_{K-L} \end{bmatrix}\bigg)\\
        v_2 &= \frac{1}{\sqrt{a^2(K-L)+b^2 L}}\bigg(b\begin{bmatrix} \bbI_{L} \\ 0_{K-L} \end{bmatrix}+a\begin{bmatrix} 0_{L} \\ \bbI_{K-L} \end{bmatrix}\bigg)\\
        u_1 &= \frac{1}{\sqrt{KL(K-L)}}\bigg((K-L)\begin{bmatrix} \bbI_{L} \\ 0_{K-L} \end{bmatrix}-L\begin{bmatrix} 0_{L} \\ \bbI_{K-L} \end{bmatrix}\bigg)\\
        u_2 &= \frac{1}{\sqrt{K}}\bbI_{K}.
    \end{align*}
\end{proposition}
%\begin{proof}[Proof of Proposition~\ref{prop:svd_simp}]
%    The proposition is just a corollary of Proposition~\ref{prop:svd}. The proof of Proposition~\ref{prop:svd_simp} is similar to that of Proposition~\ref{prop:svd}. We state it explicitly here for the ease of the proof of our main results. 
%\end{proof}

\begin{proof}[Proof of Proposition~\ref{prop:svd_simp}]
    This proposition is a direct corollary of Proposition~\ref{prop:svd}. The matrix $X = \diag(x)- K^{-1}\bbI_{K}\cdot x^{\top}$ is an instance of the general form $\Lambda+C$ from Proposition~\ref{prop:svd}.

    The diagonal part is $\Lambda = \diag(x) = \diag(a \cdot \bbI_L, b\cdot \bbI_{K-L})$. The off-diagonal part is $C = -K^{-1}\bbI_{K}\cdot x^{\top}$. We can write $C$ in block form:
    \begin{align*}
        C = -\frac{1}{K} \begin{bmatrix} \bbI_L \\ \bbI_{K-L} \end{bmatrix} \begin{bmatrix} a\bbI_L^\top & b\bbI_{K-L}^\top \end{bmatrix} = -\frac{1}{K} \begin{bmatrix} a J_{L,L} & b J_{L,K-L} \\ a J_{K-L,L} & b J_{K-L,K-L} \end{bmatrix}.
    \end{align*}
    This corresponds to setting the block-wise constants in Proposition~\ref{prop:svd} to:
    \begin{align*}
        c_{11} = -a/K, \quad c_{12} = -b/K, \quad c_{21} = -a/K, \quad c_{22} = -b/K.
    \end{align*}
    Substituting these into the formulas for $\alpha, \beta, \gamma, \delta$ from Proposition~\ref{prop:svd} gives:
    \begin{align*}
        \alpha &= a + L(-a/K) = a(K-L)/K \\
        \beta &= \sqrt{L(K-L)}(-b/K) \\
        \gamma &= \sqrt{L(K-L)}(-a/K) \\
        \delta &= b + (K-L)(-b/K) = bL/K
    \end{align*}
    These coefficients define the $2\times 2$ matrix $M$ from Proposition~\ref{prop:svd} for this specific case. We now analyze this matrix $M$. A key observation is that its determinant is zero:
    \begin{align*}
        \det(M) = \alpha\delta - \beta\gamma = \frac{a(K-L)}{K}\frac{bL}{K} - \left(\frac{L(K-L)}{K^2}\right)(-b)(-a) = 0.
    \end{align*}
    Since the determinant is zero, one of its singular values must be zero. The other singular value, $s_1$, can be calculated from the squared Frobenius norm (sum of squares of elements), which is also the sum of squared singular values ($s_1^2+s_2^2$):
    \begin{align*}
        s_1^2 + 0^2 = \alpha^2+\beta^2+\gamma^2+\delta^2 &= \frac{a^2(K-L)^2}{K^2} + \frac{L(K-L)b^2}{K^2} + \frac{L(K-L)a^2}{K^2} + \frac{b^2L^2}{K^2}\\
        & = \frac{a^2(K-L)+b^2L}{K}.
    \end{align*}
    This confirms the singular values stated in the proposition. The singular vectors $v_1, v_2, u_1, u_2$ can be derived by performing the SVD on this specific $2\times 2$ matrix $M$.
\end{proof}
\end{document}